\documentclass[phd,12pt]{psuthesis}

\usepackage{amsmath}
\usepackage{amssymb}
\usepackage{amsthm}
\usepackage{exscale}
\usepackage[mathscr]{eucal}
\usepackage{bm}
\usepackage{eqlist} 
\usepackage[final]{graphicx}
\usepackage[dvipsnames]{color}
\DeclareGraphicsExtensions{.pdf, .jpg}

\usepackage{epsf,psfig}
\usepackage{subfigure}
\usepackage{epsfig}
\usepackage{latexsym}
\usepackage{algorithm,algorithmic}
\usepackage{cite,url}
\usepackage{slashbox}

\usepackage{latexsym,amsmath,epsfig}
\usepackage{multirow}

\def\tb{\textbf}

\newcommand{\vect}[1]{\pmb{#1}}
\newcommand{\mat}[1]{\pmb{#1}}

\newcommand{\norm}[1]{\left\|#1\right\|}
\newcommand{\R}{\mathbb{R}}

\def\mb {\mathbf}

\def\tcb{\textcolor{blue}}
\def\tcr{\textcolor{red}}
\definecolor{newgreen}{RGB}{0,150,0}

\newtheorem{prop}{Proposition}
\newtheorem{theorem}{Theorem}
\newtheorem{lemma}{Lemma}

\def\ben{\begin{equation*}}
\def\een{\end{equation*}}
\def\be{\begin{equation}}
\def\ee{\end{equation}}
\def\beaa{\begin{eqnarray*}}
\def\eeaa{\end{eqnarray*}}
\def\bea{\begin{eqnarray}}
\def\eea{\end{eqnarray}}

\def\bleq{\begin{flalign}}
\def\eleq{\end{flalign}}

\DeclareMathOperator{\Tr}{Tr}

\usepackage[Lenny]{fncychap}
\ChTitleVar{\Huge\sffamily\bfseries}

\title{Enhanced Signal Recovery \\via Sparsity Inducing Image Priors}

\author{Hojjat Seyed Mousavi}
\dept{Electrical Engineering}
\degreedate{December 2017}

\honorsdegreeinfo{for a baccalaureate degree \\ in Engineering Science \\ with honors in Engineering Science}

\documenttype{Dissertation }

\submittedto{The Graduate School}

\numberofreaders{4}

\honorsadviser{Honors P. Adviser}

\secondthesissupervisor{Second T. Supervisor}

\honorsdepthead{Department Q. Head}

\advisor[Dissertation Advisor, Chair of Committee]
        {Vishal Monga}
        {Associate Professor of Electrical Engineering}

\readerone[]
          {William E. Higgins}
          {Distinguished Professor of Electrical Engineering}

\readertwo[]
          {Constantino Lagoa}
          {Professor of Electrical Engineering}

\readerthree[]
            {Robert T. Collins}
            {Associate Professor of Computer Science and Engineering}

\readerfour[]
            {Kultegin Aydin}
            {Professor of Electrical Engineering and Department Head}

\begin{document}
\frontmatter

%


\psutitlepage

\psucommitteepage

%
\pagestyle{plain}
\chapter*{Abstract}
    \begin{singlespace}
Parsimony in signal representation is a topic of active research. Sparse signal processing and representation is the outcome of this line of research which has many applications in information processing 
and has shown significant improvement in real-world applications such as recovery, classification, clustering, super resolution, etc. 
This vast influence of sparse signal processing in real-world problems raises a significant need in developing novel sparse signal representation algorithms to obtain more robust systems.
In such algorithms, a few open challenges remain in (a) efficiently posing sparsity on signals that can capture the structure of underlying signal and (b)  the design of tractable algorithms that can recover signals under aforementioned sparse models.


In this dissertation, we try to view the signal recovery problem from these viewpoints. First, we address the sparse signal recovery problem from a Bayesian perspective where sparsity is enforced on reconstruction coefficients via probabilistic priors. In particular, we focus on a variant of spike and slab prior, which is known to be the gold standard to encourage sparsity. The optimization problem resulting from this model has broad applicability in recovery and regression problems and is known to be a hard non-convex problem whose existing solutions involve simplifying assumptions and/or relaxations. We propose an approach called Iterative Convex Refinement (ICR) that aims to solve the aforementioned optimization problem directly allowing for greater generality in the sparse structure. Essentially, ICR solves a sequence of convex optimization problems such that sequence of solutions converges to a sub-optimal solution of the original hard optimization problem. We propose two versions of our algorithm: (a) an unconstrained version, and (b) with a non-negativity constraint on sparse coefficients, which may be required in some real-world problems. 
Many signal processing problems in computer vision and recognition world can benefit from ICR. These include face recognition in surveillance applications, object detection and classification in the video, image compression and recovery, image quality enhancement etc.

On the other hand, one of the most significant challenges in image processing is the enhancement of image quality.
To address this challenge we aim to recover signals by using the prior structural knowledge about them. In particular, we pose physically meaningful probabilistic priors to promote sparsity on reconstruction coefficients or design parameters of the problem.
This has a variety of applications in signal and image processing including but not limited to regression, denoising, inverse problems, demosaicking and super-resolution.
In particular, sparsity constrained single image super-resolution (SR) has been of much recent interest.
A typical approach involves sparsely representing patches in a low
resolution input image via a dictionary of example LR
patches, and then using the coefficients of this representation
to generate the high-resolution output via an analogous
HR dictionary.

In this dissertation, we propose extension of the SR problem which is twofold: (1) extension of sparsity-based SR problems
to multiple color channels
 by taking prior knowledge about the color information into account. Edge similarities amongst RGB color bands are exploited as cross-channel correlation constraints. These additional constraints lead to
 a new optimization problem, which is not easily solvable;
 however, a tractable solution is proposed to solve it efficiently. Moreover, to fully exploit the complementary information among color channels, a dictionary learning method is also proposed specifically to learn color dictionaries that encourage edge similarities
 (2) Tackle the super-resolution problem from a deep learning standpoint and provide deep network structures designed for super-resolution. A step further in this line of research is to integrate sparsifying priors into deep networks and investigate their impact on the performance especially in absence of abundant training.

    \end{singlespace}
\newpage

\thesistableofcontents

\thesislistoffigures

\thesislistoftables


%
\chapter{Acknowledgments}
\begin{singlespace}
I am grateful to my parents, who have provided me through moral and emotional support in my life. Without their support I wouldn't be here and I am very grateful to have them by my side. I am also grateful to my siblings, especially my brother, Mehdi, who was a role model for me and have supported me along the way.

With a special gratitude to Vishal Monga, my advisor, who help me through my PhD. It was fantastic to have the opportunity to work with him.

And finally, last but by no means least, also to everyone in the iPAL. Especially, Mahesh, Tiep, Tiantong and John whom I worked closely with and enjoyed every moment of it. It was great sharing laboratory with all of you during last five years.

Thanks for all your encouragement!
\end{singlespace}

\thesisdedication{SupplementaryMaterial/Dedication}{Dedication}

\thesismainmatter

\allowdisplaybreaks{
%

\chapter{Introduction}
\label{chapter:introduction}

In many domains of data processing, a commonly encountered problem is that of signal representation. This is a well studied problem and has a vast number of applications in information processing, in general. In this dissertation, we are particularly interested in
 signal recovery
under sparsity inducing image priors and its application to some real world problems.

In recent years, sparse modeling and sparse representation of signals have received a lot of attention and are very well-studied by the researchers in the signal processing and statistics community. This chapter aims to review some of relevant ideas to sparse signal representation and give a comprehensive understanding of its application in signal and image processing, machine learning and computer vision. Finally, we will motivate the contributions of this disseration.

\section {Sparsity in Signal Processing: Overview }

Parsimony in signal representation has been  shown to have value in many applications. It is now a well-known fact that a large class of signals can be represented in a compact manner with respect to some basis.
Among those bases, the most widely applicable ones are Fourier, wavelet or cosine basis.  In fact, this idea has been leveraged successfully in commercial signal compression algorithms known to be JPEG 2000 \cite{Taubman:JPEG_Book01}.
Motivated by prevalence of sparsity in human perception \cite{Olshausen_SimpleCellReceptive_Nature96}, research is conducted on sparse coding for signal and images for a variety of applications. Namely, signal and image acquisition \cite{Candes:ExactRecons_InfoTheory06}, data compression \cite{Taubman:JPEG_Book01} and modeling \cite{Lustig:CSforMRI_MRM07}.

A sparse vector or matrix is essentially a vector or matrix with only a few number of non-zero elements. The number of non-zero elements is called sparsity level. Signals or their corresponding vector/matrix representation can be sparse in a variety of different ways. Each of which has its own structure and can be applied in real-world problems to induce different kinds of structures and \emph{a priori} on the signals. Fig. \ref{Fig:SparseNotions} shows a few examples of how differently sparse signals can be represented. these structures have different motivations and backgrounds so that make each of them suitable for a specific scenario.

\begin{figure}
  \centering
  \includegraphics[width=0.4\textwidth]{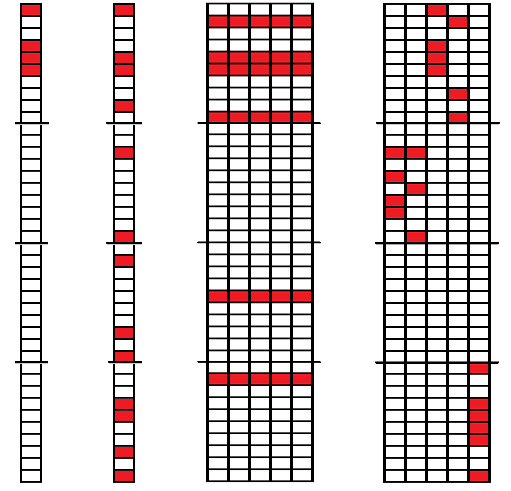}\\
  \caption{Examples of different sparse structures that can be captured by a vector or a matrix. Red boxes indicate non-zero values and white ones indicate zero values.}
  \label{Fig:SparseNotions}
\end{figure}

\subsection{Compressive Sensing}

This section surveys the theory of compressive sensing or CS in signal processing. CS theory \cite{Candes:ExactRecons_InfoTheory06} asserts that certain signals and images are recoverable from far fewer samples or measurements than traditional methods use. Compressive sensing aims to acquire and reconstruct a signal through optimization and  finding solutions to underdetermined linear systems. To make this possible, CS significantly relies on the fact that underlying signal is sparse in some proper basis.

In fact, compressive sensing which is also known as sparse sensing can be viewed as a formalization of the quest for economical signal acquisition and representation and it has witnessed a spurt over the past decade \cite{Candes:ExactRecons_InfoTheory06,Donoho:CS_InfoTheory06}. The sparse sensing problem is described as follows:
\be
\mb{y} = \bm{\Phi}\mb{x} + \mb{n},
\label{eq:sparse_sensing}
\ee
where $\mb{x}\in \mathbb{R}^n$ is the signal, $\mb{y} \in \mathbb{R}^m$ is the measurement ($m < n$), $\bm{\Phi} \in \mathbb{R}^{m\times n}$ is the measurement matrix, and $\mb{n} \in \mathbb{R}^m$ is additive noise. The sparse signal recovery problem is then posed as an optimization problem of the form:
\be
\mb{\hat{x}} = \arg\min_{\mb{x}}\|\mb{x}\|_0 ~\mbox{s.t.}~\|\mb{y} - \bm{\Phi}\mb{x}\|_2 \leq \epsilon.
\label{eq:sparse_recovery}
\ee
where $\|\vect x \|_0$ is the $\ell_0$ pseudo-norm and is equal to the number of non-zero elements of vector $\vect x$.
Under certain conditions, the $\ell_0$-norm can be relaxed to the $\ell_1$-norm, leading to a more tractable and convex optimization formulations. It is well-known that usually additional structure among coefficient of sparse signal is incorporated and it can be leveraged for better sparse recovery \cite{Cevher_SparseRecoveryGraphModels_SPM2010}.  For instance, a power-law relationship can be determined among the sorted coefficients of a compressible signal. Or as it is ascertained in \cite{Cevher_SparseRecoveryGraphModels_SPM2010} specific probability distributions such as the Student's t-distribution and generalized Pareto distribution are well-known to be compressible. In fact, incorporating such priors as \emph{compressible priors} in a probabilistic framework for sparse signal recovery problem has already been established and will be reviewed in later chapters.

The work in this area so far has focused on the application of sparsity for signal reconstruction problems, where a signal is recovered from a set of fewer measurements. This idea can also be extended to classification by learning \emph{class-specific dictionaries}, the novelty arising from the careful choice of dictionaries that best capture the inherent sparsity in a signal representation whilst having discriminative power.

\subsection{Sparse Representation-based Classification}

Classification is a commonly encountered problem in information processing domain. Typically in a classification problem, we have access to data/information belonging to two or more different classes or groups, and the challenge is to develop an effective automatic scheme of assigning a new object to its correct group. A list of classification problems is included but not limited to this indicative list: Object recognition, medical imaging, video tracking, pedestrian detection in video sequences, bioinformatics and biometric applications for security (fingerprint/ face recognition), communications, document classification, automatic target recognition, etc.

A few of common characteristics of real-world classification problems which we aim to address a few of them throughout this dissertation  are as follows:
\begin{itemize}
\item high-dimensional data,
\item limited access to training,
\item data acquisition in the presence of noise,
\item incorporating prior knowledge into training model,
\end{itemize}

A significant contribution to the development of algorithms for image classification that addressed some of the above mentioned challenges up to some extent is a recent sparse representation-based classification (SRC) framework \cite{Wright:SRC_PAMI2009}, which exploits the discriminative capability of sparse representations. The key idea is that of \emph{designing class-specific dictionaries} in combination with the analytical framework of compressive sensing. Given a sufficiently diverse collection of training images from each class, any image from a specific class can be approximately represented as linear combination of training images from the same class. Therefore, if we have training images of all classes and form a basis or dictionary based on that, any new and unseen test image has a sparse representation with respect to such overcomplete dictionary. It is worth to mention that sparsity assumption holds due to the class-specific design of dictionaries as well as the assumption of the linear representation model.

Suppose we have sets of training image vectors (vectorized images) from multiple classes, collected into dictionaries $\mat A_i, i=1,\ldots,K$. Let there be $N_i$ training images (each in $\R^n$) corresponding to the $K$ predefined classes ${C}_i, i = 1,\ldots,K$. The collection of all training images is expressed using the matrix called \emph{dictionary}
\be
\mat A = [\mat A_1 ~ \mat A_2 ~ \ldots ~ \mat A_K],
\label{eq:basis_rep}
\ee
where $\mat A \in \R^{n \times T}$, with $T = \sum_{i=1}^{K}N_i$. A new test image $\vect{y} \in \R^{n}$ can now be represented as a sparse linear combination of the training atoms as,
\be
\vect y ~\simeq~ \mat A_1 \vect x_1 + \mat A_2 \vect x_2 + \ldots + \mat A_K \vect x_K = \mat A \vect x,
\label{eq:sparse_rep}
\ee
where $\vect{x}$ is ideally expected to be a sparse vector (i.e., only a few entries in $\vect{x}$ are nonzero). Motivated by CS framework, the classifier seeks the sparsest representation by solving the following problem:
\be
\hat{\vect{x}} = \arg\min \norm{\vect{x}}_1 \quad\text{subject to}\quad \|\vect y - \mat A\vect{x}\|_2 \leq \epsilon. 
\label{eq:global_l0}
\ee
and then class assignment is simply carried out using the partial reconstruction error terms as follows:
\be
\text{identity}(\vect y) = \arg\min_i\|{\vect y  - \mat A \delta_{C_i}(\hat{\bm{x}})}\|_2
\label{Eq:ClassAsignment}
\ee
where $\delta_{C_i}(\hat {\vect x})$ keeps all the elements of $\vect x$ which are from class $C_i$ and zero-outs everything else.

Experiments have shown that even in the presence of severe distortions such as occlusion and pixel noise, the resulting sparse representation reveals the class-association of $\vect y$ with a high degree of accuracy \cite{Wright:SRC_PAMI2009,Pillai:IrisSRC_PAMI2011}. The robustness of the SRC method to real-world image distortions resulted in the widespread application of SRC in practical classification tasks.

With emergence of SRC as a powerful and robust classifier, modifications to the original framework have been considered. One simple extension uses  regularizers and minimizes the sum of $l_2$-norms of the sub-vectors $\vect x$, which results in an $l_1-l_2$ group sparse regularizer \cite{Yu:GroupSparsity_ISBI11}. This idea is very similar to the idea of the group Lasso proposed in \cite{Yuan:GroupRegression_RoyalStatSociet2006}. Along the same line, to enforce sparsity within each group, hierarchical Lasso was proposed in \cite{Sprechmann:CHI-LASSO_TSP2011}. Some other regularizers that promote group structures and  extend SRC have been proposed in \cite{Majumdar:GroupSparseClassification_ICASSP09,Majumdar:GroupSparseClassifier_PatternRecog10}.

\subsection{Sparse Priors for Signal Representation}

Prior information about the structure of signals often leads to significant performance improvements in many information analysis and signal processing applications. In fact, introducing regularizers in the analytical formulation of problems is motivated by having prior knowledge about the signal and the fact that it can help in obtaining solutions to ill-posed problems \cite{Tychonoff:IllPosedProblems_Book77}. Prior information also manifest itself in other forms such as constraints or probability distributions on signals.
Further, the success of sparse representation-based methods and its many extensions is based on the validity of the linear representation model which itself relies on the assumption that enough diversity is captured by a \emph{large enough} set of training samples. In practice however, most real-world applications have the limitation that rich training is not available.

An attempt to alleviate this problem was proposed to use prior information about the underlying signal. The main idea is that the sparse optimization problem in \eqref{eq:global_l0} can be interpreted as maximizing the posterior probability of observing $\vect x$ under a prior assumption that the coefficients of $\vect x$ are modeled as i.i.d. Laplacians. The Laplacian prior encourages coefficients to take values closer to zero, thereby approximating sparsity.
This framework motivates the interpretation of sparsity as a prior information for signal reconstruction. In fact, this is a particular example of a broader framework: Bayesian. In Bayesian perspective, signal comprehension can be enhanced by incorporating contextual information as priors.
One great benefit of leveraging Bayesian framework is the probabilistic prediction of sparse coefficient and automatic estimation of model parameters. \cite{Tychonoff:IllPosedProblems_Book77}

Sparsity is in fact a first-order description of structure in signals. However, often there is \emph{a priori} structure inherent to the sparse signals that is exploited for better representation, compression or modeling. In fact, We can categorize sparse recovery problems into two different category of approaches. One category uses sparsity inducing regularizers in conjunction with reconstruction error term to obtain a sparse approximation of the signal. Examples of these sort of techniques are \cite{Mohimani:fast_l_0_TSP2009, Tropp:OMP_InfoTheory2007, Tropp:ConvexSOMP_2006Elsevier, Wright:SRC_PAMI2009}. Another approach falls into the category of Model-based Compressive Sensing \cite{Baraniuk:Model_CS_InfoTheory2010} where a set of priors are introduced on top of the sparse signal to capture both sparsity and structure \cite{Baraniuk:Model_CS_InfoTheory2010, JiAndCarin:BayesianCS_TSP2008, Carin:WaveletBayesCS_TSP2009, Babacan_BayesianCSLaplacePriors_TIP2010, Andersen:BayesianSpikeSlab_NIPS2014}.


As an illustration, a connected tree structure can be enforced on wavelet coefficients to capture the multi-scale dependence \cite{Baraniuk:Model_CS_InfoTheory2010}. Other such structured prior models have also been integrated into the CS framework \cite{Carin:WaveletBayesCS_TSP2009, JiAndCarin:BayesianCS_TSP2008, Eldar:MultiChannelSparseRecov_InfoTheo2010}.
The wavelet-based Bayesian approach in \cite{Carin:WaveletBayesCS_TSP2009} employs a ``spike-and-slab'' prior \cite{Ishwaran_SpikeSlab_AnnStat2005, George_VariableSelectionGibbsSampling_StatAssoc1993, Chipman_BaysianVariableSelection_Stat1996, Carvalho_SparseFactorModeling_StatAssoc2008}, which is a mixture model of two components representing the zero and nonzero coefficients, and dependencies are encouraged in the mixing weights across resolution scales.

Introducing priors for capturing sparsity is a particular example of Bayesian inference where the signal recovery can be enhanced by exploiting contextual and prior information. As suggested by \cite{Cevher_LearningCompressiblePriors_NIPS2009, Cevher:SparseRecovGraphicalModel_SPMagaz2010}, sparsity can be induced via solving the following optimization problem:
\bea
    \max_{\vect x} P_{\vect x}(\vect x) & \textit{subject to} & ||\vect y - \mat A \vect x||_2 < \epsilon.        \label{Eq:GeneralMaxPrior}
\eea
where $P_{\vect x}$ is the probability density function of $\vect x$ that  \emph{simultaneously} captures the sparsity and structure (joint distributions of coefficients) of $\vect x$. In comparison, the standard CS recovery \eqref{eq:sparse_recovery} captures only the sparse nature of $\vect x$.
The most common example is the i.i.d. Laplacian prior which is equivalent to $\ell_1$ norm minimization \cite{Cevher:SparseRecovGraphicalModel_SPMagaz2010, Babacan_BayesianCSLaplacePriors_TIP2010}.
The choice of sparsity promoting priors that can capture joint distribution of coefficients (both structure and sparsity) is a challenging task. Examples of such priors are Laplacian \cite{Babacan_BayesianCSLaplacePriors_TIP2010}, generalized Pareto \cite{Cevher_SparseRecoveryGraphModels_SPM2010}, Spike and Slab \cite{Mitchell:BayesVarSelectSpikeSlab_StatAssoc1988}, etc.



\subsection{Other Applications}
Sparse representation methods have vast number of applications and we mentioned a few of them in the previous subsections. However, applications of sparsity based methods is not only limited to the aforementioned ones. They have also been widely used in Automatic Target Recognition (ATR) in many different modalities such as Synthetic Aperture Radar (SAR) imagery, HyperSpectral Imaging (HSI), etc \cite{Srinivas:MetaClassifierATR_RadarConf2011, Zhang:ATR_EUSAR2000, Zhang:MultiViewATR_taes12, Srinivas:SARATR_TAES2014,Bahrampour:KernelDicLearn_Arxiv2015}.

Sparsity has been playing an important role in many fields such as acoustic signal processing \cite{Grosse:SparseCodingAudio_Arxiv2012}, image processing \cite{Elad:SparsityRoleImageProcc_IEEE2010} and recognition \cite{Wright:SparsityComputerVision_IEEE2010}. It has also been widely used for clustering and subspace selection \cite{Elhamifar:SparseSubspaceClustering_PAMI2013}. Recently, sparsity has been applied to the problem of single image super resolution as well \cite{YangAndWright:SparseSR_TIP2010,YangAndWright2:SparseSR_CVPR2010}. In this problem, a single low resolution image is provided and with aim of sparse representation methods the high resolution image is achieved. We will talk about this area later in this dissertation.

\section{Information Fusion}


Advances in sensing technology have facilitated the easy acquisition of multiple different measurements of the same underlying physical phenomena. Often there is complimentary information embedded in these different measurements which can be exploited for improved performance. For example, in face recognition or action recognition we could have different views of a person's face captured under different illumination conditions or with different facial postures \cite{Zhang:JointDynamicSparseFaceRecognition_pattern12, Yuan:VisualClassificationMultitask_TIP12, Gross:MultiPIE, Guha:learning_PAMI12, Bahrampour:TreeSparsity_CVPR2014}. In automatic target recognition, multiple SAR (synthetic aperture radar) views are acquired \cite{Zhang:MultiViewATR_taes12}. The use of complimentary information from different color image channels in medical imaging has been demonstrated in \cite{Srinivas:SHIRC_ISBI2013,Srinivas:SHIRC_TMI2014}. In border security applications, multi-modal sensor data such as voice sensor measurements, infrared images and seismic measurements are fused \cite{Nasrabadi:MultiSensorFusion_Fusion11} for activity classification tasks. The prevalence of such a rich variety of applications where multi-sensor information manifests in different ways is a key motivation for one of our contribution in this dissertation.

\subsection{Collaborative Data}
Motivated by availability of information from different sources, many algorithms have proposed to use the joint and complementary information from different sources or observations. For instance, in classification scenarios, such as multi-view face recognition, multi-view SAR ATR or hyperspectral classification this joint information is  readily available and one can exploit these data for better classification accuracies. The SRC model is extended to incorporate this additional information by enforcing a common support set of training images for the $T$ correlated test images $\vect y_1,\ldots, \vect y_T$:
\be
    \label{eqn::joint_sparsity_model}
    \begin{split}
      \mat Y &= \begin{bmatrix} \vect y_1 & \vect y_2 & \cdots & \vect y_T \end{bmatrix}
    = \begin{bmatrix} \mat A\vect x_1 & \mat A\vect x_2 & \cdots & \mat A\vect x_T\end{bmatrix}\\
    &= \mat A \underbrace{\begin{bmatrix} \vect x_1 & \vect x_2 & \cdots & \vect x_T\end{bmatrix}}_{\mat X}
    = \mat A \mat X.
    \end{split}
\ee
The simplest way of enforcing structure on $\mat X$ is to assume that the vectors $\vect x_i, i = 1,\ldots,T$, all have non-zero entries at the \emph{same} locations, albeit with different weights. This leads to the recovery of a sparse matrix $\mat S$ with only a few nonzero rows,
\begin{equation}
\label{eqn::joint_sparse_recovery}
\hat{\mat X} = \arg\min \norm{\mat Y - \mat A \mat X}_F \quad\text{subject to}\quad \norm{\mat X}_{\text{row},0} \leq K_0,
\end{equation}
where $\norm{\mat X}_{\text{row},0}$ denotes the number of non-zero rows of $\mat X$ and $\norm{\cdot}_F$ is the Frobenius norm. The greedy Simultaneous Orthogonal Matching Pursuit (SOMP) \cite{Tropp:ConvexSOMP_2006Elsevier, Tropp:SOMP_Elsevier2006} algorithm has been proposed to solve this non-convex problem. More general version of such joint sparsity models have been proposed by \emph{Zhang et al.} in \cite{Zhang:JointDynamicSparseFaceRecognition_pattern12} . In their so called Joint Dynamic Sparse Representation-based Classification (JDSRC) they proposed a joint dynamic sparsity model  instead of the rigid row sparsity model discussed previously . \emph{Yuan et al.} proposed a joint classification framework using information from different modalities in \cite{Yuan:VisualClassificationMultitask_TIP12}. Recently, \emph{Srinivas et al.} also proposed a simultaneous sparsity model for histopathological image classification and used the complementary color information in RGB channels \cite{Srinivas:SHIRC_TMI2014}.  Other joint or simultaneous sparse representations methods proposed for regression and multi-view ATR \cite{Zhang:MultiViewATR_taes12, Yuan:GroupRegression_RoyalStatSociet2006}.

Collaborative data may come from different sources according to the type of the problem we are dealing with. For instance, in signal recovery multiple measurements of the same phenomena can provide different sources of data or multiple sensor measurements may provide complementary data that can be used for recovery. One typical form of collaborative data that can provide complementary information is multi-spectral or in particular color information. Color information from separate RGB channels in image processing has been successfully applied for medical image classification \cite{Srinivas:SHIRC_TMI2014} and image enhancement \cite{Srinivas:ColorSR_CIC2011}. Later in this dissertation  we exploit color information as strong prior information for image super resolution and illustrate their benefits on this task.

%
%


\section{Goals and Contributions}
Following this brief overview of the key ideas that will constitute this thesis, it is now appropriate to state the goals of this research as follows:
\begin{itemize}
\item Use prior information on sparse signals and model parameters for boosting performance of both recovery and enhancement problems.
\item Extend the application of joint information to broader applications and exploit color information in optical images for super-resolution.
\item Extend the application of image priors to deep learning frameworks by exploiting signals' structure using wavelets and regularizing network structures using prior knowledge.
\end{itemize}


A snapshot of the main contributions of this dissertation is presented next. Each contribution approaches the issue of signal recovery from different viewpoints using multiple instances of prior knowledge, with the deployment of a rich variety of algorithmic tools from signal processing and machine learning.


\textbf{Chapter \ref{chapter:contrib1}} addresses the sparse signal recovery problem in a Bayesian framework where sparsity is enforced on reconstruction coefficients via probabilistic priors. In particular, we focus on a variant of spike and slab prior to encourage sparsity. The optimization problem resulting from this model has broad applicability in recovery and regression problems and is known to be a hard non-convex problem whose existing solutions involve simplifying assumptions and/or relaxations. We propose an approach called Iterative Convex Refinement (ICR) that aims to solve the aforementioned optimization problem directly allowing for greater generality in the sparse structure. Essentially, ICR solves a sequence of convex optimization problems such that sequence of solutions converges to a sub-optimal solution of the original hard optimization problem. We propose two versions of our algorithm: a.) an unconstrained version, and b.) with a non-negativity constraint on sparse coefficients, which may be required in some real-world problems. Experimental validation is performed on both synthetic data and for a real-world image recovery problem, which illustrates merits of ICR over state of the art alternatives.

\textbf{Chapter \ref{chapter:contrib2}} serves as an introduction to single image super resolution and for the rest of the dissertation.
For super resolution task, sparsity constrained single image super-resolution (SR) has been of much recent interest. A typical approach involves sparsely representing patches in a low-resolution (LR) input image via a dictionary of example LR patches, and then using the coefficients of this representation to generate the high-resolution (HR) output via an analogous HR dictionary. However, most existing sparse representation methods for super resolution focus on the luminance channel information and do not capture interactions between color channels. In this contribution, we extend sparsity based super-resolution to multiple color channels by taking color information into account as strong prior information. Edge similarities amongst RGB color bands are exploited as cross channel correlation constraints. These additional constraints lead to a new optimization problem  which is not easily solvable; however, a tractable solution is proposed to solve it efficiently. Moreover, to fully exploit the complementary information among color channels, a dictionary learning method is also proposed specifically to learn color dictionaries that encourage edge similarities. Merits of the proposed method over state-of-the-art are demonstrated  both visually and quantitatively using image quality metrics.

Finally, \textbf{chapter \ref{chapter:contrib3}} extends the single image super resolution problem to deep learning methods where recent advances have seen a surge of such approaches for image super-resolution. Invariably, a network, e.g.\ a deep convolutional neural network (CNN) or auto-encoder is trained to learn the relationship between low and high-resolution image patches.
Most of the deep learning based image super resolution methods work on spatial domain data and aim to reconstruct pixel values as the output of network. In the first part of this chapter, we explore the advantages of exploiting transform domain data in the  SR task especially for capturing more structural information in the images to avoid artifacts. In addition to this and  motivated by promising performance of VDSR and residual nets in super resolution task, we propose our Deep Wavelet network for Super Resolution (DWSR). Using wavelet coefficients encourages activation sparsity in middle layers as well as output layer. In addition to this, wavelet coefficients decompose the image into sub-bands which  provide structural information depending on the types of wavelets used. For example, Haar wavelets provide vertical, horizontal and diagonal edges in wavelet sub-bands which can be  used to infer more structural information about the image. Essentially our network uses complementary structural information from other sub-bands to predict the desired high-resolution structure in each sub-band.

On the other hand, deep learning methods have shown promising performance in super resolution and many other tasks in presence of abundant training which means thousands or millions of training data points are available. However, they suffer in cases where training data is not readily available. In the second part of this chapter and as our final contribution in this dissertation, we investigate the performance of such deep structures in low training data scenarios and show that their performance drops significantly. We look for remedies to this performance degradation by exploiting prior knowledge about the problem. This could be in terms of prior knowledge about the structure of images, or inter-pixel dependencies. In particular, we propose to use natural image priors for image super resolution and demonstrate that image priors in low training data  scenarios enhance the recovery of high resolution images despite having much less training data available.

\chapter{Contribution I: Iterative Convex Refinement for Sparse Signal Recovery}
\label{chapter:contrib1}

%
%
%
%
%

\section{Introduction}

Sparse signal approximation and compressive sensing (CS) have recently gained considerable interest both in signal and image processing as well as statistics. Sparsity is  often a natural assumption in inverse problems and sparse reconstruction or representation has variety of applications in image/signal classification \cite{Wright:SRC_PAMI2009,Srinivas:SHIRC_TMI2014,Srinivas:SSPIC_TIP2015, Srinivas:SHIRC_ISBI2013, Mousavi:MICHS_ICIP2014, Bahrampour:TreeSparsity_CVPR2014, electronics4020221, farhat2014towards}, dictionary learning \cite{Suo1:DirtyDicLearn_ICIP2014, Pourkamali:CompresiveKSVD_ICASSP2013, SadeghiAndBabaiezade1:DicLearnSparse_SPLetter2013, Vu:DFDL_ISBI2015, Bahrampour:DicLearn_Arxiv2015, Bahrampour:KernelDicLearn_Arxiv2015}, signal recovery \cite{Wright:SpaRSA_TSP2009 , Tropp:OMP_InfoTheory2007}, image denoising and inpainting \cite{Elad:ImageDenoiseSparsity_TIP2006 } and  MRI image reconstruction \cite{ Andersen:BayesianSpikeSlab_NIPS2014}. Typically, sparse models assume that a signal can be efficiently represented as sparse linear combination of atoms in a given or learned dictionary \cite{Wright:SRC_PAMI2009,Sprechmann:CHI-LASSO_TSP2011}. In other words, from CS viewpoint, a sparse signal can be recovered from fewer number of observations \cite{Carin:WaveletBayesCS_TSP2009, JiAndCarin:BayesianCS_TSP2008 }.

A typical sparse reconstruction algorithm aims to recover a sparse signal $\vect x \in \mathbb{R}^{p }$ from a set of fewer measurements $\vect y \in \mathbb{R}^{q}$ ($q\ll p$) according to the following model:
\bea
    \vect y = \mat A \vect x + \vect n,     \label{Eq:y=Ax}
\eea
where $\mat A \in \mathbb{R}^{q\times p}$ is the measurement matrix (Dictionary) and $\vect n \in \mathbb{R}^{q}$ models the additive Gaussian noise with variance $\sigma^2$.

In recent years, many sparse recovery algorithms have been proposed including but not limited to the following: proposing sparsity promoting optimization problems involving different regularizers such as $\ell_1$ norm, $\ell_0$ pseudo norm, greedy algorithms \cite{Tropp:OMP_InfoTheory2007,   Mousavi:AssymLASSO_arXive2013, Mohimani:fast_l_0_TSP2009, bilen2014data, 7390994}, Bayesian-based methods \cite{JiAndCarin:BayesianCS_TSP2008, Lu:SparseCodeBayesPerspec_NeuralNetLearn2013, DobigeonAndHero:HierarchyBayesImageRecons_TIP2009} or general sparse approximation algorithms such as SpaRSA, ADMM, etc. \cite{Wright:SpaRSA_TSP2009, BeckerAndCandes:SparseRecoveryNESTA_ImagScienSIAM2011,  Boyd:ADMM_MachineLearn2011, farhat2016stochastic}.

In Bayesian sparse recovery, the choice of priors plays a key role in promoting sparsity and improving performance. Examples of such priors are Laplacian \cite{Babacan_BayesianCSLaplacePriors_TIP2010}, generalized Pareto \cite{Cevher:SparseRecovGraphicalModel_SPMagaz2010}, Spike and Slab \cite{Mitchell:BayesVarSelectSpikeSlab_StatAssoc1988}, etc.
In particular, we focus on the setup of \emph{Yen et al.}  \cite{Yen:MM_VariableSelectionSpikeSlab_Stat2011} who employ a variant of spike and slab prior to encourage sparsity. The optimization problem resulting from this model has broad applicability in recovery and regression problems and is known to be a hard non-convex problem whose existing solutions involve simplifying assumptions and/or relaxations \cite{Yen:MM_VariableSelectionSpikeSlab_Stat2011, Srinivas:SSPIC_TIP2015, Andersen:BayesianSpikeSlab_NIPS2014}. However, in this work we aim to solve the resulting optimization problem directly in its general form. Our approach can be seen as a logical evolution of $\ell_1$ reweighted methods \cite{Gorodnitsky:L1ReweightFOCUSS_TSP1997, Candes:ReweightedL1_FourierAnalysis2008}. Motivated by this, the \textbf{Main Contributions} of our work are as follows: (1) We propose a novel Iterative Convex Refinement (ICR) for sparse signal recovery.  Essentially, the sequence of solutions from these convex problems approaches a sub-optimal solution of the hard non-convex problem. (2) We propose two versions of ICR: a.) an unconstrained version, and b.) with a non-negativity constraint on sparse coefficients, which may be required in some real-world problems such as image recovery. (3) Finally, we perform experimental validation on both synthetic data and a realistic image recovery problem, which reveals the benefits of ICR over other state-of-the-art sparse recovery methods. Further, we compare the solution of various sparse recovery methods against the \emph{global solution} for a small-scale problem, and remarkably the proposed ICR finds the most agreement with the global solution. Finally, convergence analysis is provided in support of the proposed ICR algorithm.
\section{Proposed Setup for Sparse Signal Recovery }
Introducing priors for capturing sparsity is a particular example of Bayesian inference where the signal recovery can be enhanced by exploiting contextual and prior information. As suggested by \cite{Cevher_LearningCompressiblePriors_NIPS2009, Cevher:SparseRecovGraphicalModel_SPMagaz2010}, sparsity can be induced via solving the following optimization problem:
\bea
    \max_{\vect x} P_{\vect x}(\vect x) & \textit{subject to} & ||\vect y - \mat A \vect x||_2 < \epsilon.        \label{Eq:GeneralMaxPrior}
\eea
where $P_{\vect x}$ is the probability density function of $\vect x$ that  captures  sparsity. The most common example is the i.i.d. Laplacian prior which is equivalent to $\ell_1$ norm minimization \cite{Cevher:SparseRecovGraphicalModel_SPMagaz2010, Babacan_BayesianCSLaplacePriors_TIP2010}.
A well-suited sparsity promoting prior is spike and slab prior which is widely used in sparse recovery and Bayesian inference for variable selection and regression \cite{Ishwaran_SpikeSlab_AnnStat2005, Carin:WaveletBayesCS_TSP2009, Andersen:BayesianSpikeSlab_NIPS2014, Suo:HierarchySpikeSlab_ICASSP2013}. In fact, it is acknowledged that spike and slab prior is indeed the \emph{gold standard} for inducing sparsity in Bayesian inference \cite{Lazaro:SpikeSlabInferMultiTask_NIPS2011}. Using this prior, every coefficient $x_i$ is modeled as a mixture of two densities as follows:
\bea
    x_i \sim (1-w_i) \delta_0 + w_i P_i(x_i) \label{Eq:GeneralSpikeSlab}
\eea
where $\delta_0$ is the Dirac function at zero (spike) and $P_i$ (slab) is an appropriate prior distribution for nonzero values of $x_i$ (e.g. Gaussian). $w_i \in [0,1]$ controls the structural sparsity of the signal. If $w_i$ is chosen to be close to zero $x_i$  tends to remain zero. On the contrary, by choosing $w_i$ close to 1, $P_i$ will be the dominant distribution encouraging $x_i$ to take a non-zero value.

\noindent\textbf{Optimization Problem} (Hierarchical Bayesian Framework):
Any inference from the posterior density for this model will be ill-defined because the Dirac's delta function is unbounded. Some ways to handle this issue include approximations \cite{Lazaro:SpikeSlabInferMultiTask_NIPS2011}, such as approximation of spike term with a narrow Gaussian \cite{George_VariableSelectionGibbsSampling_StatAssoc1993}, approximating the whole posterior function with product of Gaussian(s) and Bernoulli(s) density functions  \cite{Hernandez:EP-SpikeSlab_MachineLearning2014, Hernandez:EP-GeneralizeSpikeSlab_MachineLearnResearch2013, Andersen:BayesianSpikeSlab_NIPS2014, Kappen:VariationalGarrote_MachineLearning2014, Vila:EMBGAMP_Asilomar2011}, etc. In this work, we focus on the setup of \emph{Yen et al.} \cite{Yen:MM_VariableSelectionSpikeSlab_Stat2011} which is an approximate  spike and slab prior for inducing sparsity on $\vect x$.
Inspired by Bayesian compressive sensing (CS) \cite{JiAndCarin:BayesianCS_TSP2008, Suo:HierarchySpikeSlab_ICASSP2013}, we employ a hierarchical Bayesian framework for signal recovery.
More precisely, the Bayesian formulation is as follows:
\bea
    \vect y | \mat A, \vect x, \vect \gamma, \sigma^2 & \sim &   \mathcal{N} \left(\mat A \vect x, \sigma^2\mat I \right) \label{eq:pdf_y}\\
    \vect x | \vect\gamma, \lambda, \sigma^2 & \sim &   \prod_{i=1}^{p} ~\gamma_{i} \mathcal{N}(0,\sigma^2\lambda^{-1}) + (1-\gamma_{i}) \mathbb{I}(x_i = 0) \label{eq:pdf_x}\\
    \vect\gamma | \vect  \kappa & \sim &   \prod_{i=1}^{p} ~\mbox{Bernoulli}(\kappa_{i}) \label{Eq:pdf_gamma}
\eea
where $\mathcal{N}(.)$ represents the Gaussian distribution. Also note that in \eqref{eq:pdf_x} each coefficient of $\vect x$ is modeled based on the framework proposed in \cite{Yen:MM_VariableSelectionSpikeSlab_Stat2011}. Since $\gamma_i$  is a binary variable, it implies that conditioned on $\gamma_i = 0$, $x_i$ is equal to $0$ with probability one. On the other hand, conditioned on $\gamma_i = 1$, $x_i$ follows a normal distribution with mean $0$ and variance $\sigma^2\lambda^{-1}$. Motivated by Yen {\em et al}'s \emph{maximum a posteriori} (MAP) estimation technique \cite{Yen:MM_VariableSelectionSpikeSlab_Stat2011, Srinivas:SSPIC_TIP2015} the optimal $\vect x, \vect \gamma$ are obtained by the following MAP estimate.
\bea
        (\vect x^\ast, \vect \gamma^\ast )& =& \arg \max_{\vect x,\vect \gamma } \left\{  f (\vect x,\vect \gamma |\mat A,\vect y,\mat \kappa,\lambda ,\sigma^2)\right\}.
        \label{Eq:MAP_Estimate}
\eea
\begin{prop}  
The MAP estimation above is equivalent to the following minimization problem: 
\bea
    (\vect x^\ast, \vect \gamma^\ast ) &=& \arg\min_{\vect x, \vect \gamma} ~~   ||\vect y  - \mat A\vect x ||_{2}^2 + \lambda ||\vect x||_{2}^2 + \sum_{i=1}^{p}  \rho_{i}  \gamma_{i}       \label{Eq:MainOptProbICR}
\eea
where $\rho_{i} \triangleq  \sigma^2\log\left(\frac{2\pi\sigma^2(1-\kappa_i)^2}{\lambda\kappa_i^2}\right)$.
\end{prop}
\begin{proof}
To perform the MAP estimation, note that the posterior probability is given by:
\bea
f(\vect x,\vect \gamma, |\mat A,\vect y,\lambda, \vect \kappa) \propto f(\vect y |\mat A,\vect x,\vect \gamma,\sigma^2)f(\vect x|\vect \gamma,\sigma^2,\lambda) f(\vect \gamma | \vect \kappa).
\label{eq:joint-posterior}
\eea
The optimal $\vect x^\ast, \vect \gamma^\ast $ are obtained by MAP estimation as:
\bea
(\vect x^\ast, \vect \gamma^\ast) = \arg \min_{\vect x,\vect \gamma} \left\{-2\log f(\vect x,\vect \gamma, |\mat A,\vect y,\lambda, \vect \kappa)\right\}.
\label{Eq:Log_MAP_Estimate}
\eea
We now separately evaluate each term on the right hand side of \eqref{eq:joint-posterior}. According to \eqref{eq:pdf_y} we have:
\beaa
f(\vect y |\mat A,\vect x,\vect \gamma,\sigma^2)  =    \frac{1}{(2\pi\sigma^2)^{q/2}} \exp\left\{-\frac{1}{2\sigma^2}(\vect y - \mat A \vect x)^T(\vect y - \mat A \vect x)\right\}
\eeaa
\beaa
\Rightarrow -2\log f(\vect y |\mat A,\vect x,\vect \gamma,\sigma^2)  =  q\log \sigma^2 + q\log (2\pi) + \frac{1}{\sigma^2}||\vect y - \mat A \vect x||^2.
\eeaa
Since $\gamma_i$ is assumed to be the indicator variable and only takes values $1$ and $0$, we can rewrite \eqref{eq:pdf_x} in the following form:
\beaa
     \vect x | \vect\gamma, \lambda, \sigma^2 & \sim &   \prod_{i=1}^{p} ~  \Big( \mathcal{N}(0,\sigma^2\lambda^{-1}) \Big) ^{\gamma_i} . \Big( \mathbb{I}(x_i = 0) \Big) ^{1-\gamma_i} \label{eq:pdf1_x}
\eeaa
Therefore
\beaa
f\big(\vect x |\vect \gamma,\sigma^2,\lambda\big) &=&  \prod_{i=1}^{p} \left(\frac{1}{(2\pi\sigma^2/\lambda)^{1/2}}\right)^{\gamma_i} \exp\left(-\frac{\gamma_i x_i^2}{2\sigma^2\lambda^{-1}}\right)(\mathbb{I}(x_i = 0))^{1-\gamma_i}  \\
&=  &    \left(\frac{2\pi\sigma^2}{\lambda}\right)^{-\frac{1}{2}\sum_{i=1}^{p}\gamma_i} \exp\left\{-\frac{1}{2\sigma^2\lambda^{-1}}\sum_{i=1}^{p}\gamma_i x_i^2 \right\}   \prod_{i=1}^{p} \mathbb{I}(x_i = 0)^{1-\gamma_i}\\
&= &\left(\frac{2\pi\sigma^2}{\lambda}\right) ^{-\frac{1}{2}  \sum_{i=1}^{p} \gamma_i} \exp\left(-\frac{  ||\vect x||_2^2 }{2\sigma^2 \lambda^{-1}} \right)   \prod_{i=1}^{p} \mathbb{I}(x_i = 0)^{1-\gamma_i}
\eeaa
\beaa
&\Rightarrow & -2\log f(\vect x|\vect \gamma,\sigma^2,\lambda) = \frac{  ||\vect x||_2^2}{\sigma^2\lambda^{-1}} + \log\left(\frac{2\pi\sigma^2}{\lambda}\right)  \sum_{i=1}^{p} \gamma_i -2 \sum_{i=1}^{p}(1-\gamma_i)\log\mathbb{I}(x_i = 0).
\eeaa
In fact  the final term on the right hand side evaluates to zero, since $\mathbb{I}(x_i = 0) = 1 \Rightarrow \log\mathbb{I}(x_i = 0) = 0$, and $\mathbb{I}(x_i = 0) = 0 \Rightarrow x_i \neq 0 \Rightarrow \gamma_i = 1 \Rightarrow (1-\gamma_i) = 0$.

Finally  \eqref{Eq:pdf_gamma} implies that
\beaa
f(\vect \gamma | \vect \kappa) & = &  \prod_{i=1}^{p}\kappa_i^{\gamma_i}(1-\kappa_i)^{1-\gamma_i} 
\eeaa
\beaa
\Rightarrow -2\log f(\vect \gamma | \vect \kappa) &=& -2 \sum_{i=1}^{p} \log \kappa_i^{\gamma_i} + \log (1-\kappa_i)^{1-\gamma_i} \\
&=& -2 \sum_{i=1}^{p} \gamma_i \log \kappa_i + (1-\gamma_i)\log (1-\kappa_i)\\
& = & -2 \sum_{i=1}^{p} \gamma_i \log \Big( \frac{\kappa_i}{1-\kappa_i}\Big) + \log (1-\kappa_i)\\
& = &  \sum_{i=1}^{p}\gamma_i\log\left(\frac{1-\kappa_i} {\kappa_i}\right)^2 - 2\sum_{i=1}^{p} \log(1-\kappa_i).
\eeaa
Plugging all these expressions back into  \eqref{Eq:Log_MAP_Estimate} and neglecting constant terms, we obtain:
\bea
(\vect x^\ast, \vect \gamma^\ast )  &=& \arg \min_{\vect x,\vect \gamma }
   q\log \sigma^2 + \frac{1}{\sigma^2}||\vect y - \mat A \vect x||^2 + \frac{  ||\vect x||_2^2}{\sigma^2\lambda^{-1}}  \nonumber \\
 &&  +  \log\left(\frac{2\pi\sigma^2}{\lambda}\right) \sum_{i=1}^{p} \gamma_i+  \sum_{i=1}^{p} \gamma_i\log\left(\frac{1-\kappa_i} {\kappa_i}\right)^2
\eea
Essentially, for fixed $\sigma^2$ The cost function will reduce to:
\bea
L(\vect x, \vect \gamma)& = &   ||\vect y - \mat A\vect x||_2^2 + \lambda ||\vect x||_2^2 +   \sum_{i=1}^{p}  \rho_ i \gamma_i
\eea
where $\rho_{i} \triangleq  \sigma^2\log\left(\frac{2\pi\sigma^2(1-\kappa_i)^2}{\lambda\kappa_i^2}\right)$.
\end{proof}

\emph{Remark:} Note that we are particularly interested in solving \eqref{Eq:MainOptProbICR} which has broad applicability in recovery and regression\cite{Yen:MM_VariableSelectionSpikeSlab_Stat2011}, image classification and restoration \cite{Chouzenoux:l2l1l0withMM_SIAMImageScience2013, Srinivas:SSPIC_TIP2015} and sparse coding\cite{Lu:SparseCodeBayesPerspec_NeuralNetLearn2013, Chaari:l2l1l0BAyesian_ICASSP2014}. This is a non-convex mixed-integer programming involving the binary indicator variable $\vect \gamma$ and is not easily solvable using conventional optimization algorithms. It is worth mentioning that this is a more general formulation than the framework proposed in \cite{Srinivas:SSPIC_TIP2015} or \cite{Yen:MM_VariableSelectionSpikeSlab_Stat2011} where authors simplified the optimization problem by assuming the same $\kappa$ for each coefficient $\gamma_i$. This assumption changes the last term in \eqref{Eq:MainOptProbICR} to $\rho||\vect x||_0$ and the resulting optimization is solved in \cite{Yen:MM_VariableSelectionSpikeSlab_Stat2011} by using  Majorization-Minimization Methods. Further, a relaxation of $\ell_0$ to $\ell_1$ norm reduces the problem to the well-known Elastic-Net \cite{Zou:AdapElasticNet_AnnalStat2009}. The framework in \eqref{Eq:MainOptProbICR} therefore offers greater generality in capturing the sparsity of $\vect x$.  As an example, consider the scenario in a reconstruction or classification problem where some dictionary (training) columns are more important than others\cite{Mohammadi:PCADicLearningSRC_Elsevier2014}. It is then possible to encourage their contribution to the linear model by assigning higher values to the corresponding $\kappa_i$'s, which in turn makes it more likely that the $i^{th}$ coefficient $x_{i}$ becomes activated.

We also have to mention that the goal of ICR is NOT to recover the sparsest possible solution but it is to recover a meaningful sparse signal. The word meaningful can be interpreted differently in various  contexts. For example in images, it is required from ICR to generate images that have structure. Note that the signal $\vect x$ that we recover is not only sparse as it is captured by the spike-and-slab prior but also may satisfy additional properties that are physically meaningful required by the application. For example the smoothness regularizer (prior) $\|\vect  x\|_2$ in images generates output results that are pleasant images.

\section{Iterative Convex Refinement (ICR)}
We first develop a solution to \eqref{Eq:MainOptProbICR} for the case when the entries of $\vect x$ are non-negative. Then, we propose our method in its general form with no constraints.

The central idea of the proposed Iterative Convex Refinement (ICR) algorithm -- see Algorithm 1 -- is to generate a sequence of optimization problems that refines the solution of previous iteration based on solving a modified convex problem. At  iteration $n$ of ICR, the indicator variable $\gamma_i$ is replaced with the normalized ratio $\frac{x_i}{\mu_i^{(n-1)}}$ and the convex optimization problem in \eqref{Eq:NonNegOptProb} is solved which is a simple quadratic programming with non-negativity constraint. Note that, $\mu_i^{(n-1)}$ is intuitively the average value of optimal $x_i^{\ast}$'s obtained  from iteration $1$ up to $n-1$ and is rigorously defined as in \eqref{Eq:UpdateMu}.
The motivation for this substitution is that, if the sequence of solutions $\vect x^{(n)} $ converges to a point in $\mathbb{R}^p$ we also expect $\frac{x_i}{\mu_i^{(n-1)}}$  to converge to $\gamma_i$. Essentially, ICR is solving a sequence of convex quadratic programming problem that their solution converges to a sub-optimal solution of \eqref{Eq:MainOptProbICR}.

To generalize ICR to the unconstrained case, a simple modification is needed at each iteration. In fact, at each iteration \eqref{Eq:UnconsOptProb} is solved instead of \eqref{Eq:NonNegOptProb}. Note that \eqref{Eq:UnconsOptProb} is still convex and we solve it by alternating direction method of multipliers \cite{Boyd:ADMM_MachineLearn2011}. Again we expect the ratio $\frac{|x_i|}{ |\mu_i^{(n-1)} |}$ to converge to the value of optimal $\gamma_i$ and the result of ICR be a sub-optimal solution for \eqref{Eq:MainOptProbICR}. ICR in both its versions is summarized in Algorithm 1\footnote{The Matlab code for ICR is made available online at {\url{http://signal.ee.psu.edu/ICR/ICRpage.htm}}}.
\begin{algorithm}[t]
\caption{Iterative Convex Refinement (ICR)  }
\label{Alg:NonNeg}
\begin{algorithmic}
\REQUIRE $\mat A, \vect \kappa,  \vect y $.\\
\emph{initialize: } $\vect \mu^{(0)} = \mat A^T \vect y $, iteration index $n=1$.
\WHILE{Stopping criterion not met }
\STATE(1) Solve the convex optimization problem at iteration $n$:
\STATE(Non-negative) For non-negative ICR solve
\bea
	\vect x^{(n)}  =   \arg\min_{\vect x  \succcurlyeq \vect 0} ||\vect y - \mat A\vect x ||_2^2 + \lambda ||\vect x||_2^2 +  \sum_{i=1}^{p}  \rho_{i}   \frac{x_{i}}{\mu_{i}^{(n-1)}} \label{Eq:NonNegOptProb}
\eea
\STATE(Unconstrained) For unconstrained ICR solve
\bea
	\vect x^{(n)}  =   \arg\min_{\vect x} ||\vect y - \mat A\vect x ||_2^2 + \lambda ||\vect x||_2^2 +  \sum_{i=1}^{p}  \rho_{i}   \frac{|x_{i}|}{\big|\mu_{i}^{(n-1)}\big|} \label{Eq:UnconsOptProb}
\eea
\STATE(2) \vspace{-0.38in}
\bea
    \text{~~~~Update~} \mu_i^{(n)}:~~~~~~~~~~~~   \mu_{i}^{(n)} = \frac{1}{n} \sum_{k=1}^{n} x_{i}^{(k)} ~~~ i=1,...,p~~~~~~~~~~~~~~~~~~~~~~~~~\label{Eq:UpdateMu}
\eea
\STATE(3) Increase iteration index $n$.
\ENDWHILE{ if $ ||\vect x^{(n)}-\vect x^{(n-1)}|| \le tol$}
\ENSURE $\vect x^\ast = ~\vect x^{(n-1)} ,~  \gamma_i ^\ast = \frac{x_i^\ast}{\mu_i^{(n-1)}}$ for all $i=1,...,p$.
\end{algorithmic}
\end{algorithm}

To analyze the convergence properties of ICR, we first define the function $f_n : \mathbb{R}^p \rightarrow \mathbb{R}$ as follows: 
\bea
	f_n(\vect x) =  \vect x^T(\mat A^T \mat A +\lambda \mat I)\vect x  -2\vect y^T \mat A  \vect x  			+ \sum_{i=1}^{p} \frac{\rho_i}{\big|\mu_i^{(n-1)}\big|} |x_i|		\label{Eq:f_n}
\eea
which is another form of the functions to be minimized at each iteration of ICR. For the rest of our analysis, without loss of generality we assume that $|y_i|\le 1$, $i=1...q$,   $|x_i|\le 1$, $i=1...p$ and columns of $\mat A$ have unity norm.
With this definition and assuming $\alpha$ is a constant that $\alpha <\frac{1}{2(q+p)}$, we propose the following two lemmas:
\begin{lemma}
If \big|$\mu_j^{(n_0)}\big|< \alpha \rho_j$, then $x_j^{(n_0+1)} = 0$. ($\gamma_j \thickapprox \frac{x_j}{\mu_j^{(n_0)}} = 0$)
\end{lemma}
\begin{proof}
Assume that for a specific $j$, $\big|\mu_j^{(n_0)}\big|< \alpha \rho_j$. Then for the next iteration the cost function to be minimized is as follows:
\bea
	f_{n_0+1}(\vect x) =  \vect x^T(\mat A^T \mat A +\lambda \mat I)\vect x  -2\vect y^T \mat A  \vect x  			+ \sum_{i=1}^{p} \frac{\rho_i}{\big|\mu_i^{(n_0)}\big|} |x_i|		\label{Eq:f_n+1}
\eea
Assume that the argument that minimizes \eqref{Eq:f_n+1} is $\vect x^{(n_0+1)}$. we can rewrite it in the following form:
\bea
    \vect x^{(n_0+1)} = \vect x_b + x_j \vect e_j
\eea
where $\vect e_j$ is the $j^{th}$ basis function with one at component $j$ and zeros elsewhere. $x_j$ is the $j^{th}$ element of $\vect x^{(n_0+1)}$ and $\vect x_b$ is equal to $\vect x^{(n_0+1)}$ except at $j^{th}$ element which is zero. We prove that if $\big|\mu_j^{(n_0)}\big|< \alpha \rho_j$, then $x_j = 0$.
\beaa
    f_{n_0+1}(\vect x_b)  &=& \vect x_b^T(\mat A^T \mat A +\lambda \mat I)\vect x_b  -2\vect y^T \mat A  \vect x_b  + \sum_{i=1}^{p} \frac{\rho_i}{\big|\mu_i^{(n_0)}\big|} |x_{b_i}|		\\
    f_{n_0+1}(\vect x^{(n_0+1)}) &=& ( \vect x_b+x_j \vect e_j )^T(\mat A^T \mat A +\lambda \mat I)( \vect x_b+x_j \vect e_j )   \\
    & &-2\vect y^T \mat A  ( \vect x_b+x_j \vect e_j ) + \sum_{i=1}^{p} \frac{\rho_i}{\big|\mu_i^{(n_0)}\big|} |x_{b_i}| + \frac{\rho_j}{\big|\mu_j^{(n_0)}\big|} |x_j|
\eeaa
Therefore, their difference is:
\bea
    f_{n_0+1}\big (\vect x^{(n_0+1)}\big) - f_{n_0+1}\big(\vect x_b\big)  &=& x_j ^2 \vect e_j^T  (\mat A^T \mat A + \lambda \mat I) \vect e_j +2 x_j \vect x_b^T (\mat A^T \mat A + \lambda \mat I) \vect e_j \nonumber\\
     &&- 2 x_j \vect y^T \mat A \vect e_j + \frac{\rho_j}{\big|\mu_j^{(n_0)}\big|} |x_j| \nonumber\\
    &=& \big |x_j \big| \Big ( |x_j| (\mat A^T \mat A +\lambda \mat I)_{jj} + \frac{\rho_j}{\big|\mu_j^{(n_0)}\big|}\Big ) \nonumber \\
    &&- 2 x_j \Big(  \vect y^T \mat A \vect e_j - \vect x_b^T(\mat A^T \mat A + \lambda \mat I)\vect e_j \Big) \label{Eq:difference}
\eea
We want to show that this difference is always positive except for $x_j=0$ which means $x_j$ must be zero in order for $f_{n_0+1}\big (\vect x^{(n_0+1)}\big)$ to be minimum. To do so, we show the following statements are true for nonzero $x_j$:
\beaa
    \Big| 2 x_j \Big(  \vect y^T \mat A \vect e_j - \vect x_b^T(\mat A^T \mat A + \lambda \mat I)\vect e_j \Big) \Big| <\big |x_j \big| \Big ( |x_j| (\mat A^T \mat A +\lambda \mat I)_{jj} + \frac{\rho_j}{\big|\mu_j^{(n_0)}\big|}\Big )
\eeaa
\bea
    \Leftrightarrow 2\Big|    \vect y^T \mat A \vect e_j - \vect x_b^T\mat A^T \mat A \vect e_j + \lambda \vect x_b^T \vect e_j \Big| &< &  |x_j| (\mat A^T \mat A +\lambda \mat I)_{jj} + \frac{\rho_j}{\big|\mu_j^{(n_0)}\big|} \nonumber\\
    \Leftrightarrow ~~~2 \Big| \vect y^T \mat A \vect e_j - \vect x_b^T\mat A^T \mat A\vect e_j \Big| &<&
    (1+\lambda)|x_j| + \frac{\rho_j}{\big|\mu_j^{(n_0)}\big|} \nonumber\\
    \Leftrightarrow ~~~~~~~2 \Big| (\vect y -\mat A \vect x_b)^T  \mat A \vect e_j \Big| &<&
    (1+\lambda)|x_j| + \frac{\rho_j}{\big|\mu_j^{(n_0)}\big|} \label{Eq:ineq}
\eea
In the above derivations, we used the fact that $(\mat A^T \mat A)_{jj} = 1$ since columns of $\mat A$ have unity norm. On the other hand, Cauchy-Schwarz inequality implies that,
\beaa
    2 \Big| (\vect y -\mat A \vect x_b)^T  \mat A \vect e_j \Big| &\le~~ 2||\vect y -\mat A \vect x_b||.||\mat A \vect e_j || &= ~~ 2||\vect y -\mat A \vect x_b|| \\
    &\le~~~ 2(||\vect y|| + ||\mat A \vect x_b||) &\le ~~2(\sqrt{q}+p)
\eeaa
Last inequality holds because of the fact that we assumed that magnitude of $x_i$ and $y_i$ do not exceed one. Also since we assumed $\big|\mu_j^{(n_0)}\big|< \alpha \rho_j$ and by definition of $\alpha$ we have:
\beaa
    (1+\lambda)|x_j| + \frac{\rho_j}{\big|\mu_j^{(n_0)}\big|} \ge \frac{1}{\alpha} \ge 2(q+p) \ge 2(\sqrt{q}+p)
\eeaa
Therefore, \eqref{Eq:ineq} is always true, since the right hand side is always greater than the left hand side. This implies that \eqref{Eq:difference} is positive for nonzero $x_j$ and, hence we must have $x_j$ = 0. Otherwise, it would contradict the fact that $f\big(\vect x^{(n_0+1)}\big)$ is the minimum value. Note that these are loose bounds and in practice they are easily satisfied. For example, $||\vect y - \mat A \vect x_b||$ is practically very small.
\end{proof}
This lemma also implies that if $\big|\mu_j^{(n)}\big|< \alpha \rho_j$ for some $n$, then $x_j$ will remain zero for all the following iterations.
\emph{Remark:} This is a very interesting result and may be potentially useful in updating $x_j$. During the iterations, the moment one of the $x_j$'s goes below $\alpha \rho _j $, $x_j$ should stay zero for all the following iterations. We can use this property and bring the $j^{th}$ component of $\vect x$ and the corresponding column from matrix $A$ out of the model and reduce the size of the model. This can significantly expedite the ICR algorithm.

\begin{lemma}
If $\big|\mu_j^{(n)}\big| \ge \alpha \rho_j$ for all $n \ge n_0$, then there exists $N_j\ge n_0$ such that for all $n>N_j$ we have
\bea
	\bigg |\frac{1}{\big | \mu_j^{(n+1)} \big |} - 	\frac{1}{\big |\mu_j^{(n)} \big |} \bigg | \le \frac{c}{n+1}
\eea
where $c$ is some positive constant.
\end{lemma}
Another interpretation of this lemma is that as the number of iterations grows, the cost functions at each iteration of ICR get closer to each other. In view of these two lemmas, we can show that the sequence of optimal cost function values obtained from ICR algorithm forms a Quasi-Cauchy sequence \cite{Burton:QuasiCauchy_AmericanMAth2010}. In other words, this is a sequence of bounded values that their difference at two consecutive iterations gets smaller.
\begin{proof}
%
%
Assume $\big|\mu_j^{(n)}\big| \ge \alpha\rho_j = \epsilon$.
First, note that it is straightforward to see that the difference of consecutive average values has the following property: $-\frac{1}{n+1}\le\big|\mu_j^{(n+1)}\big| - \big|\mu_j^{(n)}\big| \le \frac{1}{n+1}  $. Now, let $N_j = \frac{2}{\alpha\rho_j}$, then for all $n>N_j$ we have:
\bea
    \big|\mu_j^{(n)}\big| - \frac{1}{n+1} \le \big|\mu_j^{(n+1)}\big| \le \big|\mu_j^{(n)}\big| + \frac{1}{n+1}  \label{Eq:mu_jRange}
\eea
where the left hand side is  positive, since
\beaa
    \big|\mu_j^{(n)}\big| - \frac{1}{n+1} \ge \alpha \rho_j - \frac{1}{N_j} =\frac{\alpha \rho_j}{2} =\delta > 0 \label{eq:Temp}
\eeaa
Using this fact and \eqref{Eq:mu_jRange} we infer that:
\beaa
    &&\frac{1}{\big|\mu_j^{(n)}\big| + \frac{1}{n+1}}  \le \frac{1}{\big|\mu_j^{(n+1)}\big|} \le \frac{1}{\big|\mu_j^{(n)}\big| - \frac{1}{n+1}}\\
    &\Rightarrow  &\frac{1}{\big|\mu_j^{(n)}\big| + \frac{1}{n+1}} -\frac{1}{\big|\mu_j^{(n)}\big|} \le \frac{1}{\big|\mu_j^{(n+1)}\big|} -\frac{1}{\big|\mu_j^{(n)}\big|} \le \frac{1}{\big|\mu_j^{(n)}\big| - \frac{1}{n+1}} - \frac{1}{\big|\mu_j^{(n)}\big|}\\
    &\Rightarrow  &\frac{-\frac{1}{n+1}}{\Big(\big|\mu_j^{(n)}\big| + \frac{1}{n+1}\Big) \big|\mu_j^{(n)}\big| } \le \frac{1}{\big|\mu_j^{(n+1)}\big|} -\frac{1}{\big|\mu_j^{(n)}\big|} \le \frac{\frac{1}{n+1}}{\Big(\big|\mu_j^{(n)}\big| - \frac{1}{n+1}\Big)\big|\mu_j^{(n)} \big|}
\eeaa
In the last expression, we have:
\beaa
    \textit{RHS} = \frac{1}{(n+1)\Big(\big|\mu_j^{(n)}\big| - \frac{1}{n+1}\Big)\big|\mu_j^{(n)} \big|} \le \frac{1}{(n+1)\epsilon \delta}\\
    \textit{LHS} = \frac{-1}{(n+1)\Big(\big|\mu_j^{(n)}\big| + \frac{1}{n+1}\Big) \big|\mu_j^{(n)}\big| } \ge \frac{-1}{(n+1)\epsilon \delta}
\eeaa
Therefore,
\beaa
    \bigg|\frac{1}{\big|\mu_j^{(n+1)}\big|} -\frac{1}{\big|\mu_j^{(n)}\big|} \bigg| \le \frac{1}{(n+1)\epsilon\delta}, ~~~~ n>N_j.
\eeaa
\end{proof}
\begin{theorem}
After a sufficiently large $n$, the sequence of optimal cost function values obtained from ICR forms a Quasi-Cauchy sequence. i.e. $a_n = f_n(\vect x^{(n)})$  is a Quasi-Cauchy sequence of numbers.
\bea
    \big| f_{n+1}(\vect x^{(n+1)}) - f_n(\vect x^{(n)}) \big| \le \frac{c'}{n}. \label{Eq:Theorem1}
\eea
\end{theorem}
\begin{proof}
Before proving the proof, note that we can assume for a sufficiently large $N_0$, if $n\ge N_0$, then  $\big|\mu_j^{(n)}\big|$ is either always less than $\alpha\rho_j$ or always greater. Because according to Lemma 1, we know that if $\big|\mu_j^{(n)}\big|$ once becomes smaller than $\alpha\rho_j$ for some $n$, it will remain less than $\alpha\rho_j$ for all the following iterations. Therefore, let $n_j, ~j=1...p$ be the iteration index that for all $n>n_j$, $ \big|\mu_j^{(n)}\big| <\epsilon$. Note that some $n_j$'s may be equal to infinity which means they are never smaller than $\epsilon$. For those $j$ that $n_j=\infty$, let $N_j$ to be the same as $N_j$ defined in proof of Lemma 2.
With these definitions, we now proceed to prove the Theorem. We first show that for $n>N_0 = \max(\max_j{n_j},\max_j{N_j})$, the sequence of $f_n(\vect x^{(n)})$ has the following property:
\bea
    &&\big| f_{n+1}(\vect x^{(n)}) - f_{n}(\vect x^{(n)})\big| = \Big| \sum_{i=1}^{p} \rho_i\Big( \frac{1}{\big| \mu_i^{(n)} \big|} - \frac{1}{\big| \mu_i^{(n-1)} \big|} \Big)|x_i| \Big| \nonumber\\
    &\le& \sum_{|\mu_i^{(n-1)}|<\epsilon} \rho_i\Bigg| \frac{1}{\big| \mu_i^{(n)} \big|} - \frac{1}{\big| \mu_i^{(n-1)} \big|} \Bigg||x_i| + \sum_{|\mu_i^{(n-1)}|\ge\epsilon} \rho_i\Bigg| \frac{1}{\big| \mu_i^{(n)} \big|} - \frac{1}{\big| \mu_i^{(n-1)} \big|} \Bigg||x_i|\nonumber\\
    &\le &\sum_{|\mu_i^{(n-1)}|\ge\epsilon} \rho_i\frac{c}{n}|x_i| ~\le~ p \max{\{\rho_i\}} \frac{c}{n} ~\le~ \frac{c'}{n}. \label{Eq:Prop1}
\eea
This property also holds for $\vect x^{(n+1)}$. Finally, We show that for $n>N_0$, $a_n = f_n(\vect x^{(n)})$ is Quasi-Cauchy.
Since the minimum value $f_{n+1}(\vect x^{(n+1)})$ is smaller than $f_{n+1}(\vect x^{(n)})$, we can write:
\beaa
    f_{n+1}(\vect x^{(n+1)}) - f_{n}(\vect x^{(n)}) \le f_{n+1}(\vect x^{(n)}) - f_{n}(\vect x^{(n)}) \le \frac{c'}{n}
\eeaa
where we used \eqref{Eq:Prop1} for $n>N_0$. With the same reasoning for $n>N_0$ we have:
\beaa
    f_{n+1}(\vect x^{(n+1)}) - f_{n}(\vect x^{(n)}) \ge f_{n+1}(\vect x^{(n+1)}) - f_{n}(\vect x^{(n+1)}) \ge  -\frac{c'}{n}
\eeaa
Therefore,
\beaa
    \big| f_{n+1}(\vect x^{(n+1)}) - f_{n}(\vect x^{(n)}) \big| \le \frac{c'}{n}
\eeaa
for $n>N_0$.
\end{proof}
\emph{Remark:} Despite the fact that analytical results show a decay of order $\frac{1}{n}$ in difference between consecutive optimal cost function values, ICR shows much faster convergence in practice.
Combination of this theorem with a reasonable \emph{stopping criterion} guarantees the termination of the ICR algorithm. The stopping criteria used in this case is the norm of difference in the solutions $\vect x^{(n)}$ in consecutive iterations. At termination where the solution converges, the ratio $\frac{x_i}{\mu_i^{(n)}}$ will be zero for zero coefficients and approaches 1 for nonzero coefficients, which matches the value of $\gamma_i$ in both cases.
\section{Experimental Validation}
We now apply the ICR method to sparse  recovery problem. Two experimental scenarios are considered: 1.) synthetic data and 2.) a real-world image recovery problem. In each case, comparisons are made against state of the art alternatives. 

\noindent \emph{Synthetic data:}  We set up a typical experiment for sparse recovery as in \cite{Mohimani:fast_l_0_TSP2009, Yen:MM_VariableSelectionSpikeSlab_Stat2011} with a randomly generated Gaussian matrix $\mat A \in \mathbb{R}^{q\times p}$ and a sparse vector $\vect x_0 \in \mathbb{R}^p$. Based on $\mat A$ and $\vect x_0$, we form the observation vector $\vect y\in \mathbb{R}^q$ according to the additive noise model: $ \vect y = \mat A \vect x_0 + \vect n$ with $\sigma = 0.01$. The competitive state-of-the-art methods for sparse recovery that we compare against are: 1) SpaRSA \cite{Wright:SpaRSA_TSP2009,SPARSA:Code_Online} which is a  powerful method to solve the problems of the form \eqref{Eq:MainOptProbICR} 
2) \emph{Yen et al.} framework,  Majorization Minimization (MM) algorithm \cite{Yen:MM_VariableSelectionSpikeSlab_Stat2011} 3) Elastic Net \cite{Zou:AdapElasticNet_AnnalStat2009} 
4) FOCUSS algorithm \cite{Gorodnitsky:L1ReweightFOCUSS_TSP1997} which is a reweighted $\ell_1$ algorithm for sparse recovery \cite{Candes:ReweightedL1_FourierAnalysis2008} 5) expectation propagation approach for spike and slab recovery (SS-EP) \cite{Hernandez:EP-SpikeSlab_MachineLearning2014} and  finally 6) Variational Garrote (VG) \cite{Kappen:VariationalGarrote_MachineLearning2014}. 
Initialization for all methods is consistent as suggested in \cite{SPARSA:Code_Online}.

Table \ref{Tab:table1} reports the   experimental results for a small scale problem. We chose to first report results on a small scale problem in order to be able to use the IBM ILOG CPLEX optimizer \cite{CPLEX:IBM_Online} which is a very powerful optimization toolbox for solving many different optimization problems. It can also find the \emph{global} solution  to non-convex and mixed-integer programming problems. We used this feature of CPLEX  to compare ICR's solution with the global minimizer. For obtaining the results in Table \ref{Tab:table1}, we choose $p=64$, $q=32$ and the sparsity level of $\vect x_0$ is $10$. We generated $1000$ realizations of $\mat A, \vect x_0$ and $\vect n$ and recovered $\vect x$ using different methods. Two different types of figures of merit are used for evaluation of different sparse recovery methods: First, we compare different methods in terms of cost function value averaged over realizations, which is a direct measure of the quality of the solution to \eqref{Eq:MainOptProbICR}. Second, we compare performances from the sparse recovery viewpoint, and used the following figures of merit: mean square error (MSE) with respect to the global solution ($\vect x_g$) obtained by CPLEX optimizer, ``Support Match'' (SM) measure indicating how much the support of each solution matches to that of $\vect x_g$. However, cost function values and comparisons with global solution are not provided for SS-EP and VG since they are  not direct solutions to the optimization problem in \eqref{Eq:MainOptProbICR}.

As can be seen from Table \ref{Tab:table1}, ICR outperforms the competing methods in many different aspects. In particular from the first row, we infer that ICR is a better solution to \eqref{Eq:MainOptProbICR} since it achieves a better minimum in average sense. Moreover, significantly higher support match (SM $= 97.13\%$ ) measure for ICR shows that ICR's solution shows much more agreement with the global solution. Finally, the ICR solution is also the closest to the global solution obtained from CPLEX optimizer in the sense of MSE (by more than one order of magnitude).
    \begin{table}[t]
      \centering
      \caption{Comparison of methods for $p=64$ and $q=32$.}
      \label{Tab:table1}
      \begin{tabular}{|l||c|c|c|c|c|}
      \hline
      Method                                &SpaRSA          & MM            & E-Net          & FOCUSS          & ICR            \\
      \hline\hline
      Avg $f(\vect x^*)$                    &2.05E-2         &1.52E-2        &3.33E-2         & 3.89E-2         &\textbf{1.45E-2}            \\
      \hline
      MSE  vs. $\vect x_g$                  &1.07E-3         &5.49E-3        &2.45E-4         & 1.55E-4         &\textbf{8.45E-5}           \\
      \hline
      SM            vs. $\vect x_g$ ($\%$)   &81.57           &80.25          &70.20           & 90.53          &\textbf{97.13}              \\
      \hline
      \end{tabular}
    \end{table}
    \begin{table}
      \centering
      \caption{Comparison of methods for $p=512$ and $q=128$.}
      \label{Tab:table2}
      \begin{tabular}{|l||c|c|c|c|c|}
      \hline
      Method                                &VG               & MM            & FOCUSS        & SS-EP           & ICR        \\
      \hline\hline
      Avg $f(\vect x^*)$                    &--               &1.02E-1        &8.31E-2         & --              &\textbf{6.72E-2}            \\
      \hline
       MSE vs. $\vect x_0$                  &3.65E-4          &1.89E-3        &3.69E-4         & 3.50E-4         &\textbf{2.39E-4}              \\
      \hline
      Sparsity Level                        &16.82            &79.32          &21.37           & 21.68           &\textbf{28.88}              \\
      \hline
      SM vs. $\vect x_0$ ($\%$)             &89.45            &84.12          &94.17           & 93.90           &\textbf{95.41}              \\
      \hline
      Time (sec)                            &0.82             &3.38           &3.01            & 5.37            &3.15                 \\
      \hline
      \end{tabular}
    \end{table}

\begin{figure}
  \centering
  \includegraphics[width=\textwidth]{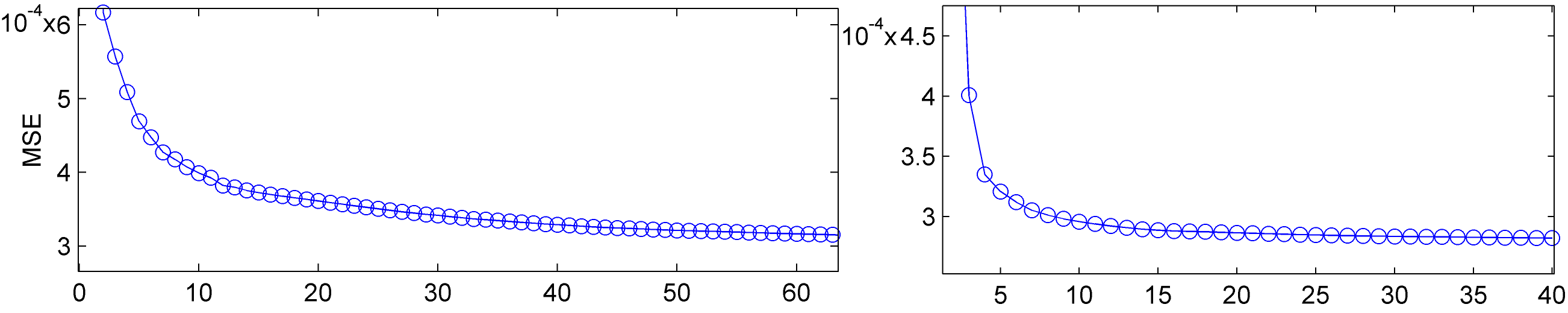}
  \caption{Convergence  of ICR (right) and ICR-NN(left). MSE vs. $\#$ of iteration.  }
  \label{Fig:Convergence1}
\end{figure}

Next, we present results for a typical larger scale problem. We chose $p=512$, $q=128$ and set the sparsity level of $\vect x_0$ to be $30$ and carry out the same experiment as before. Because of the scale of the problem, the global solution is now unavailable and therefore, we compare the results against $\vect x_0$ which is the ``ground truth". Results are reported in Table \ref{Tab:table2}. Table \ref{Tab:table2} also additionally reports the average sparsity level of the solution  and it can be seen that the sparsity level of ICR is the closest to the true sparsity level of $\vect x_0$. In all other figures of merit, viz. the cost function value (averaged over realizations), MSE and support match vs. $\vect x_0$, ICR is again the best. Fig. \ref{Fig:Convergence1} additionally shows the convergence plots for ICR and ICR-NN respectively.

\noindent \emph{Image reconstruction:} In this part we aim to apply our ICR algorithm to real data for reconstruction of handwritten digit images from the well-known MNIST dataset \cite{Lecun:MNIST_Online}. The MNIST dataset contains 60000 digit images ($0$ to $9$) of size $28\times 28$ pixels. Most of pixels in these images are inactive and zero and only a few take non-zero values. Thus, these images are naturally sparse and fit into the spike and slab model. The experiment is set up such that a sparse signal $\vect x$ (vectorized image) is to be reconstructed from a smaller set of random measurements $\vect y$. For any particular image, we assume the random measurement (150 measurements) are obtained by a Gaussian measurement matrix $\mat A \in \mathbb{R}^{150\times784}$ with added noise according to \eqref{Eq:y=Ax}. We compare our result against the following state-of-the-art image recovery methods for sparse images: 1.) SALSA-TV which uses the variable splitting proposed by Figueiredo \emph{et al.} \cite{MarioAndAfonso:ImageRecoverySALSA_TIP2010}  combined with Total Variation (TV) regularizers \cite{Chambolle:TV_Minimization_MathImagVision2004}. 2.) A Bayesian Image Reconstruction (BIR) \cite{Lu:SparseCodeBayesPerspec_NeuralNetLearn2013}, based on a more recent version of Bayesian image reconstruction method \cite{DobigeonAndHero:HierarchyBayesImageRecons_TIP2009} proposed by Hero \emph{et al.}. We also compare our results with Adaptive Elastic Net method \cite{Zou:AdapElasticNet_AnnalStat2009} which is commonly used in sparse image recovery problems. Finally, the result of the non-negative ICR (ICR-NN) is shown which explicitly enforces a non-negativity constraint on $\vect x$ which in this case corresponds to the intensity of reconstructed image pixels. Recovered images are shown in Fig. \ref{Fig:Digits} and the corresponding average reconstruction error (MSE) for the whole database appears next to each method. Clearly, ICR and ICR-NN outperform the other methods both visually and based on MSE value. It is also intuitively satisfying that ICR-NN which captures the non-negativity constraint natural to this problem, provides the best result overall.
\begin{figure}
\centering
    \includegraphics[width=0.9\textwidth]{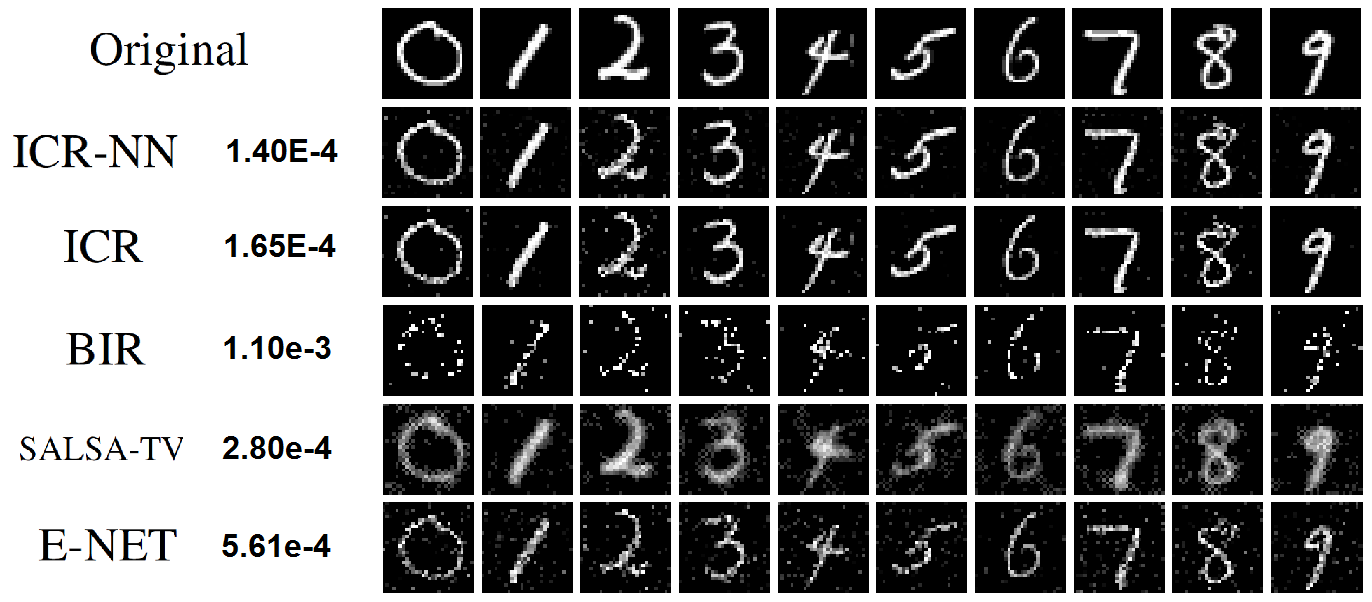}
    \caption{Examples of reconstructed images from MNIST dataset using different methods. The Numbers appeared next to each method is the average MSE.}
    \label{Fig:Digits}
\end{figure}

Following the same experimental setup for synthetic data, we illustrate the performance of the ICR in comparison with others as the sparsity level of $\vect x_0$ ($||\vect x_0||_0$) changes. We vary the true sparsity level from only $5$ non-zero elements in $\vect x_0$ up to $75$ and compared MSE and support match percentage of the solutions from each method. The length of sparse signal is chosen to be $p=512$ and number of observations is $q=128$. 

Fig. \ref{Fig:MSEandSMvsSPlevel2} shows an alternate result as the MSE plotted against the sparsity level; once again the merits of ICR are readily apparent. This Figure also illustrates that the support of ICR's solution is the closest to the support of $\vect x_0$. More than $90\%$ match between the support of ICR's solution  and that of $\vect x_0$ for a wide range of sparsity levels makes ICR very valuable to variable selection problems specially in Bayesian framework. Fig. \ref{Fig:SpLevelvsSPLevel} shows the actual sparsity level of solution for different methods. The dashed line corresponds to the true level of sparsity and ICR's solutions is the closest to the dashed line implying that the level of sparsity of ICR's solution matches the level of sparsity of $\vect x_0$ more than other methods. This also support the results obtained from Fig. \ref{Fig:MSEandSMvsSPlevel2}.

Fig. \ref{Fig:vsSNR}  illustrates the mean square errors (MSEs) and support match (SM) obtained from different methods under different SNRs. The chosen values for $\sigma$ are $0.05, 0.01, 0.001, 0.0001$.

In addition to these results, Table \ref{Tab:table3} shows the results of comparing ICR against a few other methods for sparse signal recovery. These methods are Sparse Recovery by Separable Approximation (SpaRSA) \cite{Wright:SpaRSA_TSP2009} 
and  Elastic Net \cite{Zou:AdapElasticNet_AnnalStat2009}.

\begin{figure}
  \centering
  \includegraphics[width=0.95\textwidth]{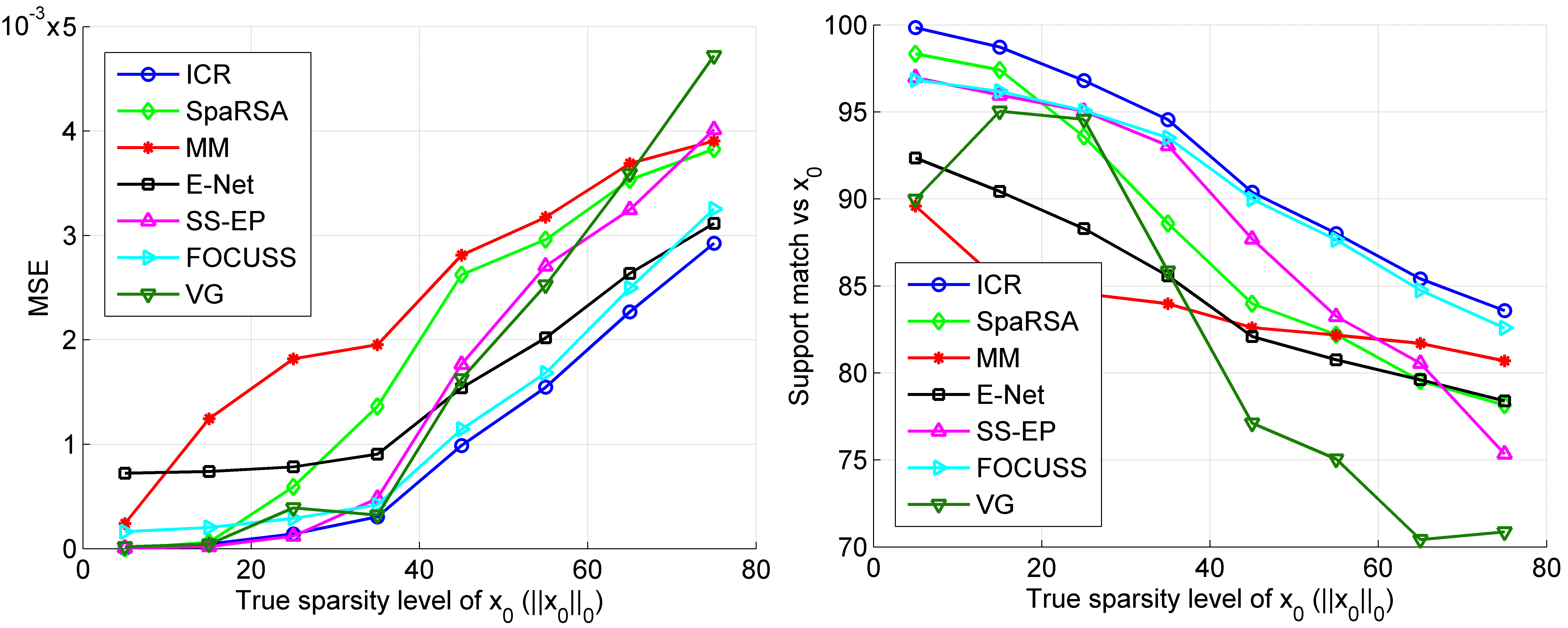}
  \caption{Comparison of  MSE (Left) and Support Match (SM) (Right) obtained by each method  versus sparsity level of $\vect x_0$.  }
  \label{Fig:MSEandSMvsSPlevel2}
\end{figure}
\begin{figure}
  \centering
  \includegraphics[width=0.48\textwidth]{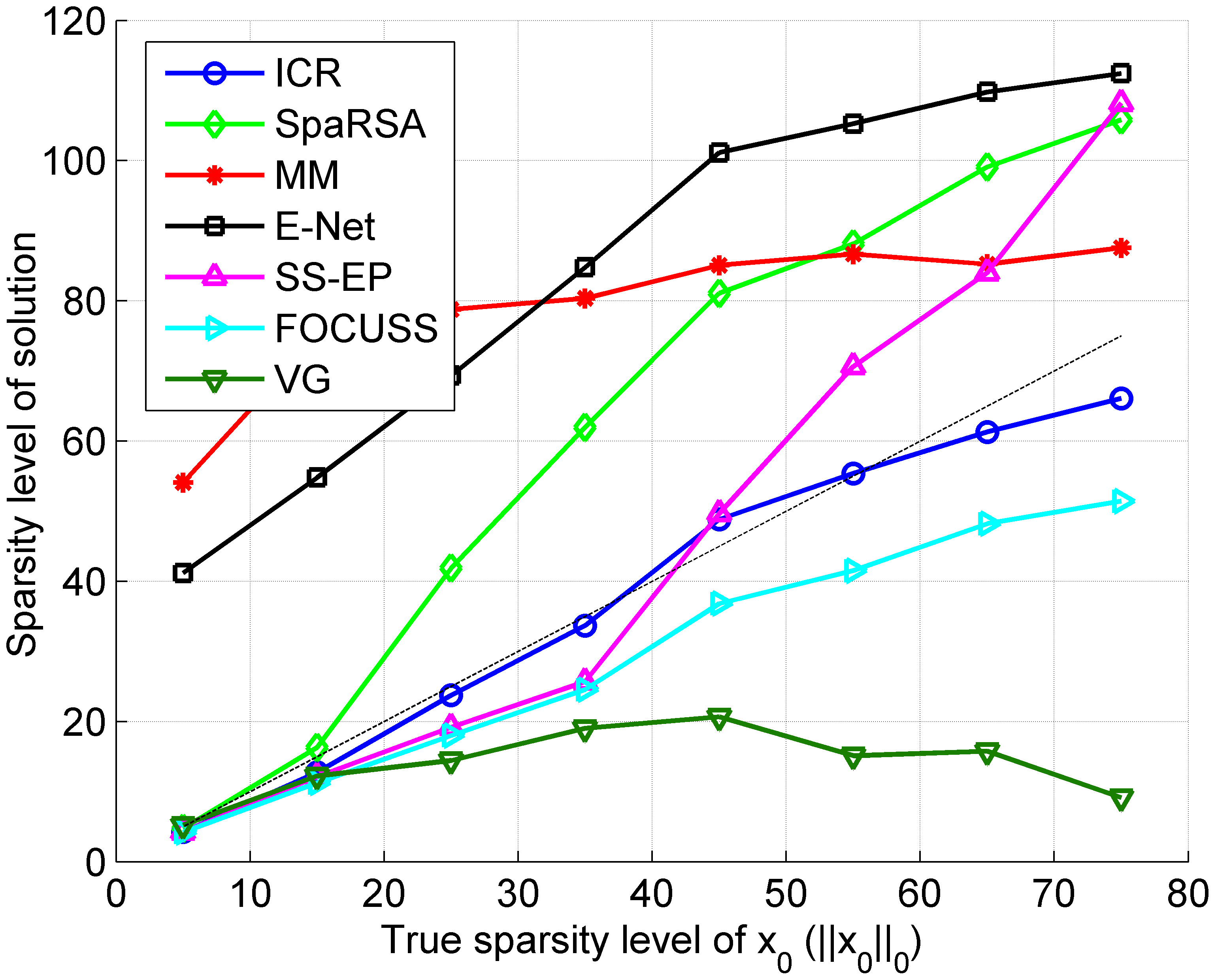}
  \caption{Comparison of average sparsity level obtained by each method versus sparsity level of $\vect x_0$. Dashed line shows the true level of sparsity  }
  \label{Fig:SpLevelvsSPLevel}
\end{figure}
%
      \begin{figure}
          \centering
          \includegraphics[width=0.95\textwidth]{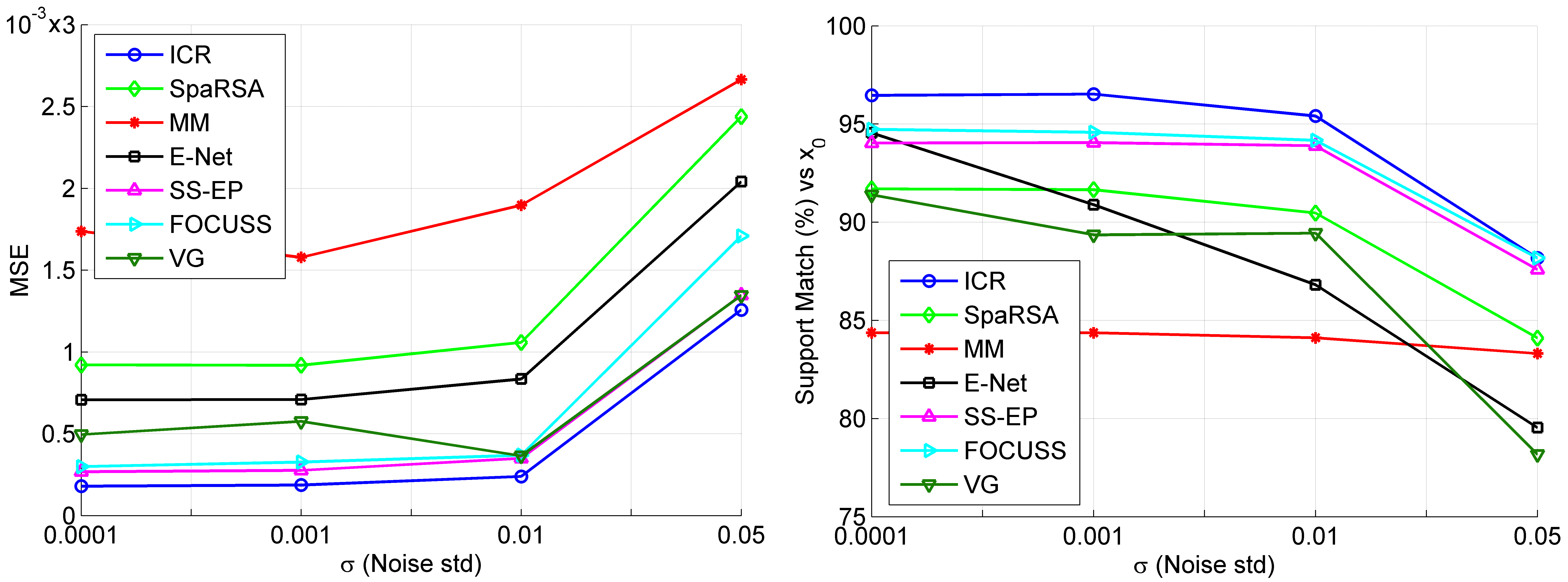}
          \caption{Comparison of  MSE (Left) and Support Match (SM) (Right) obtained by each method under various noise levels.}
          \label{Fig:vsSNR}
        \end{figure}
    \begin{table}[h]
      \centering
      \small{
      \caption{Comparison of methods for $p=512$ and $q=128$.}
      \label{Tab:table3}
      \begin{tabular}{|l||c|c|c|c|}
      \hline
      Method                                &SpaRSA             & Elastic-Net     & ICR        \\
      \hline\hline
      Avg $f(\vect x^*)$                    &8.75E-2          &9.77E-2          &\textbf{6.72E-2}            \\
      \hline
       MSE vs. $\vect x_0$                  &1.05E-3          &8.35E-4          &\textbf{2.39E-4}              \\
      \hline
      Sparsity Level                        &54.12            &77.81            &\textbf{28.88}              \\
      \hline
      SM vs. $\vect x_0$ ($\%$)             &90.47            &86.81           &\textbf{95.41}              \\
      \hline
      \end{tabular}
      }
    \end{table}

\section{Conclusions and Future Work}
\label{sec:concl}
In this chapter, we develop a novel algorithm (ICR) to optimize a hard non-convex cost function with applications in sparse recovery. Unlike existing approaches, ICR does not simplify the optimization by assumptions/relaxations and hence affords a more general sparse structure. Experiments on synthetic data as well as a real-world image recovery problem confirms practical merits of ICR.  Future research may investigate further analysis of ICR properties and extensions to multi-task scenarios.

In this line of research exploiting different sparsifying priors and extension to multi-task and collaborative signal recovery should be considered as future research direction. Also current research focus is about accelerating the ICR algorithm  and parameter learning for obtaining more accurate recovery results.

\chapter{Contribution II: Color Super Resolution via Exploiting Cross Channel Constraints}
\label{chapter:contrib2}

\section{Introduction}
\label{sec:intro}
Super resolution  is a  branch of image reconstruction and an active area of research that focuses on the enhancement of image resolution. Conventional Super-Resolution (SR) approaches require multiple Low Resolution (LR) images of the same scene as input and maps them to a High Resolution (HR) image based on some reasonable assumptions, prior knowledge, or capturing the diversity in LR images \cite{Freeman:ExampleBasedSR_CompGraph2002 ,Farsiu:MultiFrameSR_TIP2004, Park:ReviewSR_SPM2003}. This can be seen as an inverse problem of recovering the high resolution image (signal) by fusing the low resolution images of the scene. The recovered image should produce the same low resolution images if the physical image formation model is applied to the HR image. However, SR task is a severely  ill-posed problem since much information is lost in the process of going from high resolution images to low resolution images and hence the solution is not unique. Consequently, strong prior information is incorporated to yield realistic and robust solutions. Example priors include knowledge of the underlying scene, distribution of pixels, historical data, smoothness and edge information and so on so forth. \cite{Tappen:SparsePriorSR_2003,Fattal:SRstatistic_ACM2007,Dai:edgeSR_CVPR2007, Minaee:Segmen_ICIP2015}

In contrast to conventional super resolution problem with multiple low resolution images as input, single image super-resolution methods have been developed recently that generate the high resolution image only based on a \emph{single} low resolution image. Classically, the solution to this problem is based on example-based methods exploiting nearest neighbor estimations, where pairs of low and high resolution image patches are collected and each low resolution patch is mapped to a corresponding high resolution patch. Freeman \emph{et al.} \cite{Freeman:ExampleBasedSR_CompGraph2002} proposed an estimation scheme where high-frequency details are obtained by taking nearest neighbor based estimation on low resolution patches. Glasner \emph{et al.} \cite{Glasner:SRSingle2009ICCV} used the observation that patches in a natural image tend to redundantly recur many times inside the image, both within the same scale, as well as across different scales and approached the single image super resolution problem. An alternate mapping scheme was proposed by Kim \emph{et al.} \cite{Kim:SRNaturalPrior_PAMI2010} using kernel ridge regression.

Many learning techniques have been developed which attempt to capture the co-occurrence of low resolution and high resolution image patches. \cite{Sun:HallucinationSR_CVPR2003} proposed a Bayesian approach  by using Primal Sketch priors. Inspired by manifold forming methods like locally linear embedding (LLE), Chang \emph{et al.} \cite{Chang:NeighborEmbeddingSR_CVPR2004} proposed a neighbourhood embedding approach. Specifically, small image patches in the low and high resolution images form manifolds with similar local geometry in two distinct feature spaces and local geometry information is used to reconstruct a patch using its neighbors in the feature space.


More recently, sparse representation based methods have been applied to the single image super resolution problem. Essentially in these techniques, a historical record of typical geometrical structures observed in images is exploited and examples of high and low resolution image patches are collected as dictionary (matrix). Yang \emph{et al.} proposed to apply sparse coding for retrieving the high resolution image from the LR image \cite{YangAndWright:SparseSR_TIP2010}. Zeyde \emph{et al.} extended this method to develop a local Sparse-Land model on image patches \cite{Zeyde:SR_Springer2012}. Timofte \emph{et al.} proposed the Anchored Neighborhood Regression (ANR) method which uses learned dictionaries in combination with neighbor embedding methods \cite{Timofte:AnchoredANR_ICCV2013, Timofte:AnchoredARN+_ACCV2014}. Other super resolution methods based on statistical signal processing or dictionary learning methods have been proposed by \cite{Peleg:StatisticalSR_2014TIP, Zhou:KPCAcodingSR2015_SPL, Polatkan:BayesianSR_2015PAMI, Huang:TransformedExampleSR_2015cvpr, Zhang:MultiLinearMap_2015TIP, He:BetaProcessDL_CVPR2013}.

%

On top of sparsity based methods, learning based methods have also been exploited for SR problems to learn dictionaries that are more suitable for this task. Mostly, dictionary learning or example-based learning methods in super-resolution use an image patch or feature-based approach to learn the relationship between high resolution scenes and their low resolution counterparts.  Yang \emph{et al.} \cite{YangAndWright:RawSparseSR_CVPR2008} propose to use a collection of raw image patches as dictionary elements in their framework. Subsequently, a method that learns LR and HR dictionaries jointly was proposed in \cite{YangAndWright:SparseSR_TIP2010}. A semi-coupled dictionary learning (SCDL) model and a mapping function was proposed in  \cite{Wang:SemiCoupledDL_CVPR2012} where the learned dictionary pairs can characterize the structural features of the two image domains, while the mapping function reveals the intrinsic relationship between the two. In addition, coupled dictionary learning for the same problem was proposed in \cite{Yang:CoupledDicLearnSR_TIP2012}, where the learning process is modeled as a bilevel optimization problem. Dual or joint filter learning in addition to dual (joint)  dictionaries was developed by Zhang \emph{et al.} \cite{Zhang:DualDLSR_ICME2011}.

\subsection{Sparsity Based Single Image Super-Resolution}
In the setting proposed by Yang \emph{et al.} (ScSR) \cite{YangAndWright:SparseSR_TIP2010} a large collection of corresponding high resolution and low resolution image patches is obtained from training data. In this framework, the low resolution information can either be in the form of raw image patches, high frequency or edge information, or any other types of representative features, while high resolution information is in the form of image pixels to ensure reconstruction of high resolution images. Using methods mentioned for dictionary learning in SR task and sparsity constraints, high resolution and low resolution dictionaries are jointly learned such that they are capable of representing the LR image patches and their corresponding HR counterparts using the same sparse code. Once the dictionaries are learned, the algorithm searches for a sparse linear representation of each patch of LR image based on the following sparse coding optimization:
\bea
     \vect x^\ast  = \arg\min_{\vect x} ~~ \frac{1}{2}||\vect y_l - \mat D_l \vect x||_2^2 + \lambda || \vect x||_1 \label{Eq:ScSRopt}
\eea
where $\mat D_l$ is the learned low resolution dictionary (or dictionary that is learned based on features extracted from LR patches), $\vect x$ is the   sparse code representing the LR patch (or features extracted from LR patch) with respect to $\mat D_l$ and $\lambda$ is a regularizer parameter for enforcing the sparsity prior and regularizing the ill-posed problem. This is the familiar and famous LASSO \cite{Sprechmann:CHI-LASSO_TSP2011, Mousavi:ICR_SPL2015} problem which can be easily solved using any sparse solver toolbox. The high resolution reconstruction ($\vect y_h$) of each low resolution patch or features of the patch ($\vect y_l$) is then reconstructed using the same sparse code according to the HR dictionary as: $\vect y_h = \mat D_h \vect x^\ast $. Joint dictionary learning for SR considers the problem of learning two joint dictionaries $\mat D_l$ and $\mat D_h$ for two features spaces (low resolution and high resolution domains) which are assumed to be tied by a certain mapping function \cite{YangAndWright:SparseSR_TIP2010, Wang:SemiCoupledDL_CVPR2012}. The assumption is that $\vect x$, the sparse representation of $\vect y_l$ based on learned low resolution dictionary, should be  the same as that of $\vect y_h$ according to $\mat D_h$. The following optimization problem encourages this idea and learns low resolution and high resolution image dictionaries according to the same sparse code:
\bea
    \min_{\mat D_l, \mat D_h,\{\vect x^i\}}  \frac{1}{N} \sum_{i=1}^{N} \frac{1}{2}\| \vect y_l^i - \mat D_l \vect x^i \|_2^2 + \frac{1}{2}\| \vect y_h^i - \mat D_h \vect x^i \|_2^2 + \lambda \|\vect x^i\|_1  \nonumber \\
    \text{st.~~} \|\mat D_l(:,k) \|_2^2 \le 1, ~ \|\mat D_h(:,k) \|_2^2 \le 1, ~~k=1,2,...,K.
\eea
where $N$ is the number of training sample pairs and $K$ is the number of desired dictionary basis atoms. $\mat D(:,k)$ denotes the $k^{th}$ column of the matrix $\mat D$.

\subsection{Motivation and Contributions}
Most  super-resolution methods, especially in single image SR literature, have been designed to increase the resolution of a single channel (monochromatic) image. A related yet more challenging problem, color super-resolution, addresses enhancing resolution of color (multi-channel) low resolution images to increase their spatial resolution.
The typical solution for color super resolution involves applying SR algorithms to each of the color channels independently \cite{Shah:ColorVideoSR_TIP1999, Tom:ColorVideoSR_TIP2001}. Another approach which is more common is to transform the problem to a different color space such as YCbCr, where chrominance information is separated from luminance, and SR is applied only to the luminance channel \cite{Yang:CoupledDicLearnSR_TIP2012, YangAndWright:SparseSR_TIP2010, Timofte:AnchoredANR_ICCV2013} since   human eye is more sensitive to luminance information than chrominance information. Both of these methods are suboptimal for the color super-resolution problem as they do not fully exploit the complementary information that may exist in different color channels. The former ignores cross channel correlations and the latter despite taking into account some level of color correlation, does not methodically capture chrominance geometry. 
In particular, many images have key geometric components in chrominance channels. For instance, Fig. \ref{Fig:Motivation} illustrates a synthetic image where chrominance (Cb and Cr) edges are even more pronounced than those in the luminance channel (Y). In traditional multi-frame super resolution problem, color information has indeed been used in different ways to enhance super resolution results. Farsiu \emph{et al.} \cite{Farsiu:ColorDemosaicSR_TIP2006} proposed a multi-frame demosaicing and super resolution framework for color images using different color regularizers. Belekos \emph{et al.} proposed multi channel video super resolution in \cite{Belekos:MultiChannelVideoSR_2010TIP} and general color dictionary learning for image restoration is proposed in \cite{Mairal:SparsityImageRestoration_2008TIP, Xu:QuaternionColorSR_2015TIP, Jung:MumfordShahSR_2011TIP}. Other methods that use color channel information are proposed in \cite{Shen:SparsityColorSR_2013AppliedMechanics, Dai:SoftCutColorSR_2009TIP, Gong:NeighborEmbedColorSR_2011CompInfo, Wang:EdgeColorSR_2015APPC, Maalouf:GroupletColorSR_2009EuroSP, Liu:ColorizationForSR_2010ECCV, Cheng:RGBSR_2015IETImaPrc, Maalouf:GroupSR_2012IETImgPrc}.

In this work, we extend sparsity based super resolution by effectively using color priors to further enhance super resolution results and develop a sparsity based Multi-Channel (i.e.\ color) constrained Super Resolution (MCcSR) framework. \textbf{The key contributions of our work}\footnote{{Preliminary version of this work was presented at IEEE ICIP in September 2016 \cite{Mousavi:ColorSR_2016ICIP}}} are as follows \cite{mousavi2017sparsity}:
\begin{itemize}
    \item We explicitly address the problem of color image super-resolution by inclusion of color regularizers in the sparse coding for SR. These color regularizers capture the cross channel correlation information existing in different color channels and exploit it to better reconstruct super-resolution patches. The resulting optimization problem with added color-channel regularizers is not easily solvable and a tractable solution is proposed.
    \item The amount of color information is not the same in each region of the image and in order to be able to force color constraints we develop a measure that captures the amount of color information and then use it to balance the effect of color regularizers. Therefore, an adaptive color patch processing scheme is also proposed where patches with stronger edge similarities are optimized with more emphasis on the color constraints.
    \item In most dictionary learning algorithms for super-resolution, only the correspondence between low and high resolution patches is considered. However, we propose to learn dictionaries whose atoms (columns) are not only low resolution and high resolution counterparts of each other, but also in the high resolution dictionary in particular, we incorporate color regularizers such that the resulting learned high resolution patches exhibit high edge correlation across RGB color bands. 
    \item {\em Reproducibility:} All results in this chapter are completely reproducible. The MATLAB code as well as images corresponding to the SR results are made available at: {\url{http://signal.ee.psu.edu/MCcSR.html}}.
\end{itemize}

The rest of this chapter is organized as follows: In Section \ref{Sec:ColorSR}, we generalize the sparsity-based super resolution framework to multiple (color) channels  and motivate the choice of color regularizers. These color regularizers are used in Section \ref{Sec:ColorDL} to assist learning of color adaptive dictionaries suitable for color super resolution task. Section \ref{Sec:Experiments} includes experimental validation which demonstrates the effectiveness of our approach by comparing it with state-of-the-art image SR techniques. Concluding remarks are collected in Section \ref{Sec:Conclusion}.

\section{Sparsity Constrained Color Image Super Resolution}
\label{Sec:ColorSR}

\subsection{Problem Formulation}
\label{sec:formulation}

A characteristic associated with most natural images is strong correlation between high-frequency spatial components across the color (RGB) channels. This is based on the intuition that a luminance edge for example is spread across the RGB channels \cite{Farsiu:ColorDemosaicSR_TIP2006, Srinivas:ColorSR_CIC2011}. Fig. \ref{Fig:Motivation} illustrates this idea.

 We can hence encourage the edges across color channels to be similar to each other. Fig. \ref{Fig:MotivationEdge} also shows that RGB edges are far more close to each other than YCbCr edges. Such ideas have been exploited in traditional image fusion type super-resolution techniques \cite{Farsiu:ColorDemosaicSR_TIP2006}, yet sparsity-based single image super resolution lacks a concrete color super resolution framework.
\begin{figure}
  \centering
  \includegraphics[width=0.9\columnwidth]{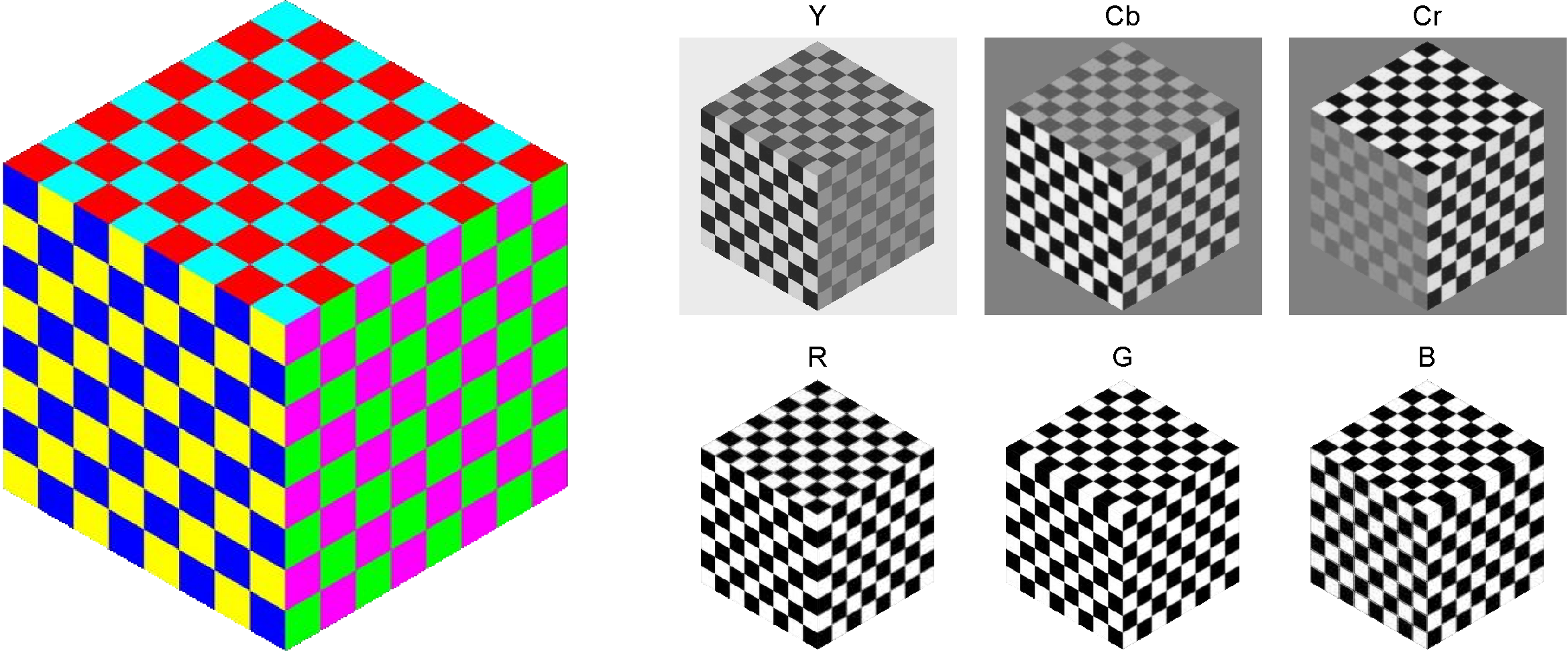}\\
  \caption{Color chessboard cube and  color channel components. }\label{Fig:Motivation}
\end{figure}
\begin{figure}
  \centering
  \centering
  \includegraphics[width=0.9\columnwidth]{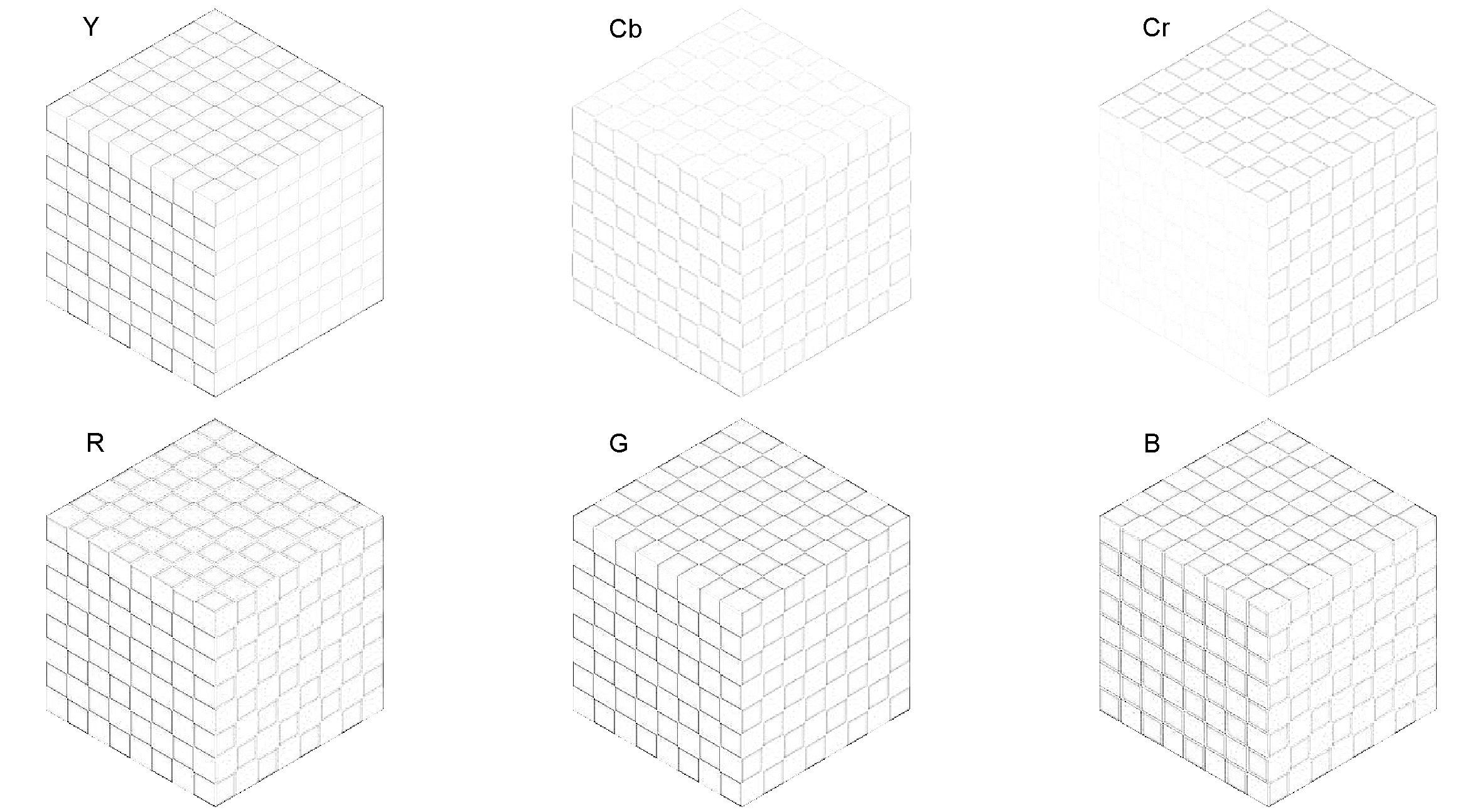}\\
  \caption{Edges for color channels of chessboard cube. }\label{Fig:MotivationEdge}
\end{figure}
Edge similarities across RGB color channels may be enforced in the following manner  \cite{Farsiu:ColorDemosaicSR_TIP2006, Keren:ColorSR_1999MachineVision, Menon:ColorDemosaick_2009TIP}.
\bea
        \|\mat S_\mu \vect y_{h_\mu} - \mat S_\nu \vect y_{h_\nu}\|_2 < \epsilon_{\mu\nu} &,~ \mu,\nu\in\{r,g,b\} &,~ \mu\ne\nu \label{Eq:CrossConstrains}
\eea
where $r,g$ and $b$ subscripts are indicating signals in R, G and B channels and $\mat S$ matrix is a high-pass edge detector filter as in \cite{Srinivas:ColorSR_CIC2011}. For instance, $\mat S_r \vect y_{h_r}$  illustrate the edges in red channel of the desired high resolution image. These constraints are essentially enforcing the edge information across color channels to be similar in {\em high resolution patches}. The underlying assumption here is that the high resolution patches need  to be known beforehand which is not true in practice. We recognize however that these constraints can be equivalently posed on the sparse coefficient vector(s) corresponding to the individual color channels, since: $\vect y_{h_r} = \mat D_{h_r} \vect x_r,~~~\vect y_{h_g} = \mat D_{h_g} \vect x_g,~~~\vect y_{h_b} = \mat D_{h_b} \vect x_b$.

Note that sparse codes for different color channels are no longer independent and they may be jointly determined by solving the following optimization problem:
\bea
     [\vect x_r,\vect x_g,\vect x_b]  &=&\arg\min \sum_{c \in \{r,g,b\}} \frac{1}{2} \|\vect y_{l_c} -\mat D_{l_c} \vect x_c\|_2^2 +  \lambda \|\vect x_c \|_1 \nonumber\\
&&     +\tau \Big[ \|\mat S_r \mat D_{h_r} \vect x_r - \mat S_g \mat D_{h_g} \vect x_g \|_2^2 \nonumber\\
&&     ~~+\|\mat S_g \mat D_{h_g} \vect x_g - \mat S_b \mat D_{h_b} \vect x_b\|_2^2 \nonumber\\
&&     ~~+\|\mat S_b \mat D_{h_b} \vect x_b-\mat S_r \mat D_{h_r} \vect x_r\|_2^2 \Big]. \label{Eq:MainOptProb}
\eea
where the cost function is equivalent to the following:
\bea
     L_1&=&   \sum_{c \in \{r,g,b\}} \Big[ \frac{1}{2} \|\vect y_{l_c} -\mat D_{l_c} \vect x_c\|_2^2 +  \lambda \|\vect x_c \|_1  \nonumber\\
     &&+2\tau \vect x_c^T \mat D_{h_c}^T \mat S_c^T \mat S_c \mat D_{h_c} \vect x_c  \Big ]     -2\tau \Big [  \vect x_r^T \mat D_{h_r}^T \mat S_r^T \mat S_g \mat D_{h_g} \vect x_g +   \nonumber\\
     && \vect x_g^T \mat D_{h_g}^T \mat S_g^T \mat S_b \mat D_{h_b} \vect x_b    +   \vect x_b^T \mat D_{h_b}^T \mat S_b^T \mat S_r \mat D_{h_r} \vect x_r  \Big ] \label{Eq:Cost2}
\eea
For simplicity, we assume the same regularization parameters $\tau$ and $\lambda$ for each of the edge difference terms and color channels. The high-pass edge detectors ($\mat S_r,\mat S_g,\mat S_b$) are also chosen to be the same for each color channel. It is worth mentioning that if  $\tau = 0$, \eqref{Eq:MainOptProb} reduces to three independent sparse coding problems (ScSR) for each color channel. With the cross channel regularization terms, these sparse codes are no longer independent and \eqref{Eq:MainOptProb} presents a challenging optimization problem in contrast with the optimization problem corresponding to single channel sparsity based super resolution. In the new problem, the additional color channel regularizers are of quadratic nature and make the optimization problem more challenging to solve. Next, we propose a tractable solution.

\subsection{Solution to the Optimization Problem}
We introduce the following vectors and matrices:
\beaa
    \vect x =
    \begin{bmatrix}
    \vect x_r  \\
    \vect x_g  \\
    \vect x_b  \\
    \end{bmatrix}_{3m\times 1}
    ,
    \vect y_l =
    \begin{bmatrix}
    \vect y_{l_r}  \\
    \vect y_{l_g}  \\
    \vect y_{l_b}  \\
    \end{bmatrix}_{3p\times 1}
    ,~
    \mat P =
    {\begin{bmatrix}
    \mat 0       & \mat 0       & \mat I      \\
    \mat I       & \mat 0       & \mat 0      \\
    \mat 0       & \mat I       & \mat 0 \\
    \end{bmatrix}}_{3m \times 3m}
\eeaa
\beaa
    \begin{bmatrix}
    \vect x_b  \\
    \vect x_r  \\
    \vect x_g  \\
    \end{bmatrix} =
    \underbrace{\begin{bmatrix}
    \mat 0       & \mat 0       & \mat I      \\
    \mat I       & \mat 0       & \mat 0      \\
    \mat 0       & \mat I       & \mat 0 \\
    \end{bmatrix}}_{\mat P}
    \underbrace{\begin{bmatrix}
    \vect x_r  \\
    \vect x_g  \\
    \vect x_b  \\
    \end{bmatrix}}_{\vect x} = \mat P \vect x
\eeaa
\bea
    \mat D_l =
        \begin{bmatrix}
        \mat D_{l_r} & \mat 0       & \mat 0      \\
        \mat 0       & \mat D_{l_g} & \mat 0      \\
        \mat 0       & \mat 0       & \mat D_{l_b}\\
        \end{bmatrix},~
    \mat D_h =
        \begin{bmatrix}
        \mat D_{h_r} & \mat 0       & \mat 0      \\
        \mat 0       & \mat D_{h_g} & \mat 0      \\
        \mat 0       & \mat 0       & \mat D_{h_b}\\
        \end{bmatrix}
        \label{Eq:DicsDef}
\eea
\beaa
    \mat S =
        \begin{bmatrix}
        \mat S_{r} & \mat 0       & \mat 0      \\
        \mat 0       & \mat S_{g} & \mat 0      \\
        \mat 0       & \mat 0       & \mat S_{b}\\
        \end{bmatrix}_{3p\times 3p},~~
    \mat P_s =
        \begin{bmatrix}
    \mat 0       & \mat 0       & \mat I      \\
    \mat I       & \mat 0       & \mat 0      \\
    \mat 0       & \mat I       & \mat 0 \\
    \end{bmatrix}_{3p\times 3p}
\eeaa

Where $\vect x$ and $\vect y_l$ respectively are concatenation of sparse codes and low resolution image patches (or features) in different color channels. $\mat P$ and $\mat P_s$ are shifting matrices that can shift the order of coefficients in the vectors and matrices. They consist of zero and identity matrices and have a size of $3m\times 3m$ and $3p\times 3p$, respectively. $m$ is the length of sparse code for each color channel, $p$ is the size of HR patches. $\mat D_l \in \mathbb{R}^{3q\times 3m}$ and $\mat D_h \in \mathbb{R}^{3p\times 3m}$ are dictionaries that contain color dictionaries in their block diagonals and
$q$ is length of LR features (patches). We also define and simplify $\mat D_{hs}$:
\bea
    \mat D_{hs} &=&
    \begin{bmatrix}
    \mat D_{h_b}^T \mat S_b^T \mat S_r \mat D_{h_r}   &                          \mat 0                              &                       \mat 0  \\
    \mat 0                                            & \mat D_{h_r}^T \mat S_r^T \mat S_g \mat D_{h_g}              &                       \mat 0   \\
    \mat 0                                            &                         \mat 0                               & \mat D_{h_g}^T \mat S_g^T \mat S_b \mat D_{h_b}    \\
    \end{bmatrix}  \nonumber\\
    &=& \mat P \mat D_h^T \mat S^T \mat P_s^T \mat S \mat D_h \label{Eq:Dhs}
\eea
Finally, the cost function in \eqref{Eq:Cost2} can be written as follows:
\bea
    L_1 &=& \frac{1}{2} \| \vect y_l - \mat D_l \vect x\|_2^2 + \lambda \|\vect x \|_1 \nonumber\\
    &&+ 2 \tau \vect x^T \mat D_h^T \mat S^T \mat S \mat D_h \vect x - 2 \tau \vect x^T \mat P^T \mat D_{hs} \vect x. \label{Eq:Cost3}\\
    &=& \vect x^T [ \frac{1}{2} \mat D_l^T \mat D_l + 2 \tau \mat D_h^T \mat S^T \mat S \mat D_h - 2\tau \mat P^T \mat D_{hs}] \vect x \nonumber\\
    && -\vect y_l^T \mat D_l \vect x + \frac{1}{2} \vect y_l^T \vect y_l + \lambda \|\vect x \|_1
\eea
Substituting \eqref{Eq:Dhs} in the above we have:
\bea
    \vect x^\ast &=~~ \arg\displaystyle\min_{\vect x}& \vect x^T \overbrace{[ \frac{1}{2} \mat D_l^T \mat D_l + 2 \tau \mat D_h^T \mat S^T (\mat I - \mat P_s^T)\mat S \mat D_h ]}^{\mat D} \vect x \nonumber\\
    && -\vect y_l^T \mat D_l \vect x + \frac{1}{2} \vect y_l^T \vect y_l + \lambda \|\vect x \|_1 \\
    &=~~ \arg\displaystyle\min_{\vect x} &\vect x^T \mat D \vect x - \vect y_l^T \mat D_l \vect x \ + \lambda\|\vect x\|_1 \label{Eq:FinalOptProb}
\eea
 The re-written cost function in \eqref{Eq:FinalOptProb}, which is now in a more familiar form, is a convex sparsity constrained quadratic optimization and consequently numerical algorithms such as FISTA \cite{Beck:IterativeShrinkageThresholdFISTA_ImagScienSIAM2009, Wright:SpaRSA_TSP2009, Minaee:Segmen_Asilomar2015} can be applied to solve it. Note that matrix $\mat D$ captures cross channel constraints using its off-diagonal blocks.


\subsection{Color Adaptive Patch Processing}
In the previous subsection we presented our color image super resolution framework by exploiting color edge similarities across color channels. However, we should emphasize that not all patches in an image have the same amount of color information and edge similarities. Therefore, any single patch should be treated differently in terms of color constraints. The regularizer parameter $\tau$ can control the emphasis on color edge similarities. Next, we explain our approach to automatically determine $\tau$ in an image/patch adaptive manner.


We use the following color variance measure to quantify the color information in each patch:
\bea
    \beta  = \frac{1}{2s} \Big( \frac{\|\mat H_1 \vect y_{Cb} \| + \|\mat H_1 \vect y_{Cr} \|}{\|\mat H_1 \vect y_{Y} \|}  +  \frac{\|\mat H_2 \vect y_{Cb} \| + \|\mat H_2 \vect y_{Cr} \|}{\|\mat H_2 \vect y_{Y} \|} \Big)
\eea
where $s$ is normalization parameter, $\mat H_1$ and $\mat H_2$ are high-pass Scharr operators and $y_Y, y_{Cb}$ and $y_{Cr}$ are Y, Cb and CR channel bands in YCbCr color space.

Mapping from  $\beta$ values to actual regularizer values ($\tau$) in the optimization framework is illustrated in Fig. \ref{Fig:AdaptivePatch}. It  is a lookup table, which is derived empirically based on a cross validation procedure. In particular: $\tau = c e^{-a \beta}$  and we used a validation dataset to find best values for $c$ and $a$.
Via the aforementioned cross validation, we found $c=0.02$ and $a = 7$.

\begin{figure}
  \centering
  \includegraphics[width=0.8\columnwidth]{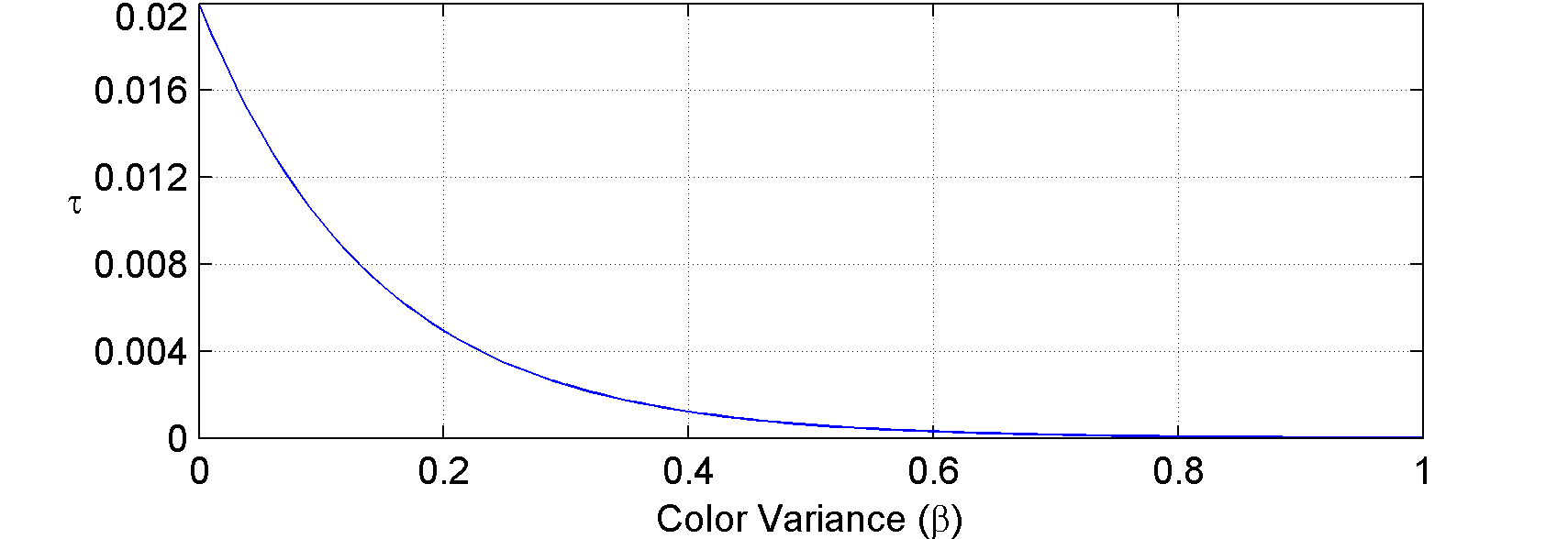}
  \caption{Relationship between color variance $\beta$ and regularizer parameter  $\tau$.}\label{Fig:AdaptivePatch}
\end{figure}

\section{Joint learning of Color Dictionaries}
\label{Sec:ColorDL}

Correlation between color channels can be even better captured if the individual color channel dictionaries are also designed to facilitate the same. In order to learn such dictionaries, we propose a new cost function which involves {\em joint learning} of color channel dictionaries.

Given a set of $N$ sampled training image patch pairs $\{\mat Y_h, \mat Y_l\}$, where $\mat Y_h = \{ \vect y_h^1, \vect y_h^2, ..., \vect y_h^N\}$ is the set of high resolution patches sampled from training images and $\mat Y_l = \{ \vect y_l^1, \vect y_l^2, ..., \vect y_l^N\}$ is the set of corresponding low resolution patches or extracted features, we aim to learn dictionaries with aforementioned characteristics. One essential requirement of course is that the sparse representation of low resolution patches and corresponding high resolution patches be the same. At the same time, the high resolution dictionary, which is responsible for reconstructing HR patches, should be designed to capture RGB edge correlations in the super-resolved images. Individually, sparse coding problems in low resolution and high resolution settings may be written as:
\bea
    \mat D_l  &=~~ \arg\displaystyle\min_{\mat D_l,\{\vect x^i\}} & \frac{1}{N} \sum_{i=1}^{N} \frac{1}{2}\| \vect y_l^i - \mat D_l \vect x^i \|_2^2 + \lambda \|\vect x^i\|_1  \nonumber \\
    & \text{st.}& \|\mat D_l(:,k) \|_2^2 \le 1, ~~k=1,2,...,K, \label{Eq:LRcost}
\eea
\bea
    \mat D_h  &=~~ \arg\displaystyle \min_{\mat D_h,\{\vect x^i\}} & \frac{1}{N} \sum_{i=1}^{N}  \frac{1}{2}\| \vect y_h^i - \mat D_h \vect x^i \|_2^2 + \lambda\|\vect x^i\|_1  \nonumber \\
    &&     +\tau \Big[ \|\mat S_r \mat D_{h_r} \vect x_r^i - \mat S_g \mat D_{h_g} \vect x_g^i \|_2^2 \nonumber\\
    &&     ~~+\|\mat S_g \mat D_{h_g} \vect x_g^i - \mat S_b \mat D_{h_b} \vect x_b^i\|_2^2 \nonumber\\
    &&     ~~+\|\mat S_b \mat D_{h_b} \vect x_b^i-\mat S_r \mat D_{h_r} \vect x_r^i\|_2^2 \Big] \nonumber \\
    & \text{st.}& \|\mat D_h(:,k) \|_2^2 \le 1, ~~k=1,2,...,K. \label{Eq:HRcost}
\eea
The additional terms in \eqref{Eq:HRcost} incorporate the edge information across color channels as in \eqref{Eq:MainOptProb}. Note that there is an implicit constraint on $\mat D_l$ and $\mat D_h$ that they both are block diagonal matrices as defined in \eqref{Eq:DicsDef}.  Considering the requirement that the sparse codes are the same for LR and HR framework, we can obtain the following optimization problem which simultaneously optimizes the LR and HR dictionaries:
\bea
     &\arg\displaystyle \min_{\mat D_h, \mat D_l,\{\vect x^i\}} & \frac{1}{N} \sum_{i=1}^{N}  \frac{\gamma}{2}\| \vect y_l^i - \mat D_l \vect x^i \|_2^2 + \frac{1-\gamma}{2}\| \vect y_h^i - \mat D_h \vect x^i \|_2^2 \nonumber \\
    &&     +\tau \Big[ \|\mat S_r \mat D_{h_r} \vect x_r^i - \mat S_g \mat D_{h_g} \vect x_g^i \|_2^2 \nonumber\\
    &&     ~~+\|\mat S_g \mat D_{h_g} \vect x_g^i - \mat S_b \mat D_{h_b} \vect x_b^i\|_2^2 \nonumber\\
    &&     ~~+\|\mat S_b \mat D_{h_b} \vect x_b^i-\mat S_r \mat D_{h_r} \vect x_r^i\|_2^2 \Big] + \lambda \|\vect x^i\|_1 \nonumber \\
    & \text{st.}& \|\mat D_h(:,k) \|_2^2 \le 1, ~\|\mat D_l(:,k) \|_2^2 \le 1,~~ k=1,2,...,K \nonumber\\ \label{Eq:DicCost}
\eea
where $\gamma$ balances the reconstruction error in low resolution and high resolution settings. Using simplifications similar to \eqref{Eq:Cost3}, this cost function can be re-written as follows:
\bea
     L_2&=& \frac{1}{N} \sum_{i=1}^{N}  \frac{\gamma}{2}\| \vect y_l^i - \mat D_l \vect x^i \|_2^2 + \frac{1-\gamma}{2}\| \vect y_h^i - \mat D_h \vect x^i \|_2^2  + \lambda \|\vect x^i\|_1 \nonumber \\
        && +2\tau   \vect x^{i^T} \mat D_h^T \mat S^T (\mat I - \mat P_s^T) \mat S \mat D_h \vect x^{i^T}  \label{Eq:DicCost2} \\
        &=&     \frac{\gamma}{2N}\| \mat Y_l - \mat D_l \mat X \|_F^2 + \frac{1-\gamma}{2N}\| \mat Y_h - \mat D_h \mat X \|_F^2  + \frac{\lambda}{N} \|\mat X\|_1 \nonumber \\
        && +\frac{2\tau}{N} ~\Tr \Big( \mat X^T \mat D_h^T \mat S^T (\mat I - \mat P_s^T) \mat S \mat D_h \mat X^T \Big).~ \label{Eq:MainDicCost}
\eea
where $\mat X = [\vect x^1~ \vect x^2~...~ \vect x^N] \in \mathbb{R}^{3m\times N}$. The first and second terms in \eqref{Eq:MainDicCost} are respectively responsible for small reconstruction error in low resolution and high resolution training data. The third term enforces sparsity and the last one encourages edge similarity via the learned dictionaries. We propose to minimize this cost function by alternatively optimizing over $\mat X, \mat D_l$  and $\mat D_h$ individually, while keeping the others fixed.

With $\mat D_l$ and $\mat D_h$ being fixed, we optimize \eqref{Eq:MainDicCost} over sparse code matrix $\mat X$. Interestingly because of the Trace operator and Frobenius norm, columns of $\mat X$  can be obtained independently. For each column of $\mat X$ ($i=1...N$) we can simplify the problem:
\bea
    \vect x^i &= ~~\arg\displaystyle\min_{\vect x}& \frac{\gamma}{2}\| \vect y_l^i -\mat D_l \vect x \|_F^2 + \frac{1-\gamma}{2}\| \vect y_h^i - \mat D_h \vect x \|_F^2  + \lambda \|\vect x\|_1 \nonumber\\
     &&     +2\tau   \vect x^{T} \mat D_h^T \mat S^T (\mat I - \mat P_s^T) \mat S \mat D_h \vect x^{T}  \nonumber\\
      &=~~ \arg\displaystyle\min_{\vect x}& \vect x^T [ \frac{\gamma}{2} \mat D_l^T \mat D_l +\frac{1-\gamma}{2} \mat D_h^T \mat D_h \nonumber\\
      &&~~~~~~~+ 2 \tau \mat D_h^T \mat S^T (\mat I - \mat P_s^T)\mat S \mat D_h ]  \vect x \nonumber\\
    && -\big( \gamma \vect y_l^{i^T} \mat D_l + (1-\gamma)\vect y_h^{i^T} \mat D_h \big)\vect x + \lambda \|\vect x \|_1  \nonumber\\
    &=~~ \arg\displaystyle\min_{\vect x} &\vect x^T \mat A \vect x - \vect b^T  \vect x \ + \lambda\|\vect x\|_1 \label{Eq:findx}
\eea
where $\mat A =  \frac{\gamma}{2} \mat D_l^T \mat D_l +\frac{1-\gamma}{2} \mat D_h^T \mat D_h + 2 \tau \mat D_h^T \mat S^T (\mat I - \mat P_s^T)\mat S \mat D_h $ and\\
 $\vect b^{i^T} = \gamma \vect y_l^{i^T} \mat D_l + (1-\gamma)\vect y_h^{i^T} \mat D_h $. The optimization in \eqref{Eq:findx} can be solved using FISTA \cite{Beck:IterativeShrinkageThresholdFISTA_ImagScienSIAM2009}.

The next step is to find the low resolution dictionary $\mat D_l$. By fixing $\mat X$ and $\mat D_h$, the cost function reduces to:
\bea
    \mat D_l &= ~~\arg\displaystyle\min_{\mat D_l}& \|\mat Y_l - \mat D_l \mat X\|_F^2 \nonumber \\
    &\text{s.t.}&   \|\mat D_l(:,k) \|_2^2 \le 1,~~ k=1,2,...,K \nonumber\\
    && \mat D_l \text{ is block diagonal as in \eqref{Eq:DicsDef}}.
\eea
Since $\mat D_l$ is block diagonal and there is no explicit cross channel constraint for the low resolution dictionary, the above optimization can be split into three separate dictionary learning procedures as follows where $c\in \{r,g,b\}$.
\bea
    \mat D_{l_c} &= ~~\arg\displaystyle\min_{\mat D_{l_c}}& \|\mat Y_{l_c} - \mat D_{l_c} \mat X_c\|_F^2 \nonumber \\
    &\text{s.t.}&   \|\mat D_{l_c}(:,k) \|_2^2 \le 1,~~ k=1,2,...,K \label{Eq:findDl}
\eea
which $\mat X_c = [\vect x_c^1~ \vect x_c^2~...~ \vect x_c^N] \in \mathbb{R}^{m\times N}$, $\mat Y_{l_c} = [\vect y_c^1~ \vect y_c^2~...~ \vect y_c^N]\in \mathbb{R}^{p\times N}$ and $c$ takes the subscripts from $\{r,g,b\}$ indicating a specific color channel.
Each of the above dictionaries are learned by the dictionary learning  method in \cite{Mairal:ODL_ICML2009}.

Finally, for finding $\mat D_h$, when $\mat X$ and $\mat D_l$ are fixed, we have:
\bea
     \mat D_h &= ~~ \arg\displaystyle\min_{\mat D_{h}}& \frac{1}{N} \sum_{i=1}^{N}   \frac{1-\gamma}{2}\| \vect y_h^i - \mat D_h \vect x^i \|_2^2  \nonumber\\
     &&~~+ 2\tau   \vect x^{i^T} \mat D_h^T \mat S^T (\mat I - \mat P_s^T) \mat S \mat D_h \vect x^{i^T} \nonumber\\
        &\text{s.t}&  \|\mat D_{h}(:,k) \|_2^2 \le 1,~~ k=1,2,...,K.    \nonumber\\
        && \mat D_h \text{ is block diagonal as in \eqref{Eq:DicsDef}}\label{Eq:findDh}
\eea
We develop a solution for (\ref{Eq:findDh}) using the Alternative Direction Method of Multipliers (ADMM) \cite{Boyd:ADMM_MachineLearn2011}.
We first define the function $g(\mat D_h, \mat Z)$ as follows which is essentially the same cost function with the multiplication by $\mat D_h$ in the final term of (\ref{Eq:findDh}) substituted by a slack matrix $\mat Z$:
\bea
    g(\mat D_h, \mat Z) = \frac{1}{N} \sum_{i=1}^{N}   \frac{1-\gamma}{2}\| \vect y_h^i - \mat D_h \vect x^i \|_2^2 + 2\tau   \vect x^{i^T} \mat D_h^T \mat S^T (\mat I - \mat P_s^T) \mat S \mat Z \vect x^{i^T}\nonumber
\eea
Then, solving the following optimization problem, which is a bi-convex problem, is equivalent to solving \eqref{Eq:findDh}.
\bea
    \mat D_h &= ~~ \arg\displaystyle\min_{\mat D_{h},\mat Z}& g(\mat D_h, \mat Z) \nonumber\\
    &\text{s.t}& \mat D_h - \mat Z = \mat 0,\nonumber\\
    &&              \|\mat D_{h}(:,k) \|_2^2 \le 1,~~ k=1,2,...,K.\nonumber\\
    && \mat D_h \text{ is block diagonal as in \eqref{Eq:DicsDef}}. \label{Eq:EquivalenDicOpt}
\eea
The following is a summary of iterative solution to \eqref{Eq:EquivalenDicOpt} using ADMM until a convergence is achieved where $t$ is the iteration index of ADMM procedure:

    \bea
        1)~~~\mat D_h^{t+1} &=~~ \arg\displaystyle\min_{\mat D_{h}}& \Big( \frac{1}{N} \sum_{i=1}^{N}   \frac{1-\gamma}{2}\| \vect y_h^i - \mat D_h \vect x^i \|_2^2  \nonumber\\
         &&~~+ 2\tau   \vect x^{i^T} \mat D_h^T \mat S^T (\mat I - \mat P_s^T) \mat S \mat Z^t \vect x^{i^T} \Big) \nonumber\\
         &&~~+ \frac{\rho}{2} \| \mat D_h -  \mat Z^t + \mat U^t\|_F^2 \nonumber\\
         &s.t.& \|\mat D_{h}(:,k) \|_2^2 \le 1, ~k=1,...,K.\nonumber\\
         &&     \mat D_h \text{ is block diagonal as \eqref{Eq:DicsDef}.} \label{Eq:ADMMstep1}
    \eea
    \bea
        2)~~~\mat Z^{t+1} &= \arg\displaystyle\min_{\mat Z}& \Big( \frac{2\tau}{N} \sum_{i=1}^{N} \vect x^{i^T} \mat D_h^{t+1^T} \mat S^T (\mat I - \mat P_s^T) \mat S \mat Z^t \vect x^{i^T} \Big) \nonumber\\
         &&~~+ \frac{\rho}{2} \| \mat D_h^{t+1} -  \mat Z + \mat U^t\|_F^2 \label{Eq:ADMMstep2}
    \eea
    \bea
        3)~~~\mat U^{t+1} &=& \mat U^{t} + \mat D_h^{t+1} - \mat Z^{t+1} ~~~~~~~~~~~~~~~~~~~~~~\label{Eq:ADMMstep3}
    \eea
Step 3 of the above ADMM procedure is straight forward. However, Steps 1 and 2 need further analytical simplifications for tractability.

\textbf{Step 1:} The optimization  in this step can be re-written as:
\bea
    \mat D_h^{t+1} &=~~ \arg\displaystyle\min_{\mat D_{h}}& \Tr (\mat D_h \mat F \mat D_h^T) - 2 \Tr (\mat E \mat D_h^T) \nonumber\\
    &\text{s.t.} & \|\mat D_{h}(:,k) \|_2^2 \le 1, ~k=1,...,K. \nonumber\\
    &&\mat D_h \text{ is block diagonal as in \eqref{Eq:DicsDef}}   \label{Eq:DhADMM}
\eea
where
\bea
    \mat F &=& \frac{1-\gamma}{2N} \mat X \mat X^T   + \frac{\rho}{2} \mat I_{3m\times 3m} \\
    \mat E &=& \frac{1-\gamma}{2N} \mat Y_h \mat X^T +\frac{\rho}{2} (\mat Z^k - \mat U^k) -\frac{\tau}{N} \mat S^T (\mat I -\mat P_s^T) \mat S \mat Z^k \mat X \mat X^T. \nonumber\\
\eea
Assuming the following block structure for $\mat E$ and $\mat F$:
\bea
    \mat F =
        {\begin{bmatrix}
        \mat F_{rr}       & \mat F_{12}     & \mat F_{13}      \\
        \mat F_{21}       & \mat F_{gg}     & \mat F_{23}   \\
        \mat F_{31}       & \mat F_{32}     & \mat F_{bb} \\
        \end{bmatrix}},~~
    \mat E =
        {\begin{bmatrix}
        \mat E_{rr}       & \mat E_{12}     & \mat E_{13}      \\
        \mat E_{21}       & \mat E_{gg}     & \mat E_{23}   \\
        \mat E_{31}       & \mat E_{32}     & \mat E_{bb} \\
        \end{bmatrix}}
\eea
and due to the block diagonal structure of $\mat D_h$ as in \eqref{Eq:DicsDef}, we can rewrite each term in \eqref{Eq:DhADMM} in the following form:
\bea
    \Tr (\mat E \mat D_h^T) = \Tr (\mat E_{rr} \mat D_{h_r}^T) +\Tr (\mat E_{gg} \mat D_{h_g}^T) +\Tr (\mat E_{bb} \mat D_{h_b}^T)
\eea
\bea
    \Tr (\mat D_h \mat F \mat D_h^T) = \Tr (\mat D_{h_r} \mat F_{rr} \mat D_{h_r}^T) +\Tr (\mat D_{h_g} \mat F_{gg} \mat D_{h_g}^T) \nonumber\\ ~~~~+\Tr (\mat D_{h_b} \mat F_{bb} \mat D_{h_b}^T)
\eea
Finally the cost function reduces to:
\bea
    & \arg\displaystyle\min_{\mat D_{h_r},\mat D_{h_g},\mat D_{h_b}}& \Tr(\mat D_{h_r} \mat F_{rr} \mat D_{h_r}^T) -2 \Tr (\mat E_{rr} \mat D_{h_r}^T) \nonumber\\
    &&+  \Tr(\mat D_{h_g} \mat F_{gg} \mat D_{h_g}^T) -2 \Tr (\mat E_{gg} \mat D_{h_g}^T) \nonumber\\
    &&+  \Tr(\mat D_{h_b} \mat F_{bb} \mat D_{h_b}^T) -2 \Tr (\mat E_{bb} \mat D_{h_b}^T) \nonumber\\
    &\text{s.t.} & \|\mat D_{h_c}(:,k) \|_2^2 \le 1, ~ c\in \{r,g,b\}.
\eea
which is a separable optimization problem, i.e. it can be solved for $\mat D_{h_r}, \mat D_{h_g}$ and $\mat D_{h_b}$ separately as follows:
 \bea
    & \arg\displaystyle\min_{\mat D_{h_c}}& \Tr(\mat D_{h_c} \mat F_{cc} \mat D_{h_c}^T) -2 \Tr (\mat E_{cc} \mat D_{h_c}^T) \nonumber\\
    &\text{s.t.} & \|\mat D_{h_c}(:,k) \|_2^2 \le 1, ~ k=1,2,...,K_c.
\eea
Each of above subproblems now is solvable using the algorithmic approach in Online Dictionary Learning \cite{Mairal:ODL_ICML2009}.

\textbf{Step 2:} This is an {\em unconstrained} convex optimization problem in terms of $\mat Z$ and we can find the minimum by taking the derivative. The closed form solution for $\mat Z$ is given by:
\bea
    \mat Z^{t+1} = \mat D_h^{t+1} + \mat U^{t+1} - \frac{2\tau}{N\rho}  \mat S^T (\mat I - \mat P_s) \mat S \mat D_h^{t+1} \mat X \mat X^T
\eea

A formal stepwise description of our color dictionary learning algorithm is given in Algorithm \ref{Alg:ColorDL}.

\begin{algorithm}[t]
\caption{Color Dictionary Learning }
\label{Alg:ColorDL}
\begin{algorithmic}
\REQUIRE $\mat Y_l, \mat Y_h, \tau, \lambda, \rho$.\\
\emph{initialize: } $\mat D_h^0, \mat D_l^0$, iteration index $n=1$.
\FOR{$n=1: $ Maxiter}
\STATE(1) Find the sparse code matrix by Solving the convex optimization problem in \eqref{Eq:findx}:
\STATE(2) Solve the LR dictionary learning problem in \eqref{Eq:findDl}
\STATE(3) Solve the HR dictionary learning problem in \eqref{Eq:findDh}:
    \WHILE{stopping criterion not met} 
    \STATE(3-1) Solve for $\mat D_h^{t+1}$ using \eqref{Eq:ADMMstep1}
    \STATE(3-2) Solve for $\mat Z^{t+1}$ using \eqref{Eq:ADMMstep2}
    \STATE(3-3) Solve for $\mat U^{t+1}$ using \eqref{Eq:ADMMstep3}
    \STATE(3-4) Increase inner iteration index $t$.
   \ENDWHILE{ if $\| \mat D_h^{t+1} - \mat D_h^{t}\|_F < tol$ }
\STATE(4) Increase iteration index $n$.
\ENDFOR
\ENSURE $\mat D_h, \mat D_l$.
\end{algorithmic}
\end{algorithm}

\section{Experimental Results}
\label{Sec:Experiments}
Our experiments are performed on the widely used \emph{set 5} and \emph{set 14} images as in \cite{Zeyde:SR_Springer2012}. We compare the proposed Multi-Channel constrained Super Resolution (MCcSR) method with several well-known single image super resolution methods. These include the ScSR \cite{Yang:CoupledDicLearnSR_TIP2012} method because our MCcSR method can be seen as a multi-channel extension of the same. Other methods for which we report results are the Single Image Scale-up using Sparse Representation by Zeyde \emph{et al.} \cite{Zeyde:SR_Springer2012}, Anchored Neighborhood Regression for Fast Example-Based Super-Resolution (ANR) \cite{Timofte:AnchoredARN+_ACCV2014} and Global Regression (GR) \cite{Timofte:AnchoredANR_ICCV2013} methods by Timofte \emph{et al}, Neighbor Embedding with Locally Linear Embedding (NE+LLE) \cite{Chang:NeighborEmbeddingSR_CVPR2004} and Neighbor Embedding with NonNegative Least Squares (NE+NNLS) \cite{Bevilacqua:NENNLS_BMVA2012} that were both adapted to learned dictionaries.

In our experiments, we will magnify the input images by a factor of $2$, $3$ or $4$, which is commonplace in the literature. For the low-resolution images, we
use $5 \times 5$ low-resolution patches with overlap of 4 pixels between adjacent patches and extract features based on method in \cite{YangAndWright:SparseSR_TIP2010}. It is noteworthy to mention that these features are not extracted from the $5 \times 5$ low resolution patches, but rather from bicubic interpolated version of the whole image with the desired magnification factor. Extracted features are then used to find the sparse codes according to \eqref{Eq:FinalOptProb} which involves color information as well. Then, high resolution patches are reconstructed based on the same sparse code using the learned high resolution dictionaries and averaged over the overlapping regions. Dictionaries are obtained by training over $100000$ patch pairs which are preprocessed by cropping out the textured regions and discarding the smooth regions. The number of columns in each learned dictionary is $512$ for most of our experiments and regularization parameter $\lambda$ is picked via cross-validation to be $0.1$.
\begin{figure*}
  \centering
  \includegraphics [width = 1\textwidth]{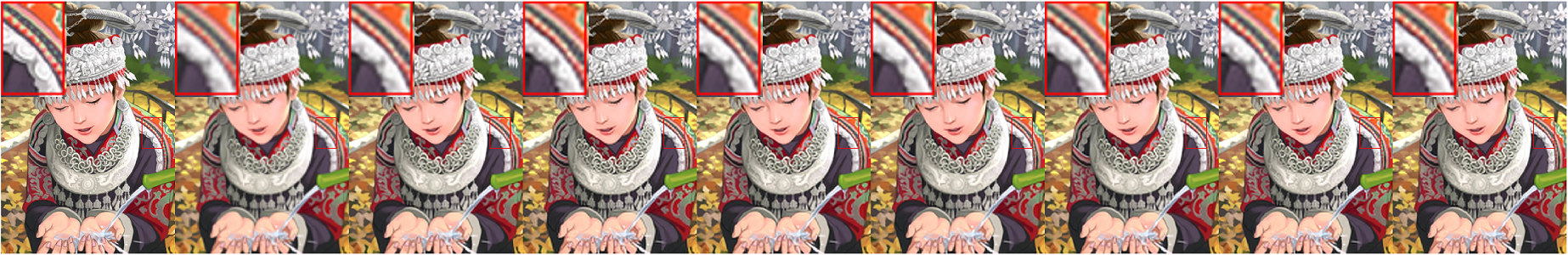}
  \includegraphics [width = 1\textwidth]{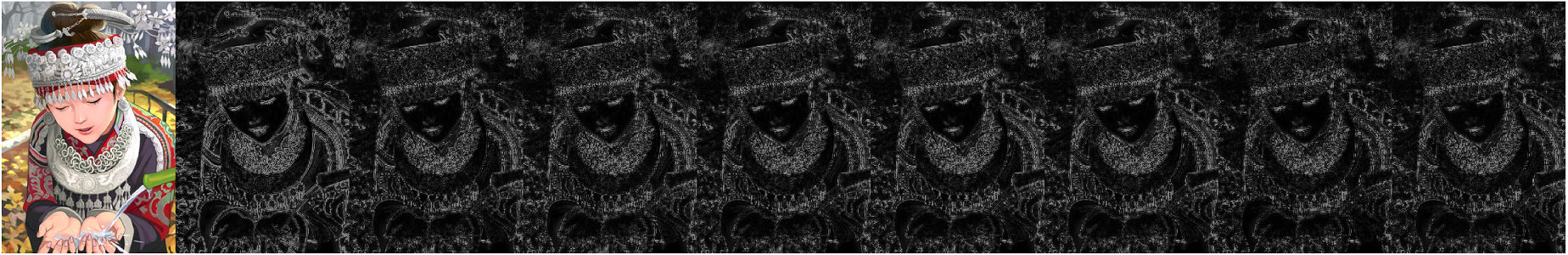}
  \caption{\scriptsize{Comparison of different methods for comic image with scaling factor of 2 (Please refer to the electronic version and zoom in for obvious comparison). Numbers in parenthesis are PSNR, SSIM and SCIELAB error measures, respectively. Left to right: Original,
                    Bicubic      (30.46, 0.840, 1.898e4),
                    Zeyde et al. (31.97, 0.887, 1.127e4),
                    GR           (31.70, 0.879, 1.198e4),
                    ANR          (32.09, 0.889, 1.077e4),
                    NENNLS       (31.87, 0.884, 1.159e4),
                    NELLE        (32.03, 0.889, 1.099e4),
                    \tb{MCcSR    (32.23, 0.899, 9.770e3)},
                    ScSR         (32.14, 0.893, 1.014e4). }}
    \label{Fig:Comic2x}
\end{figure*}
\begin{figure*}
  \centering
  \includegraphics [width = 1\textwidth]{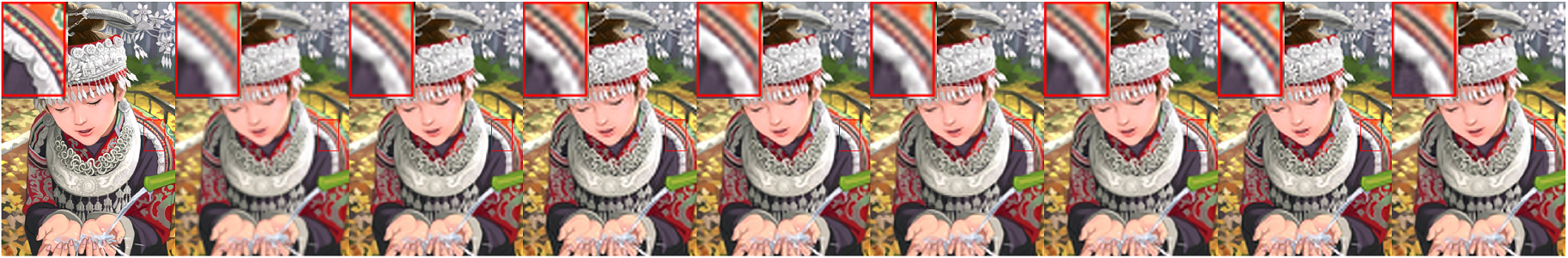}
  \includegraphics [width = 1\textwidth]{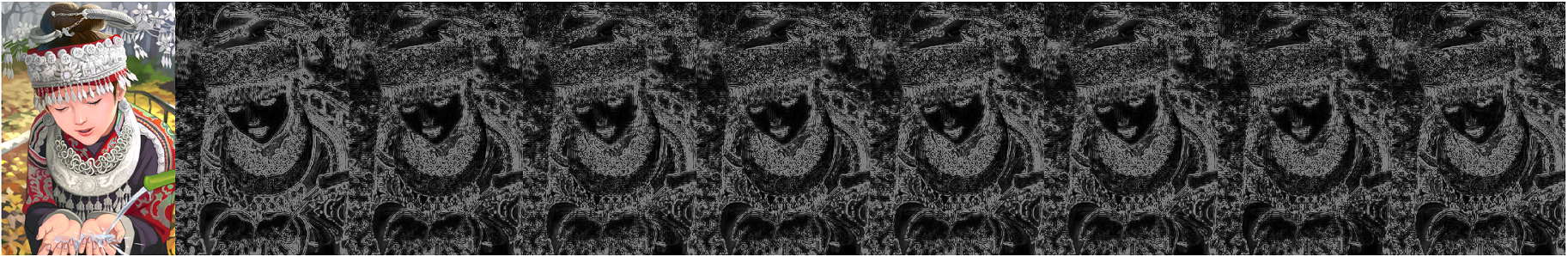}
       \caption{\scriptsize{ Super-resolution results for scaling factor 3 and quantitative measures. Left to right:
    {               Original,
                    Bicubic      (27.51, 0.685, 3.423e4),
                    Zeyde et al. (28.28, 0.737, 2.896e4),
                    GR           (28.15, 0.729, 3.008e4),
                    ANR          (28.36, 0.742, 2.865e4),
                    NENNLS       (28.17, 0.730, 2.961e4),
                    NELLE        (28.30, 0.738, 2.905e4),
                    \tb{MCcSR    (28.51, 0.758, 2.709e4)},
                    ScSR         (28.31, 0.740, 2.860e4) . }} }
        \label{Fig:Comic3x}
\end{figure*}
\begin{figure*}
  \centering
  \includegraphics [width = 1\textwidth]{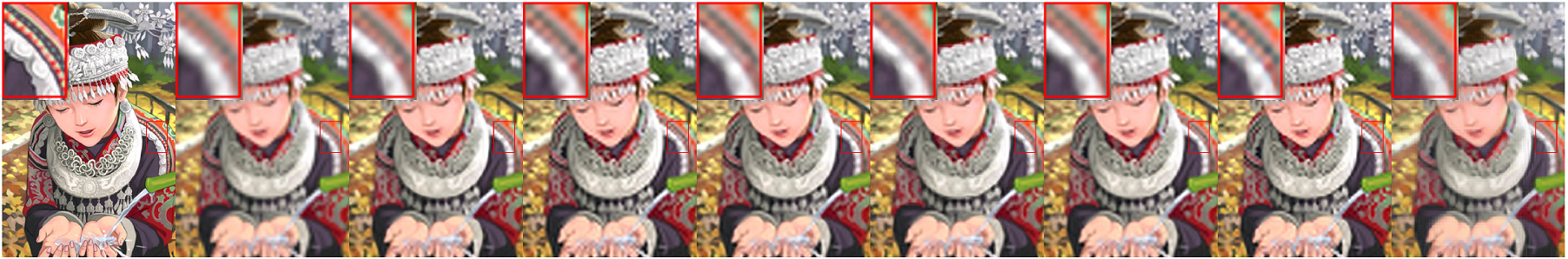}
  \includegraphics [width = 1\textwidth]{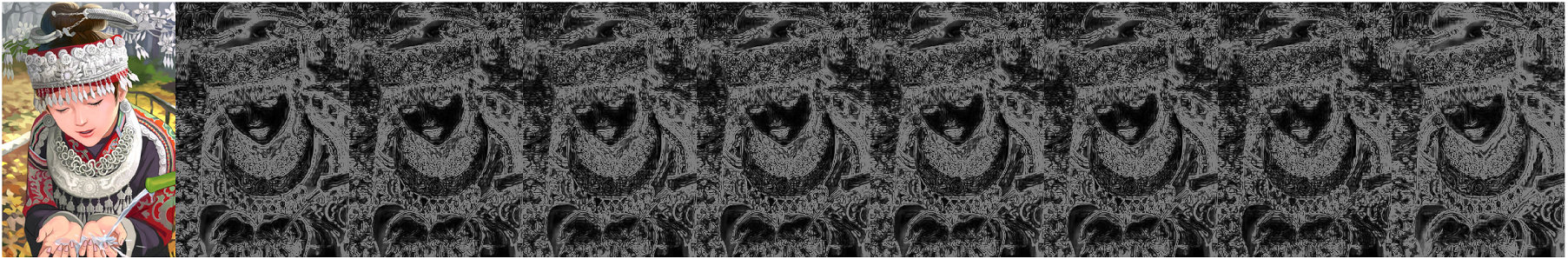}
     \caption{\scriptsize{  Super-resolution results for scaling factor 4 and quantitative measures. Left to right:
    {               Original,
                    Bicubic      (26.05, 0.566, 4.369e4),
                    Zeyde et al. (26.61, 0.615, 3.923e4),
                    GR           (26.51, 0.607, 4.045e4),
                    ANR          (26.63, 0.618, 3.928e4),
                    NENNLS       (26.50, 0.606, 3.984e4),
                    NELLE        (26.57, 0.614, 3.967e4),
                    \tb{MCcSR    (26.74, 0.632, 3.818e4)},
                    ScSR         (26.35, 0.608, 4.002e4) . }} }
      \label{Fig:Comic4x}
\end{figure*}

\begin{figure*}
  \centering
  \includegraphics [width = 1\textwidth]{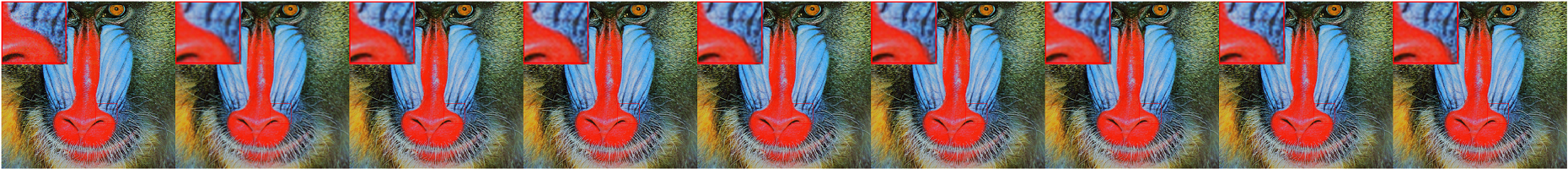}
  \includegraphics [width = 1\textwidth]{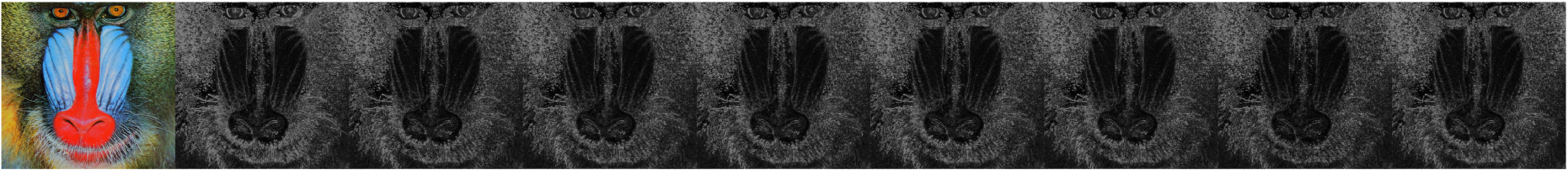}
        \caption{\scriptsize{  Comparison of different methods for baboon image with scaling factor of 2. Numbers in parenthesis are PSNR, SSIM and SCIELAB error measures, respectively. Left to right:
    {               Original,
                    Bicubic      (28.19, 0.635, 7.856e4),
                    Zeyde et al. (28.62, 0.683, 6.570e4),
                    GR           (28.63, 0.690, 6.388e4),
                    ANR          (28.67, 0.689, 3.287e4),
                    NENNLS       (28.58, 0.680, 6.585e4),
                    NELLE        (28.66, 0.688, 6.421e4),
                    \tb{MCcSR    (28.78, 0.705, 5.799e4)},
                    ScSR         (28.69, 0.692, 6.296e4) . }} }
                    \label{Fig:Baboon2x}
\end{figure*}
\begin{figure*}
  \centering
  \includegraphics [width = 1\textwidth]{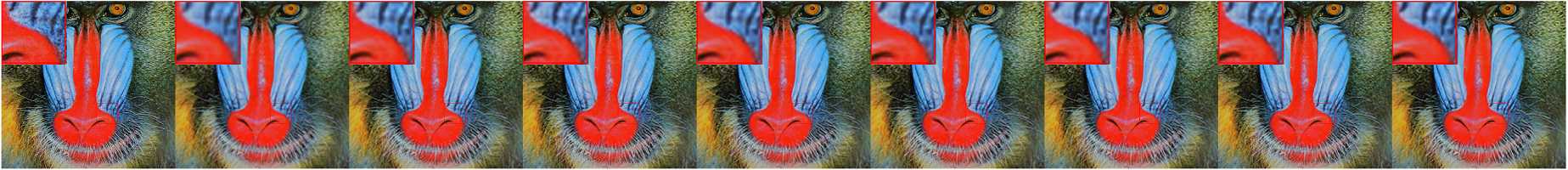}
  \includegraphics [width = 1\textwidth]{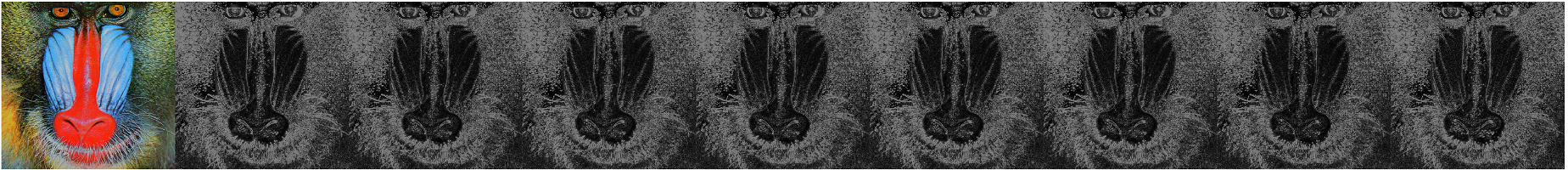}
       \caption{\scriptsize  { Super-resolution results for scaling factor 3 and quantitative measures. Left to right:
    {               Original,
                    Bicubic      (26.71, 0.480, 1.078e5),
                    Zeyde et al. (26.94, 0.520, 1.008e5),
                    GR           (26.95, 0.529, 1.000e5),
                    ANR          (26.97, 0.527, 9.962e4),
                    NENNLS       (26.92, 0.518, 1.010e5),
                    NELLE        (26.97, 0.526, 9.998e4),
                    \tb{MCcSR    (27.11, 0.549, 9.574e4)},
                    ScSR         (26.95, 0.524, 1.018e5) . }}}
                    \label{Fig:Baboon3x}
\end{figure*}
\begin{figure*}
  \centering
  \includegraphics [width = 1\textwidth]{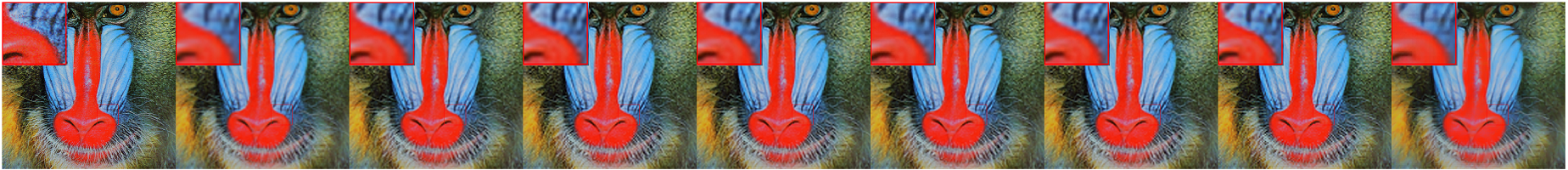}
  \includegraphics [width = 1\textwidth]{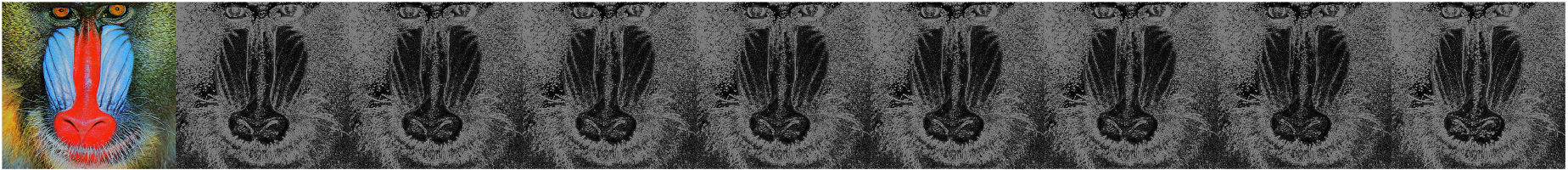}
     \caption{ \scriptsize {Super-resolution results for scaling factor 4 and quantitative measures. Left to right:
     {              Original,
                    Bicubic      (26.00, 0.390, 1.237e5),
                    Zeyde et al. (26.17, 0.420, 1.186e5),
                    GR           (26.17, 0.428, 1.183e5),
                    ANR          (26.19, 0.426, 1.180e5),
                    NENNLS       (26.15, 0.419, 1.190e5),
                    NELLE        (26.18, 0.425, 1.183e5),
                    \tb{MCcSR    (26.25, 0.446, 1.136e5)},
                    ScSR         (26.11, 0.415, 1.185e5) . }}}
                    \label{Fig:Baboon4x}
\end{figure*}
\begin{figure*}
\centering
\subfigure[]    {\includegraphics[width=0.48\textwidth]{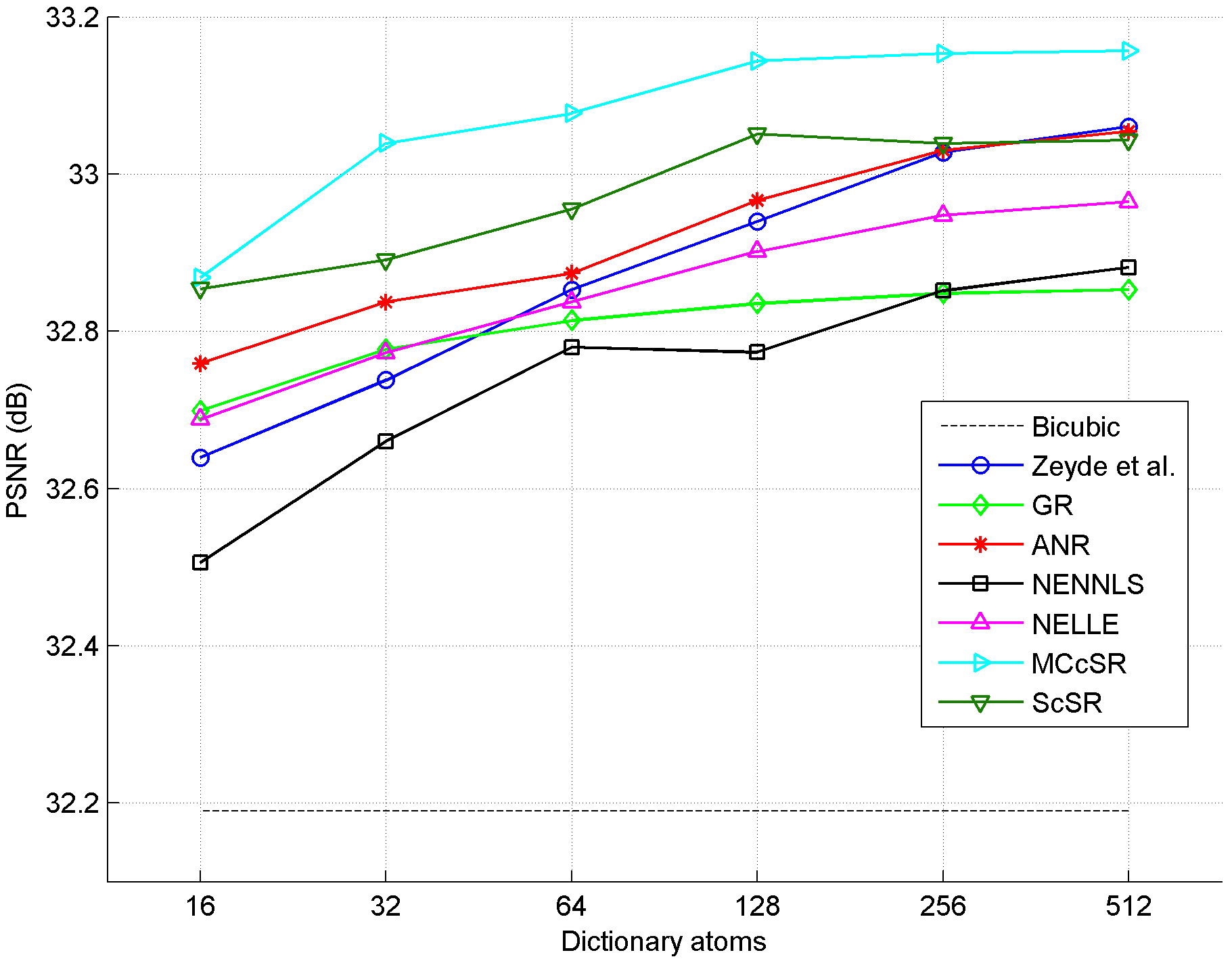}}
\subfigure[]    {\includegraphics[width=0.48\textwidth]{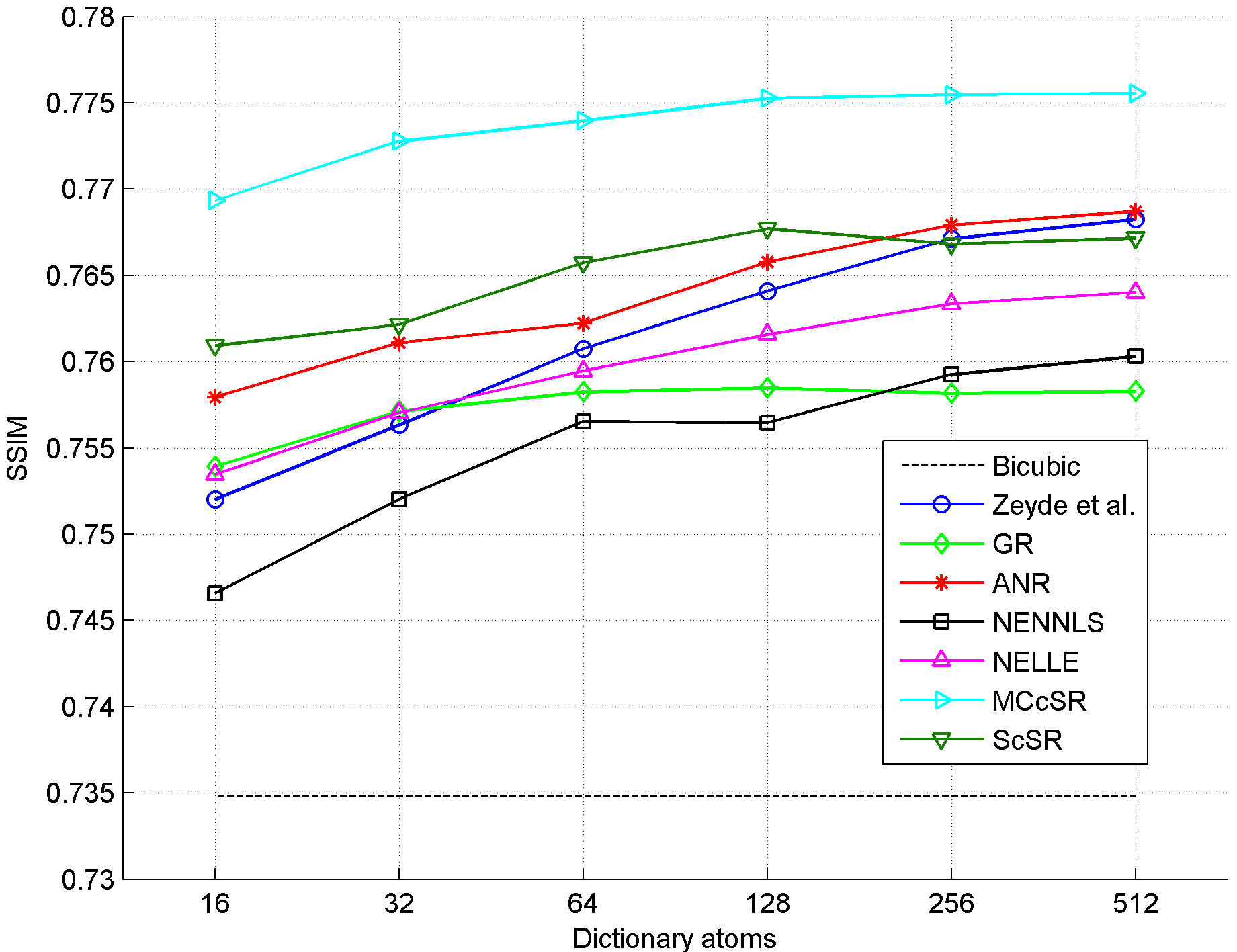}}\\
\subfigure[]    {\includegraphics[width=0.48\textwidth]{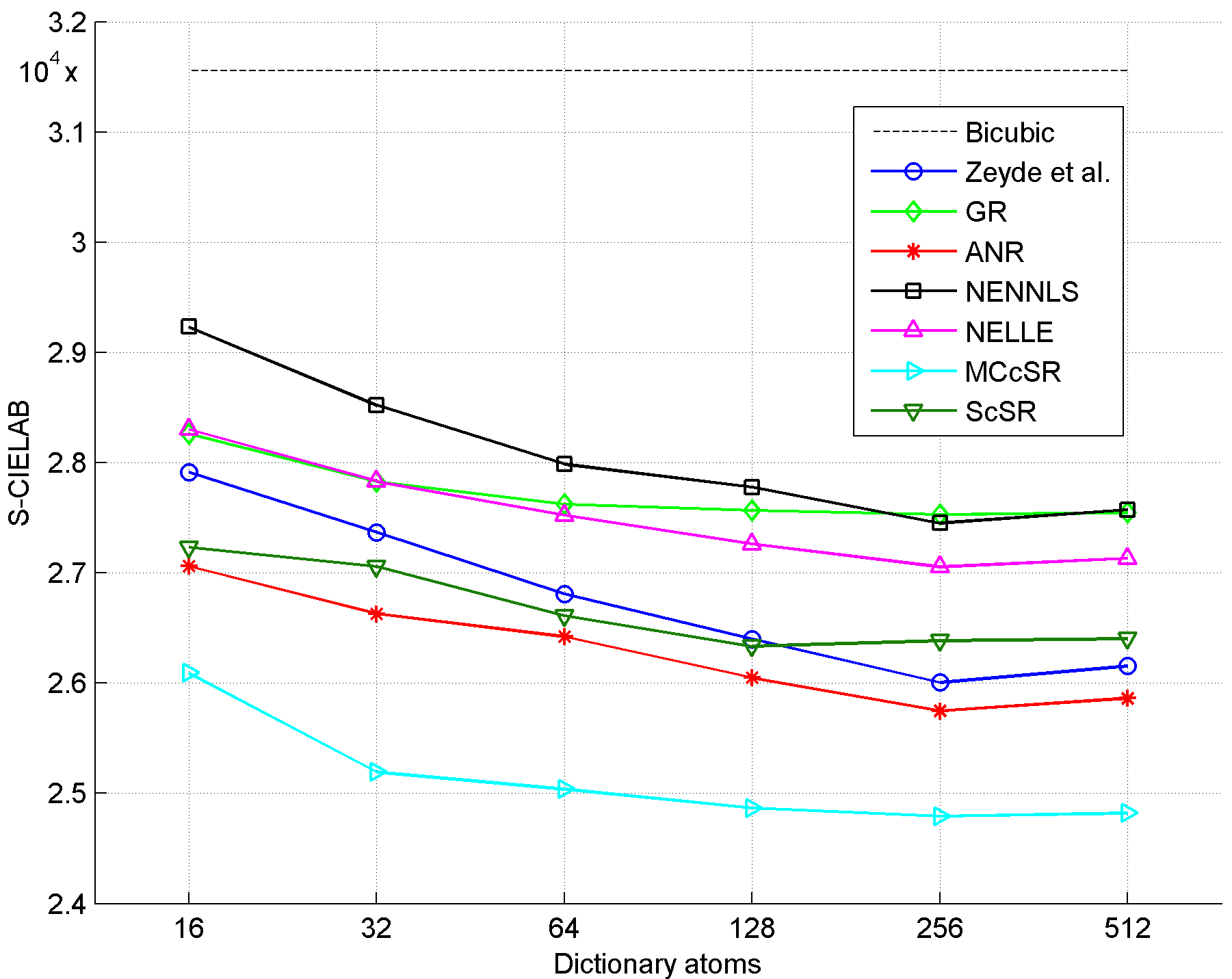}}
  \caption{Effect of dictionary size on PSNR, SSIM and S-CIELAB error of SR methods with a scaling factor of $3$. Number of dictionary atoms are varied between 16 and 512 and performance of different methods are compared.}\label{Fig:DictSizePSNR}
 \end{figure*}

\begin{table*}
\centering
\begingroup
      \caption{PSNR results of different methods for various images with scaling factor of $3$.}
\label{Tab:imagesPSNR}
      \begin{tabular}{ |l||c|c|c|c|c|c|c|c| }
\hline	
\multirow{2}{*}{Images}  & \multicolumn{8}{c|}{PSNR (dB)} 	\\															 		
      \cline{2-9}
                         &   Bicub  &	Zeyde   &	GR	    &   ANR	    &   NENNLS	&   NELLE   &	MCcSR   &	ScSR        \\ \hline
baby		             &   38.42  &	39.51   &	39.38   &\tb{39.56} &	39.22	&   39.49   &	39.51   &	39.40       \\ \hline
butterfly	             &   28.73  &	30.60   &	29.73   &	30.57   &	30.29	&   30.42   &	30.59   &\tb{30.64}     \\ \hline
bird		             &   36.37  &	37.90   &	37.44   &	37.92   &	37.68	&   37.90   &\tb{38.02} &	37.59       \\ \hline
face		             &   35.96  &	36.44   &	36.40   &\tb{36.50} &	36.39	&   36.47   &	36.48   &	36.37       \\ \hline
foreman	                 &   35.76  &	37.67   &	36.84   &	37.71   &	37.37	&   37.69   &\tb{37.74} &	37.64       \\ \hline
coastguard	             &   31.31  &	31.91   &	31.78   &	31.84   &	31.77	&   31.83   &\tb{31.95} &	31.83       \\ \hline
flowers	                 &   30.92  &	31.84   &	31.62   &	31.88   &	31.68	&   31.80   &\tb{32.07} &	31.87       \\ \hline
head		             &   36.02  &	36.47   &	36.42   &\tb{36.52} &	36.40	&   36.50   &	36.51   &	36.42       \\ \hline
lenna		             &   35.26  &	36.23   &	35.99   &	36.29   &	36.11	&   36.24   &\tb{36.33} &	36.14       \\ \hline
man		                 &   31.78  &	32.68   &	32.44   &	32.71   &	32.50	&   32.65   &\tb{32.75} &	32.68       \\ \hline
pepper		             &   35.25  &	36.27   &	35.77   &	36.13   &	35.99	&   36.12   &\tb{36.30} &	36.20       \\ \hline \hline
average                  &   33.08  &	34.06   &	33.76   &	34.07   &	33.88	&   34.03   &\tb{34.14} &	34.00       \\ \hline
      \end{tabular}
\endgroup
    \end{table*}
\begin{table*}
\centering
\begingroup
      \caption{SSIM results of different methods for various images with scaling factor of $3$.}
\label{Tab:imagesSSIM}
      \begin{tabular}{ |l||c|c|c|c|c|c|c|c| }
\hline	
\multirow{2}{*}{Images}  & \multicolumn{8}{c|}{SSIM} 	\\															 		
      \cline{2-9}
                         &   Bicub	&   Zeyde	&	GR	    &	ANR	    &	NENNLS	&	NELLE	&	MCcSR	&	ScSR        \\ \hline
baby		             &   0.88	&   0.90	&	0.90	&	0.90	&	0.89	&	0.90	&	0.90	&	0.89        \\ \hline
butterfly	             &   0.79	&   0.85	&	0.80	&	0.84	&	0.84	&	0.84	&	0.85	&	0.85        \\ \hline
bird		             &   0.90	&   0.92	&	0.91	&	0.92	&	0.92	&	0.92	&	0.93	&	0.91        \\ \hline
face		             &   0.72	&   0.74	&	0.74	&	0.74	&	0.74	&	0.74	&	0.75	&	0.74        \\ \hline
foreman	                 &   0.89	&   0.91	&	0.90	&	0.91	&	0.90	&	0.91	&	0.91	&	0.90        \\ \hline
coastguard	             &   0.57	&   0.62	&	0.63	&	0.62	&	0.61	&	0.62	&	0.63	&	0.62        \\ \hline
flowers	                 &   0.77	&   0.80	&	0.79	&	0.80	&	0.79	&	0.80	&	0.81	&	0.80        \\ \hline
head		             &   0.72	&   0.74	&	0.74	&	0.75	&	0.74	&	0.74	&	0.75	&	0.74        \\ \hline
lenna		             &   0.78	&   0.80	&	0.80	&	0.80	&	0.80	&	0.80	&	0.81	&	0.80        \\ \hline
man		                 &   0.72	&   0.76	&	0.76	&	0.77	&	0.76	&	0.76	&	0.76	&	0.76        \\ \hline
pepper		             &   0.78	&   0.80	&	0.79	&	0.80	&	0.79	&	0.79	&	0.80	&	0.79        \\ \hline \hline
average                  &   0.745	&   0.776	&   0.769	&   0.778	&   0.771	&   0.775	&\tb{0.785}	&   0.774       \\ \hline
      \end{tabular}
\endgroup
    \end{table*}
\begin{table*}
\centering
\centering
      \caption{S-CIELAB error results of different methods for various images with scaling factor of $3$.}
\label{Tab:imagesSCIELAB}
\resizebox{\textwidth}{!}{
\begin{tabular}{ |l||c|c|c|c|c|c|c|c| }
\hline
\multirow{2}{*}{Images}  & \multicolumn{8}{c|}{S-CIELAB} 	\\	
      \cline{2-9}
                         &   Bicub	    &   Zeyde	    &   GR	        &   ANR	        &   NENNLS	    &   NELLE	    &   MCcSR	    &   ScSR            \\ \hline
baby		             &   2.07E+04	&   1.36E+04	&   1.40E+04	&\tb{1.32E+04}	&   1.47E+04	&   1.34E+04	&   1.34E+04	&   1.50E+04        \\ \hline
butterfly	             &   2.28E+04	&   1.55E+04	&   1.84E+04	&   1.55E+04	&   1.60E+04	&   1.60E+04	&   1.54E+04	&\tb{1.49E+04}      \\ \hline
bird		             &   1.07E+04	&   7.36E+03	&   8.02E+03	&   7.21E+03	&   7.73E+03	&   7.30E+03	&\tb{6.50E+03}	&   7.81E+03        \\ \hline
face		             &   3.79E+03	&   2.71E+03	&   2.73E+03	&   2.57E+03	&   2.73E+03	&   2.61E+03	&\tb{2.47E+03}	&   2.70E+03        \\ \hline
foreman	                 &   8.46E+03	&   3.90E+03	&   4.79E+03	&\tb{3.48E+03}	&   4.01E+03	&   3.62E+03	&   3.72E+03	&   3.89E+03        \\ \hline
coastguard	             &   1.96E+04	&   1.71E+04	&   1.70E+04	&   1.70E+04	&   1.76E+04	&   1.71E+04	&\tb{1.69E+04}	&   1.70E+04        \\ \hline
flowers	                 &   4.47E+04	&   3.75E+04	&   3.89E+04	&   3.69E+04	&   3.84E+04	&   3.74E+04	&\tb{3.29E+04}	&   3.70E+04        \\ \hline
head		             &   3.79E+03	&   2.69E+03	&   2.74E+03	&   2.54E+03	&   2.79E+03	&   2.61E+03	&\tb{2.42E+03}	&   2.65E+03        \\ \hline
lenna		             &   2.44E+04	&   1.74E+04	&   1.85E+04	&   1.67E+04	&   1.79E+04	&   1.69E+04	&\tb{1.58E+04}	&   1.72E+04        \\ \hline
man		                 &   3.80E+04	&   2.91E+04	&   3.03E+04	&\tb{2.84E+04}	&   3.02E+04	&   2.89E+04	&   2.88E+04	&   2.95E+04        \\ \hline
pepper		             &   2.48E+04	&   1.91E+04	&   2.15E+04	&   1.96E+04	&   2.02E+04	&   1.95E+04	&\tb{1.73E+04}	&   1.91E+04        \\ \hline \hline
average                  &   2.79E+04	&   2.27E+04	&   2.36E+04	&   2.24E+04	&   2.33E+04	&   2.26E+04	&\tb{2.14E+04}	&   2.28E+04        \\ \hline
      \end{tabular}
}\end{table*}

%

We perform visual comparisons of obtained super-resolution images and additionally evaluate them quantitatively using image quality metrics. The metrics we use include: 1) Peak Signal to Noise Ratio (PSNR) while recognizing its limitations \cite{Wang:HowGoodIsMSE_SPM2009}\footnote{Note that since we  work on color images, the PSNR reported is carried out on all the color channels.}, 2) the widely used Structural Similarity Index (SSIM) \cite{Wang:SSIM_TIP2004} and 3) a popular color-specific quality measure called S-CIELAB \cite{Zhang:SCIELAB_Elsevier1998} which evaluates color fidelity while taking spatial context into account.

\subsection{Generic SR Results}
Fig. \ref{Fig:Comic2x} show SR results for a popular natural image where resolution enhancement was performed via scaling by a factor of $2$. In the description of the figure, PSNR (in dB), SSIM and S-CIELAB error measure appear in the parenthesis for each method. As can be seen in the enlarged area of Fig. \ref{Fig:Comic2x},  MCcSR more faithfully retains color texture. The bottom row of Fig. \ref{Fig:Comic2x} shows the S-CIELAB error maps for different methods. It is again apparent that the MCcSR method produces less error around edges and color textures. Consistent with the visual observations, the S-CIELAB error is lowest for MCcSR.

Fig. \ref{Fig:Comic3x} also shows the same image with a scaling factor of $3$ and the corresponding S-CIELAB error maps. In this case,  the color texture in the enlarged area is even more pronounced for MCcSR vs. other methods. The trend continues and benefits of MCcSR are most significant for a scaling factor of $4$ in Fig. \ref{Fig:Comic4x}. Similar results for the Baboon image are shown for scaling factors of $2$, $3$, $4$ respectively in Figs. \ref{Fig:Baboon2x}-\ref{Fig:Baboon4x}.


The degradation in image quality for SR results with increased scaling factor is intuitively expected. In a relative sense however, MCcSR suffers a more graceful decay. This is attributed to the use of prior information in the form of the quadratic color regularizers in our cost function, which compensates for the lack of information available to perform the superresolution task.

Tables \ref{Tab:imagesPSNR}-\ref{Tab:imagesSCIELAB} summarize the results of super resolution on images in \emph{set 5} and \emph{set 14} databases with a scaling factor of $3$. PSNR, SSIM and S-CIELAB error measures are compared and almost consistently our MCcSR method outperforms all the other competing state-of-the-art methods. The last row in these tables is essentially the average performance of each method over all the images in {\em set 5} and {\em set 14} datasets. Due to space constraints, we do not include all the LR and SR images for {\em set 5} and {\em set 14} in the chapter but they are made  available online in addition to the code at:  {\url{http://signal.ee.psu.edu/MCcSR.html}}.


\begin{figure}
  \centering
  \includegraphics[width=0.8\columnwidth]{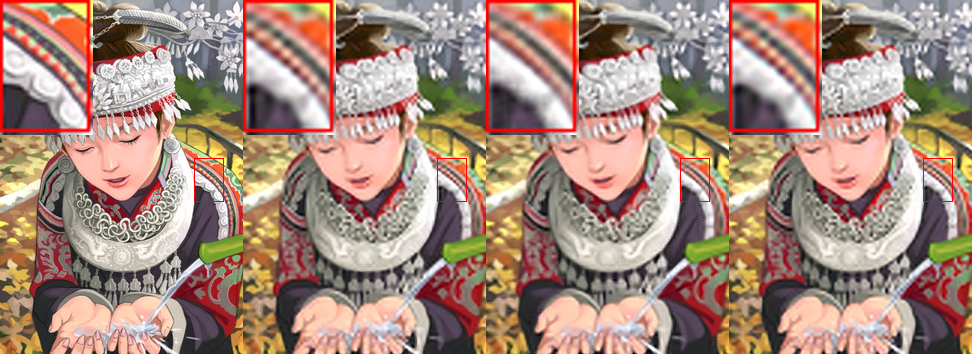}\\\vspace{-0.2in}
  \includegraphics[width=0.8\columnwidth]{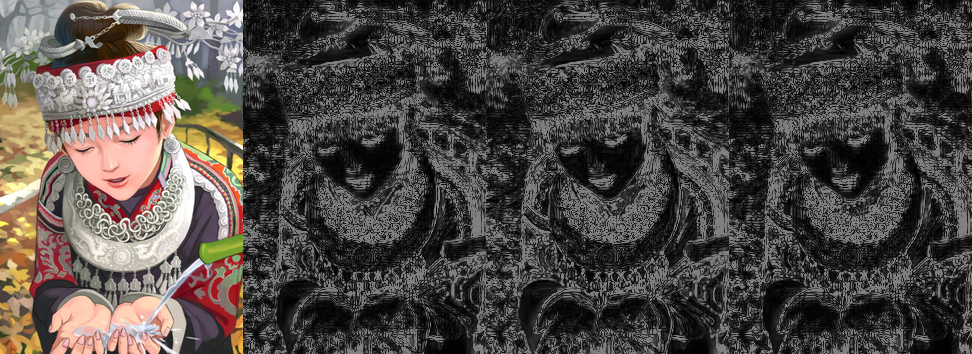}\\\vspace{-0.2in}
  \includegraphics[width=0.8\columnwidth]{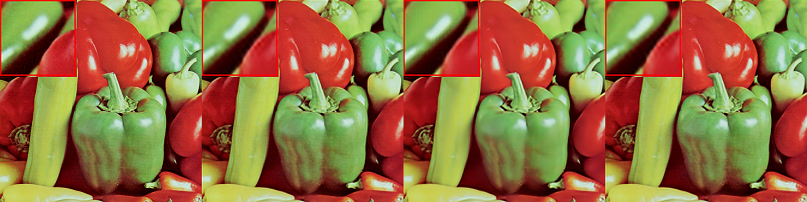}\\\vspace{-0.1in}
  \includegraphics[width=0.8\columnwidth]{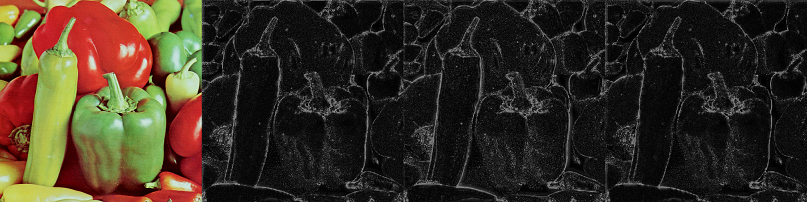}\\\vspace{-0.1in}
  \includegraphics[width=0.8\columnwidth]{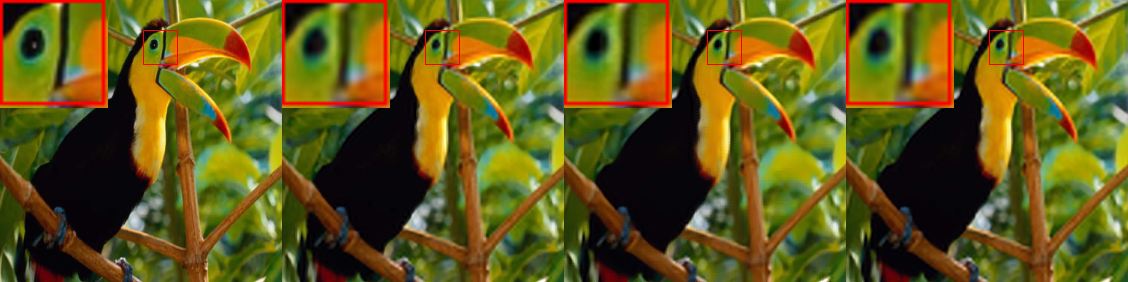}\\
  \includegraphics[width=0.8\columnwidth]{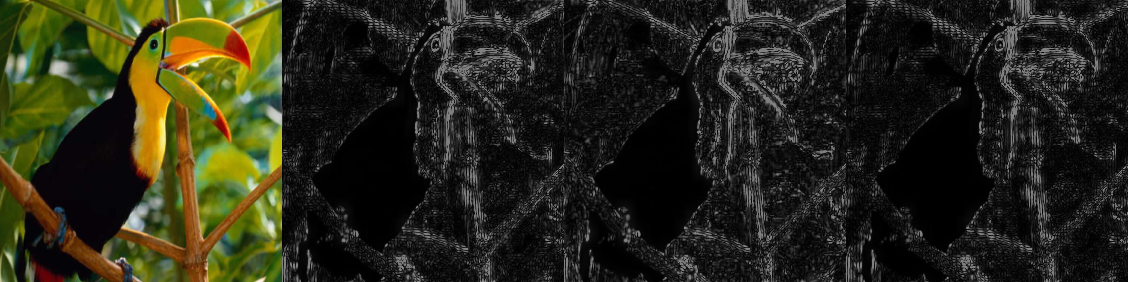}
    \caption{{Visual Images as well as S-CIELAB error maps are shown for a scaling factor of 3. From left to right for each row Images correspond to: Original Image, applying SR separately on RGB channels, ScSR, MCcSR }}\label{Fig:SeparateRGB2}
\end{figure}
\begin{figure*}
  \centering
  \includegraphics[width=\textwidth]{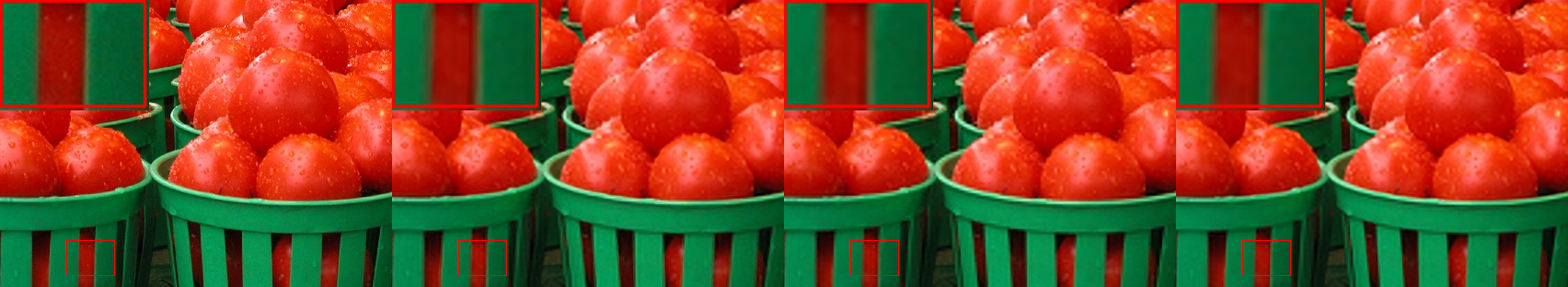}\\
  \includegraphics[width=\textwidth]{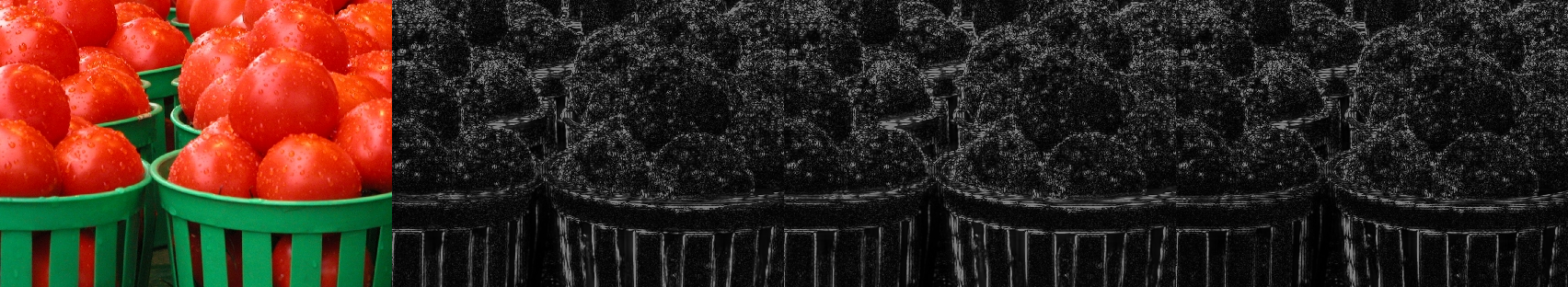}
    \caption{{Visual Images as well as S-CIELAB error maps are shown for a scaling factor of 3. From left to right for each row images correspond to: Original image, applying SR separately on RGB channels (36.26, 0.83, 1.57e4), ScSR (36.13, 0.83, 1.67e4) and \textbf{MCcSR (36.67, 0.85, 1.43e4)}. Numbers in parenthesis are PSNR, SSIM and SCIELAB error measures.}}\label{Fig:SeparateRGB1}
\end{figure*}

\subsection{Effect of Dictionary Size}

So far we have used a fix dictionary of size $512$ atoms for all the methods. In this Section, we evaluate the effect of the learned dictionary size for super-resolution. We again sampled $100,000$ image patches and train $6$ dictionaries of size $16$, $32$, $64$, $128$, $256$ and $512$ respectively. The results are evaluated both visually and quantitatively in terms of PSNR, SSIM and S-CIELAB. As is intuitively expected reconstruction artifacts gradually diminish with an increase in dictionary size and our visual observations are also supported by PSNR, SSIM and S-CIELAB of the recovered images. Fig \ref{Fig:DictSizePSNR}  shows the variation of different image quality metrics against dictionary size. For SSIM and S-CIELAB in particular, MCcSR is able to generate effective results even with smaller dictionaries.

\subsection{Effect of Color Regularizers: Separate RGBs}


We provide evidence for the importance of effectively accounting for color geometry via an illustrative example image. Three variations of color SR results are presented next:
\begin{enumerate}
  \item SR performed only on the luminance channel by ScSR \cite{Yang:CoupledDicLearnSR_TIP2012} method and bicubic interpolation is applied for chrominance channels.
  \item Single channel SR performed on red, green and blue channels independently. We again use  ScSR method; however, we learn separate dictionaries for RGB channels and apply ScSR  on RGB channels independently.
  \item Super-resolution  by explicitly incorporating cross channel information into the reconstruction  (our McCSR).
\end{enumerate}
In these experiments we use a scaling factor of $3$ and the results are reported in Figs. \ref{Fig:SeparateRGB1}, \ref{Fig:SeparateRGB2} and Table \ref{Tab:SeparateRGB}. It should particularly be noted (see Fig.\ \ref{Fig:SeparateRGB1}) that applying the SR method independently on RGB channels introduces very significant artifacts around color edges which are not visible in the results of MCcSR and ScSR. Fig.\ \ref{Fig:SeparateRGB2} shows similar results for a few other images. Table \ref{Tab:SeparateRGB} reports image quality measures which confirms the importance of using color channel constraints.


\begin{table*}
\centering
\begingroup
\fontsize{7pt}{7pt}\selectfont      \centering
      \caption{Quantitative measures to show effectiveness of color constraints in SR for a scaling factor of 3.}
      \label{Tab:SeparateRGB}
      \begin{tabular}{ |l||c|c|c|c|c|c|c|c|c| }
      \hline
\multirow{2}{*}{Images}  & \multicolumn{3}{c|}{PSNR (dB)}                     & \multicolumn{3}{c|}{SSIM}                    & \multicolumn{3}{c|}{S-CIELAB} 	\\	
      \cline{2-10}
                        &	Separate RGB      &	ScSR		&	MCcSR	&	Separate RGB      &	ScSR		&	MCcSR	&	Separate RGB      &	ScSR		&	MCcSR		  \\ \hline
comic					&	28.37	          &	28.25	    &\tb{28.51} &	0.74 	          &	0.74 	    &\tb{0.76} 	&	2.80e4            &	3.00e4      &\tb{2.71e4}      \\	\hline
baboon					&	26.95	          &	26.95	    &\tb{27.11}	&	0.53 	          &	0.52 	    &\tb{0.55} 	&	9.93e4            &	1.02e5      &\tb{9.57e4}     	 \\  \hline
pepper					&	36.14	          &	36.20	    &\tb{36.30}	&	0.79 	          &	0.79 	    &\tb{0.80} 	&	1.93e4            &	1.91e4      &\tb{1.73e4}    	 \\  \hline
bird					&	37.71	          &	37.59	    &\tb{38.02}	&	0.92 	          &	0.91	    &\tb{0.93} 	&	7.28e3            &	7.81e3      &\tb{6.50e3}     	 \\  \hline
      \end{tabular}
\endgroup
\end{table*}

\subsection{Effect of Color Regularizers: Edge Correlations}
In this part, we provide evidence that indeed our edge similarity prior is effective and encourages edge similarity among color channels in RGB space.
First we decompose an image into its constituent RGB channels. Then, on each channel we apply a high-pass edge detector filter (in this case the same filter that was applied in the learning phase, i.e. $\mat S_r, \mat S_g$ and $ \mat S_b$). Finally we find the cross correlation of edge information between RGB channels.  In this manner, we will obtain three correlation coefficient values between R and G, G and B,  and B and R channels. We report the average of these three correlation coefficients as an indicator of correlation of color information among channels. Ideally, with our proposed method we expect that edges across color channels be more consistent and similar, since we advocate for edge similarity using our optimization formulation. We also report the  average of the edge differences among color channels, e.g. $\|\mat S_r \vect y_{h_r} - \mat S_g \vect y_{h_g}\| $. These results are reported in Table \ref{Tab:EdgeCorrelation} for our MCcSR and two of leading state-of-the-art SR methods, i.e. ANR and ScSR. The results in Table \ref{Tab:EdgeCorrelation} are further averaged over the images in \emph{set 5}.
It is readily apparent that MCcSR exhibits the strongest edge correlations across the R, G, B color channels. In fact the MCcSR values are closest to the ground truth edge correlations in Table \ref{Tab:EdgeCorrelation}.

\begin{table}
\centering
      \caption{Effectiveness of color constraints on edge similarities in SR for a scaling factor of 3.}
      \label{Tab:EdgeCorrelation}
      \begin{tabular}{ |l||c|c|c|c| }
      \hline
                                            &	ANR               &	ScSR		&	MCcSR     &Groundtruth		     \\ \hline
 edge corr-coef             			  	&	0.8356	          &	0.8431	    &\tb{0.8511}  & 0.8785         \\	\hline
 edge differences					        &	35.87	          &	35.63	    &\tb{35.47}   & 35.25      \\	\hline
      \end{tabular}
\end{table}

\subsection{Robustness to Noise}

An often made assumption in single image SR is that the input images are clean and free of noise which is likely to be violated in many real world applications. Classical methods deal with noisy images by first denoising and filtering out the noise and then performing super-resolution. The final output of such a procedure highly depends on the denoising technique itself and the artifacts introduced in the denoising procedure may remain or even get magnified after super-resolution.

Similar to \cite{YangAndWright:SparseSR_TIP2010}, the parameter $\lambda$ in (\ref{Eq:MainOptProb}) is tuned based on the noise level of the input image and can control the smoothness of output results. We argue that our approach not only benefits from the noise robustness of ScSR \cite{YangAndWright:SparseSR_TIP2010}, but the  additional correlation information from multi-channels can help in further recovering more cleaner images.

We add different levels of Gaussian noise to the LR image input to test the robustness of our algorithm to noise and compare our results with ScSR method which has demonstrated success \cite{YangAndWright:SparseSR_TIP2010} in SR in the presence of noise. With a scaling factor of $3$, we chose the range of standard deviation of noise from $4$ to $12$ and similar to \cite{YangAndWright:SparseSR_TIP2010} set $\lambda$ to be one tenth of noise standard deviation. Likewise, we made the choice of $\tau$ in (\ref{Eq:MainOptProb}) using a cross-validation procedure to suppress noise. Fig \ref{Fig:NoisyBaboon} shows the SR results of an image with different levels of noise in comparison with ScSR and bicubic methods.
 Table \ref{Tab:Noise} reports the average PSNR, SSIM and S-CIELAB error measures of reconstructed images from different levels of noisy images. In all cases, MCcSR outperforms the competition.
 \begin{table}
\centering
      \caption{Average performance under different noise levels.}
      \label{Tab:Noise}
      \begin{tabular}{|l||c|c|c|c|c|c|c| }
\hline
Measure							&	Method                  &   $\sigma=0$ &	$\sigma=4$	&	$\sigma=6$	&	$\sigma=8$	&	$\sigma=12$		 \\  \hline
\multirow{3}{*}{PSNR}			&	Bicubic					&   33.08      &	32.99		&	32.75		&	32.50		&	31.88	 \\	 \cline{2-7}
								&	ScSR					&   34.00      &	33.95		&	33.92		&	33.90		&	33.86	 \\  \cline{2-7}
								&	MCcSR					&   34.14      &	34.11		&	34.09		&	34.09		&	34.07	 \\  \hline
\multirow{3}{*}{SSIM}			&	Bicubic					&   0.745      &	0.731		&	0.698		&	0.672		&	0.619	 \\	 \cline{2-7}
								&	ScSR					&   0.774      &	0.772		&	0.766		&	0.761		&	0.752	 \\	 \cline{2-7}
								&	MCcSR					&   0.785      &	0.783		&	0.780		&	0.775		&	0.768	 \\	 \hline
\multirow{3}{*}{\scriptsize{SCIELAB}}&	Bicubic					&	2.79E4     &	2.92E4  	&	4.40E4		&	5.25E4		&	6.31E4	 \\ \cline{2-7}
								&	ScSR					&	2.28E4	   &	2.31E4		&	2.36E4		&	2.39E4		&	2.43E4	 \\  \cline{2-7}
								&	MCcSR					&	2.14E4	   &	2.16E4		&	2.20E4		&	2.21E4		&	2.23E4	 \\  \hline

      \end{tabular}
\end{table}
\begin{figure}
  \centering
  \includegraphics[width=.98\columnwidth]{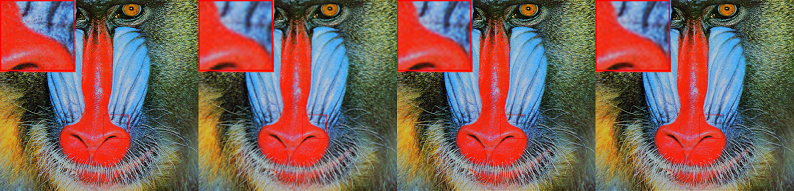}\\
  \includegraphics[width=.98\columnwidth]{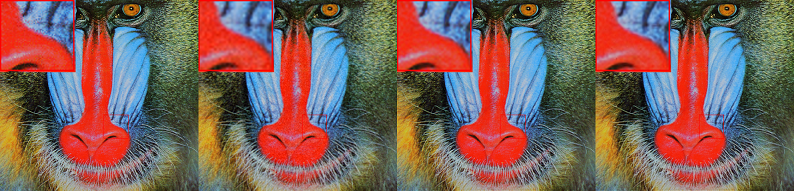}\\
  \includegraphics[width=.98\columnwidth]{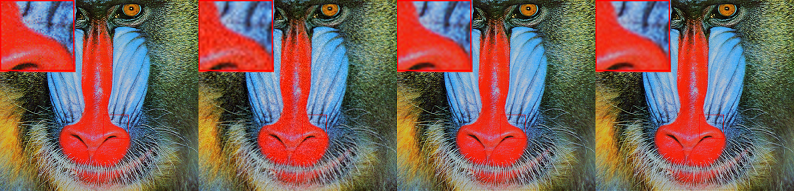}\\
  \includegraphics[width=.98\columnwidth]{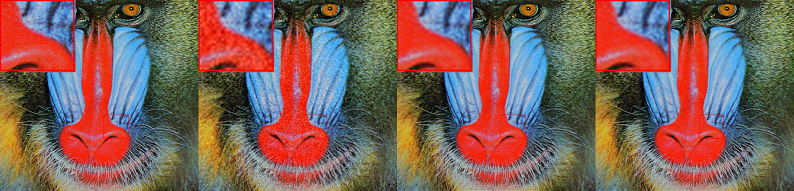}
    \caption{{SR performance under different noise standard deviations: 4,6,8,12 (from top to bottom ) with different methods: Original, bicubic, MCcSR, ScSR (from left to right) }}\label{Fig:NoisyBaboon}
\end{figure}
\section{Conclusion and Future work}
\label{Sec:Conclusion}
In this section, we extend sparsity based super-resolution to multiple color channels. We demonstrate that by using color information and cross channel constraints, significant improvement over single (luminance) channel sparsity based SR methods can be achieved.
In particular, edge similarities among color bands are exploited as cross channel correlation constraints. These additional constraints lead to new optimization problems both in the sparse coding and learning steps for which we present tractable solutions. Experimental results show the merits of our proposed method both visually and quantitatively.
 While our work offers one possible way to capture cross-channel color constraints, chrominance geometry can be captured via alternative quantitative formulations as in \cite{Srinivas:ColorSR_CIC2011, Farsiu:ColorDemosaicSR_TIP2006, Keren:ColorSR_1999MachineVision, Dai:SoftCutColorSR_2009TIP}. Incorporating these as constraints or regularizers in a sparsity based color SR framework forms a viable direction for future work. In addition to this, in the next section we tackle the problem of image super resolution from a deep learning viewpoint where a convolutional neural network in conjunction with image priors is learned for SR task.


\chapter{Contribution III: Deep Super Resolution via Exploiting Image Structures}
\label{chapter:contrib3}

\section{Introduction}
In the previous chapter, address the problem of super-resolution using priors such as edge correlations and also sparse priors. Along the same line, we address the SR task in single images from the viewpoint of deep neural networks in this chapter. We explore the use of image structures and physically meaningful priors in deep structures and convolutional networks. We first provide an overview of deep neural networks and their fundamental building blocks such as convolutional layers and then proceed with the review of existing super-resolution methods using deep neural networks.

\subsection{Background: Convolutional Neural Networks}
In recent years, deep neural networks and specifically convolutional neural network have seen a surge of applications and provided promising results in many different areas of computer vision and machine learning. These areas include classification and recognition \cite{krizhevsky2012imagenet, simonyan2014very}, segmentation, localization and  object detection \cite{girshick2014rich}, generative and descriptive networks \cite{goodfellow2014generative, karpathy2015deep} and image recovery and retrieval \cite{dong2014learning, zhou2017detecting}, etc.\cite{kamani2017skeleton}.
In machine learning, a Convolutional Neural Network (CNN) is a class of deep feed-forward artificial neural networks that has successfully been applied to analyzing visual imagery. These applications span from visual image recognition to inverse problems. A CNN consists of an input and an output layer, as well as multiple hidden layers. The hidden layers can be either convolutional, pooling or fully connected which we will briefly introduce them here. They are made up of neurons that have learnable weights and biases. Each neuron receives some inputs, performs a linear operation and optionally followed by a non-linearity. The explicit assumption in CNN architectures is that the inputs are images, which enables us to encode specific properties into the architecture \cite{CNN_Tutorial}.

As we described above, a simple CNN is a sequence of layers, and every layer of a CNN transforms one data cube (volume of activations) to another through a differentiable function. A  simple CNN architecture for super-resolution task could have the following building blocks and architecture: Input, Conv, ReLU, etc. In more detail, the input will hold the raw pixel values of the image (or the data cube corresponding to the activations of the previous layer.). Conv layer will essentially compute the convolution of the input cube with a set of 3D filters and generate the output volume. ReLU layer will apply an element-wise activation function which is non-linear. This leaves the size of the volume unchanged. Figure \ref{fig:simpleCNN} shows a regular 3-layer Neural Network and a simple CNN architecture.

\begin{figure}
  \centering
  \includegraphics[width=\textwidth]{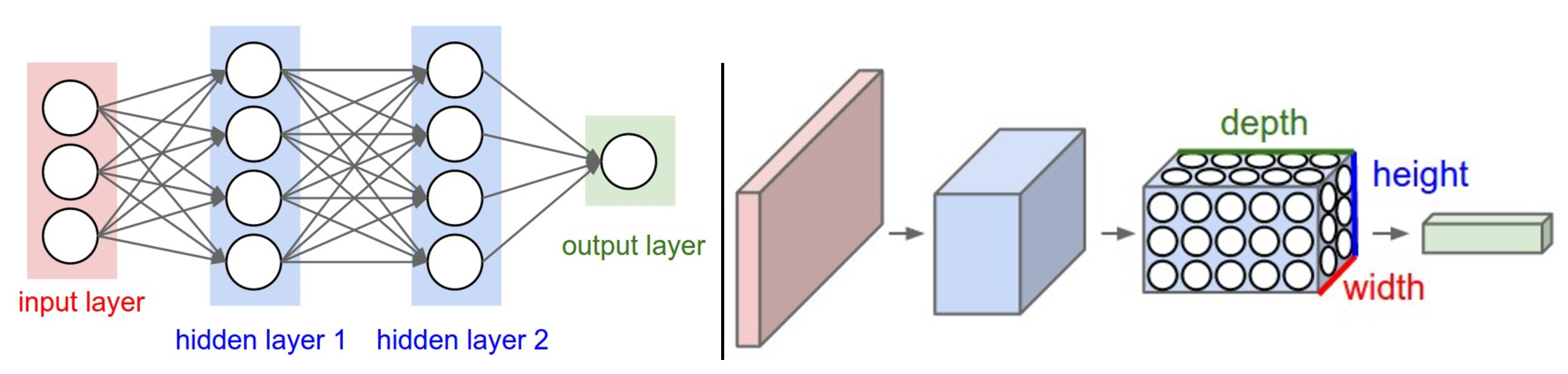}
  \caption{Left: A regular fully connected 3-layer Neural Network. Right: A 3-layer CNN. Neurons are arranged in three dimensions (width, height, depth), as visualized. Every layer transforms the 3D input data cube to a 3D output data cube of activations. In this example, input layer holds the image information, so its width and height would be the dimensions of the image, and the depth would be 3  for Red, Green, and Blue channels \cite{CNN_Tutorial}.}\label{fig:simpleCNN}
\end{figure}

by going through all the layers (feed forward), CNN transform the original image layer by layer from the original raw pixel values to the final output which can be class scores, image features, etc, depending on the purpose of the CNN model. Note that Conv layers contain parameters such as weight as biasses that should be learned and other such as ReLU, etc do not have associated parameters and will implement a fixed function. The parameters in the Conv layers will be trained with gradient descent so that the outputs that the CNN computes are consistent with the ground-truth data in the training set for each image.

The Conv layer is the core building block of a Convolutional Network that does most of the computational heavy lifting. The Conv layer consists of a set of learnable filters. Every filter is spatially small, meaning along width and height has a small receptive window, but extends through the full depth of the input volume. For instance, a typical filter on the first layer of a CNN might have a size of 5x5x3 (a 5x5 spatial size and over width and height, and 3 over depth because images have 3 color channels). During the forward pass, we convolve (slide the filter over) the input with each filter and compute dot products between the entries of the filter and the input at any position. Intuitively, the network learns filters that can activate when there is some simple type of visual feature such as an edge of some orientation, contour, etc. or some sort of color feature on the first layer, or eventually entire object parts, wheel-like patterns or more complex features on higher layers of the network. After training, we will have an entire set of filters in each layer, each of which produces a separate 2-D activation map. These activation maps are then stacked along depth dimension and produce the output data cube. An illustration of such process is shown in Fig. \ref{fig:conv_math}

\begin{figure}
  \centering
  \includegraphics[width=\textwidth]{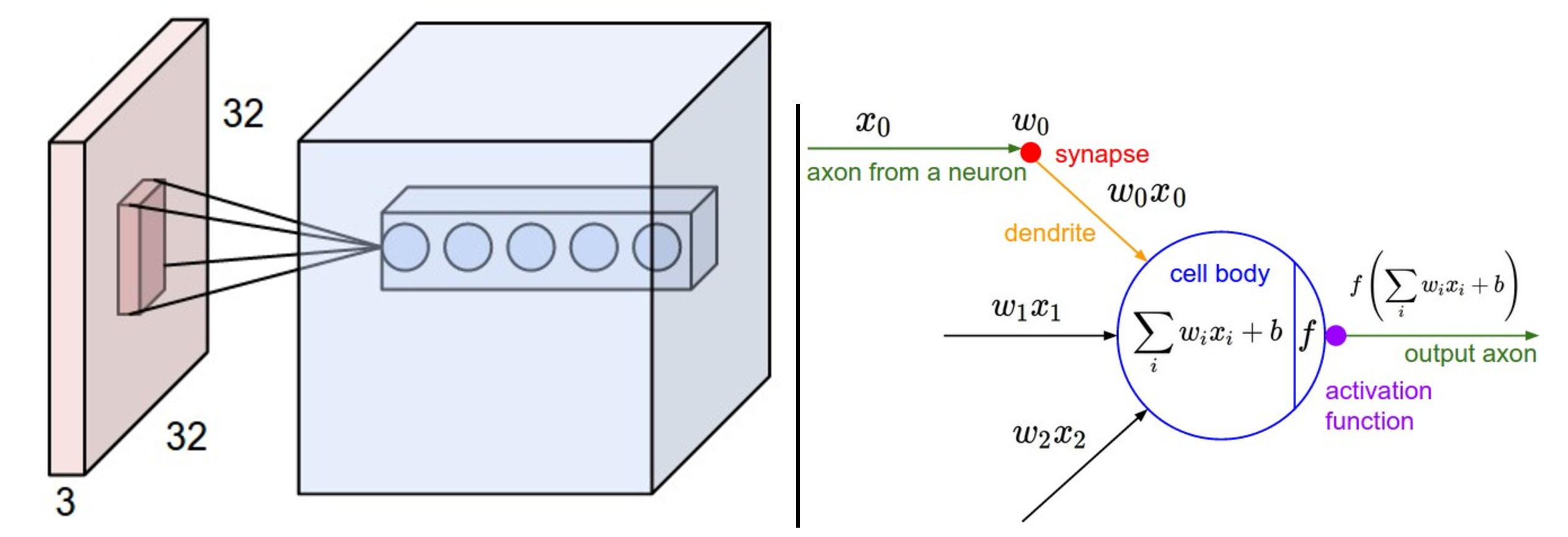}
  \caption{Left: An example input volume in red (e.g. a 32x32x3 image), and an example volume of data cubes (neurons) in the first Convolutional layer. Each neuron in the convolutional layer is connected only to a local region (spatial size of the filter) in the input volume but is connected to the full depth (i.e. all color channels). Note that in this example, there are 5 neurons in the output cube along the depth, resulting from 5 different convolutions but all looking at the same spatial region in the input. Right: The neurons compute a dot product (convolution) of their weights ($w_i$) with the input ($x_i$) followed by addition of a bias ($b$) term and then a non-linearity ($f()$) such as ReLU \cite{CNN_Tutorial}.}\label{fig:conv_math}
\end{figure}

\subsection{Deep Super Resolution}
Deep learning promotes the design of large-scale networks \cite{hinton2006fast,bengio2007greedy,poultney2006efficient} for a variety of problems including SR and recent advances have seen a surge of deep learning approaches for image super-resolution. Invariably, a network, e.g.\ a deep convolutional neural network (CNN) or auto-encoder is trained to learn the relationship between low and high-resolution image patches.

Among the first deep learning based super-resolution methods, Dong \emph{et al.} \cite{dong2014learning} trained a deep convolution neural network (SRCNN) to accomplish the image super-resolution task. In this work, the training set comprises of example LR inputs and their corresponding HR output images which were fed as training data to the SRCNN network. 
Combined with sparse coding methods, \cite{tiantong16deep} proposed a coupled network structure utilizing middle layer representations for generating SR results which reduced training and testing time. In different approaches, Cui {\em et al.} \cite{cui2014deep}  proposed a cascade network to gradually upscale LR images after each layer, while \cite{wang2015self} trained a high complexity convolutional auto-encoder called Deep Joint Super Resolution (DJSR) to obtain the SR results. Self-examples of images were explored in \cite{huang2015single} where training sets exploit self-example similarity, which leads to enhanced results. However, similar to SRCNN, DJSR suffers from the expensive computation in training and processing to generate the SR images.

Recently, residual net \cite{he2016deep} has shown great ability at reducing training time and faster convergence rate. Based on this idea, a Very Deep Super-Resolution (VDSR) \cite{Kim_2016_VDSR} method is proposed which emphasizes on reconstructing the residuals (differences) between LR and HR images rather than putting too much effort on reconstructing low-frequency details of HR images.
VDSR uses 20 convolutional layers producing state-of-the-art results in super-resolution and takes significantly shorter training time for convergence; however, VDSR is massively parameterized with these 20 layers. 

In this chapter, we address the problem of single image super-resolution from deep learning standpoint. We first show that the deep learning methods combined with transform domain data provide state-of-the-art performance in image super-resolution. Then we show how these methods suffer from performance degradation in low training scenarios and how to alleviate this problem by exploiting image priors as domain knowledge.

In the first part of this chapter, we propose to apply super-resolution in the transform domain particularly Wavelet domain. Wavelet coefficients prediction for super-resolution has been applied successfully to multi-frames SR. For instance, \cite{wahed2007image, ji2009robust,  demirel2009improved,robinson2010efficient} used multi-frames images to interpolate the missing details in the wavelet sub-bands to enhance the resolution. Several different interpolation methods for wavelet coefficients in single image SR were studied as well. \cite{anbarjafari2010image} used straightforward bicubic interpolation to enlarge the wavelet sub-bands to produce SR results in the spatial domain. \cite{nguyen2000efficient} explored interlaced sampling structure in the low-resolution data for wavelet coefficients interpolation. \cite{jiji2004single} formed a minimization problem to learn the suitable wavelet interpolation with a smooth prior. Since the detailed wavelet sub-bands are often sparse, it is suitable to apply sparse coding methods to estimate detailed wavelet coefficients and can significantly refine image details. Methods \cite{mallat2010super, tappen2003exploiting, dong2011image} used different interpolations related to sparse coding. Other attempts \cite{kinebuchi2001image,zhao2003wavelet} utilize Markov chains and  
\cite{chavez2014super} used the nearest neighbor to interpolate wavelet coefficients. However, due to limited training and straightforward prediction procedures, these methods are not powerful enough to process general input images and fail to deliver state-of-the-art SR results, especially compared to more recent deep learning based methods for super-resolution.

In the second part, we investigate the performance of deep learning methods for super-resolution in low training regime and propose to exploit image priors to alleviate the resulting performance degradation.

{\bf Motivations:} Most of the deep learning based image super-resolution methods work on spatial domain data and aim to reconstruct pixel values as the output of the network. In this work, we explore the advantages of exploiting transform domain data in the  SR task especially for capturing more structural information in the images to avoid artifacts. In addition to this and motivated by the promising performance of VDSR and residual nets in super-resolution task, we propose our Deep Wavelet network for super-resolution (DWSR). Residual networks benefit from sparsity of the input and output and the fact that learning networks with sparse activations is much easier and more robust. This motivates us to exploit spatial wavelet coefficients which are naturally sparse. More importantly, using residuals (differences) of wavelet coefficients as training data pairs further enhances the sparsity of training data resulting in more efficient learning of filters and activations. In other words, using wavelet coefficients encourages activation sparsity in middle layers as well as the output layer. Consequently, residuals for wavelet coefficients themselves become sparser and therefore easier for the network to learn. In addition to this, wavelet coefficients decompose the image into sub-bands which provide structural information depending on the types of wavelets used. For example, Haar wavelets provide vertical, horizontal and diagonal edges in wavelet sub-bands which can be used to infer more structural information about the image. Essentially our network uses complementary structural information from other sub-bands to predict the desired high-resolution structure in each sub-band.

On the other hand, deep learning methods have shown promising performance in super-resolution and many other tasks in presence of abundant training which means thousands or millions of training data points are available. However, they suffer in cases where training data is not readily available. In this chapter, we are also investigating the performance of such deep structures in low training scenarios and show that their performance drops significantly. We look for remedies to this performance degradation by exploiting prior knowledge about the problem. This could be in terms of prior knowledge about the structure of images, or inter-pixel dependencies.

The \textbf{main contributions} of this chapter are the following: 1) We proposed a Deep Wavelet Super Resolution (DWSR) framework to promote sparsity and also provide complementary structural information about the image. This complementary structural information in wavelet coefficients helps in the better reconstruction of SR results with fewer artifacts. 2) In addition to a wavelet prediction network, we built on top of residual networks which fit well to the wavelet coefficients due to their sparsity promoting nature and further enhancing it by inferring residuals. 
3) We illustrate the decrease in performance of deep super-resolution methods in low training scenarios and provide image priors as a solution to this problem. Image priors in low training scenarios enhance the recovery of high-resolution images despite having much less training data available. We now begin by the first two contributions in this chapter.



\section{Super Resolution in Transform Domain}

\subsection{2D Discrete Wavelet Transformation (2dDWT)}
To perform a 1D Discrete Wavelet Transformation, a signal $x[n]\in \mathbb{R}^N$ is first passed through a half band high-pass filter $G_H[n]$ and a low-pass filter $G_L[n]$, which are defined as (for Haar (``db1'') wavelet):
\begin{equation}
G_H[n]=\begin{cases}1, &n=0\\-1, &n=1\\0, &\text{otherwise}\end{cases}, G_L[n]=\begin{cases}1, &n=0,1\\0, &\text{otherwise}\end{cases}
\end{equation}
After filtering, half of the samples can be eliminated according to the Nyquist rule, since the signal now has a frequency bandwidth of $\pi/2$ radians instead of $\pi$.

Any digital image $x$ can be viewed as a 2D signal with index $[n,m]$ where $x[n,m]$ is the pixel value located at $n$th column and $m$th row. The 2D signal $x[n,m]$ can be treated as 1D signals among the rows $x[n,:]$ at a given $n$th column and among the columns $x[:,m]$ at a given $m$th row.
A 1-level 2D wavelet transform of an image can be captured by following the procedure in Figure \ref{fig:2dDWT} along rows and columns, respectively.
As mentioned earlier, we are using Haar kernels in this work.
\begin{figure}[t]
      \centering
      \includegraphics[width=\linewidth]{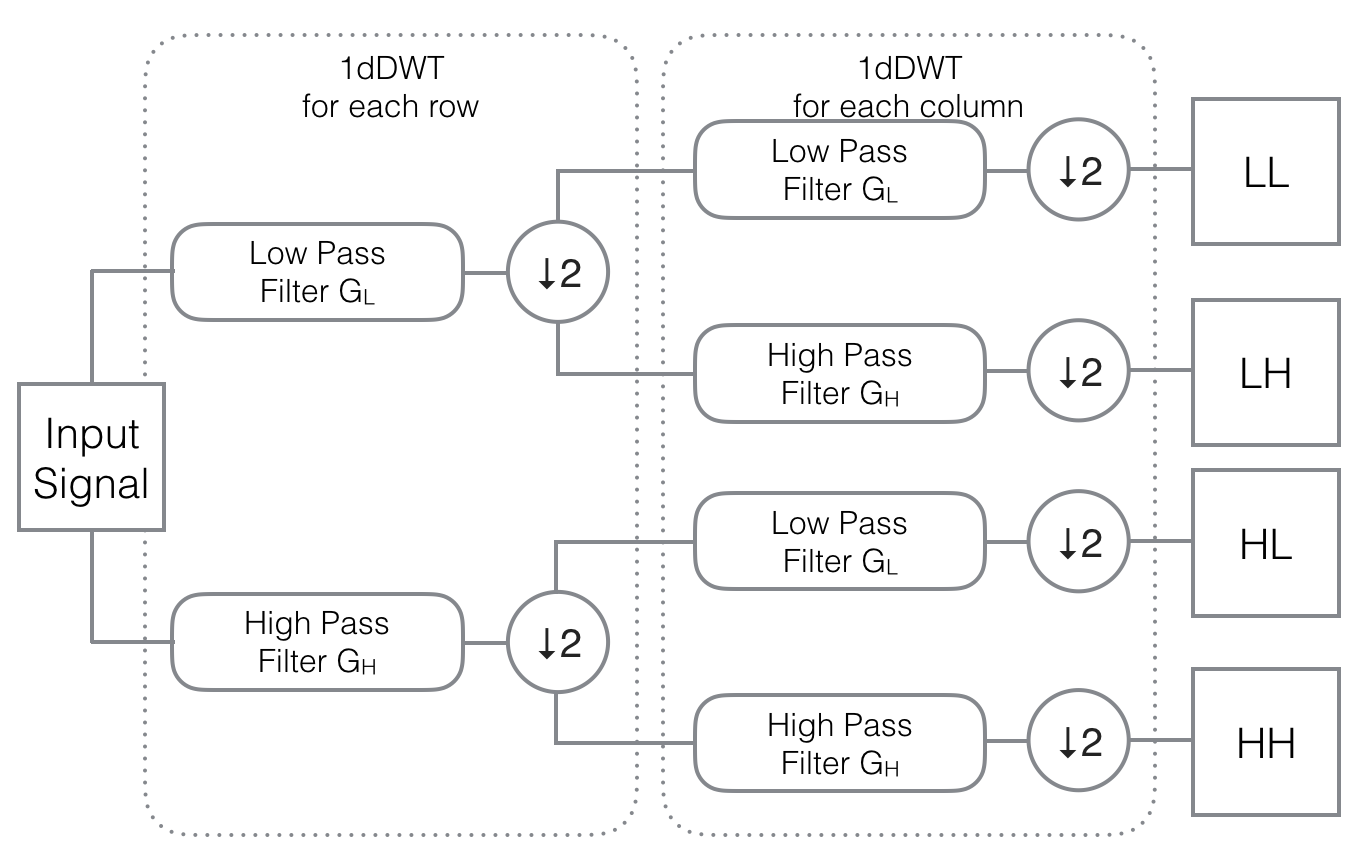}
      \caption{The procedure of 1-level 2dDWT decomposition.}
      \label{fig:2dDWT}
\end{figure}




An example of 1-level 2dDWT decomposition with Haar kernels is shown in Figure \ref{fig:2DIDWT}. The right part of Figure \ref{fig:2DIDWT} is the notation of each sub-band of wavelet coefficients. It is clear that the 2dDWT captures the image details in four sub-bands: average (LL), vertical(HL), horizontal(LH) and diagonal(HH)  information, which are corresponding to each wavelet sub-bands coefficients. Note that after 2dDWT decomposition, the combination of four sub-bands always has the same dimension as the original input image.


The 2d Inverse DWT (2dIDWT) can trace back the 2dDWT procedure by inverting the steps in Figure \ref{fig:2dDWT}. This allows the prediction of wavelet coefficients to generate SR results. Detailed wavelet decomposition introduction can be found in \cite{mallat2008wavelet}.


\subsection{Deep Wavelet Prediction for Super-resolution (DWSR)}
The SR can be viewed as the problem of restoring the details of the image given an input LR image. This viewpoint can be combined with wavelet decomposition. As shown in Figure \ref{fig:2DIDWT}, if we treat the input image as an LL output of 1-level 2dDWT, predicting the HL, LH and HH sub-bands of the 2dDWT will give us the missing details of the LL image. Then one can use  2dIDWT to gather the predicted details and generate the SR results. With Haar wavelet, the coefficients of 2dIDWT can be computed as:
\begin{equation}
\begin{cases}
A=a+b+c+d\\
B=a-b+c-d\\
C=a+b-c-d\\
D=a-b-c+d\\
\end{cases}
\end{equation}
where $A, B, C, D$ and $a, b, c, d$ represent the pixel values from corresponding image/sub-bands.
\begin{figure}[t]
      \centering
      \includegraphics[trim ={30, 20, 30, 8},clip, width=\linewidth]{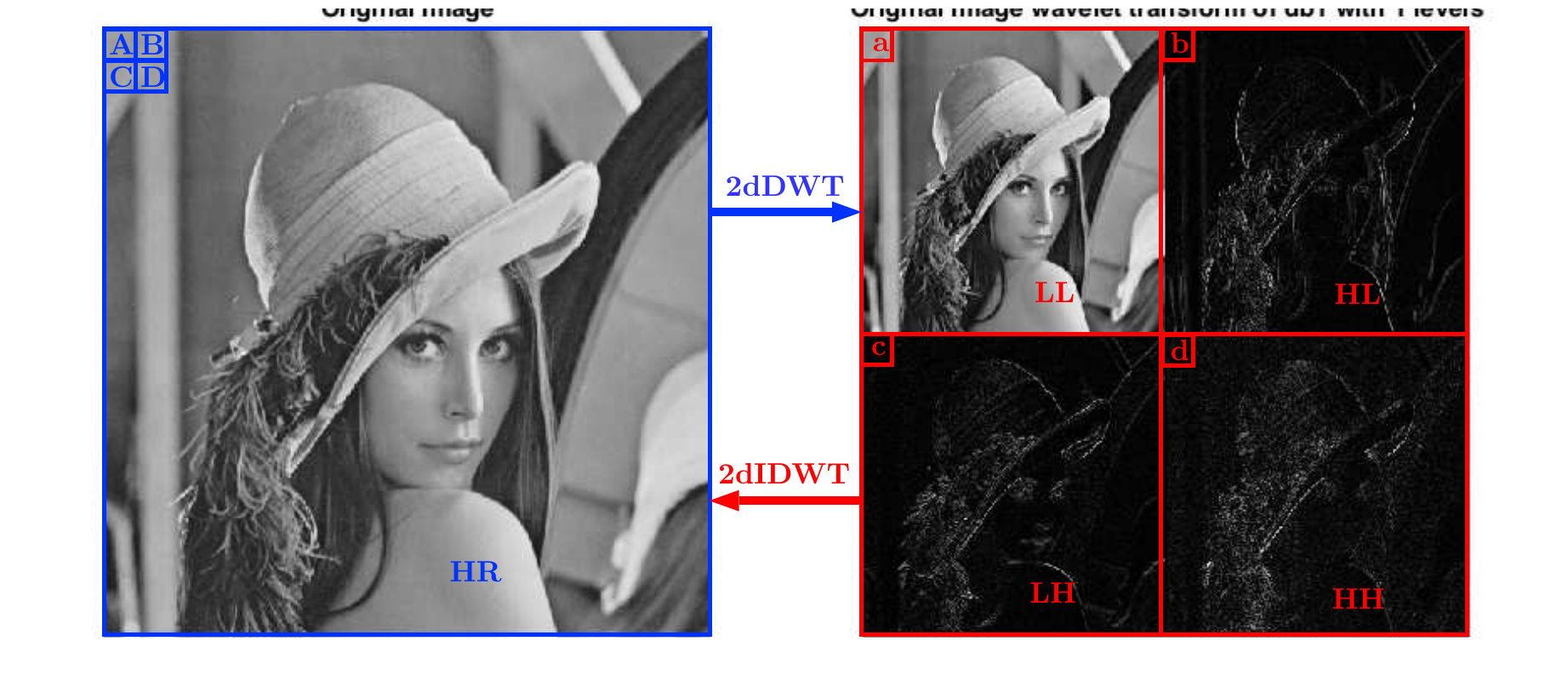}
      \caption{The 2dDWT and 2dIDWT. $A, B, C, D$ are four example pixels located in a $2\times2$ grid at the top left corner  of HR image. $a, b, c, d$ are four pixels from the top left corner of four sub-bands correspondingly.}
      \label{fig:2DIDWT}
\end{figure}
      \begin{figure*}[ht]
      \centering
      \includegraphics[width=\textwidth]{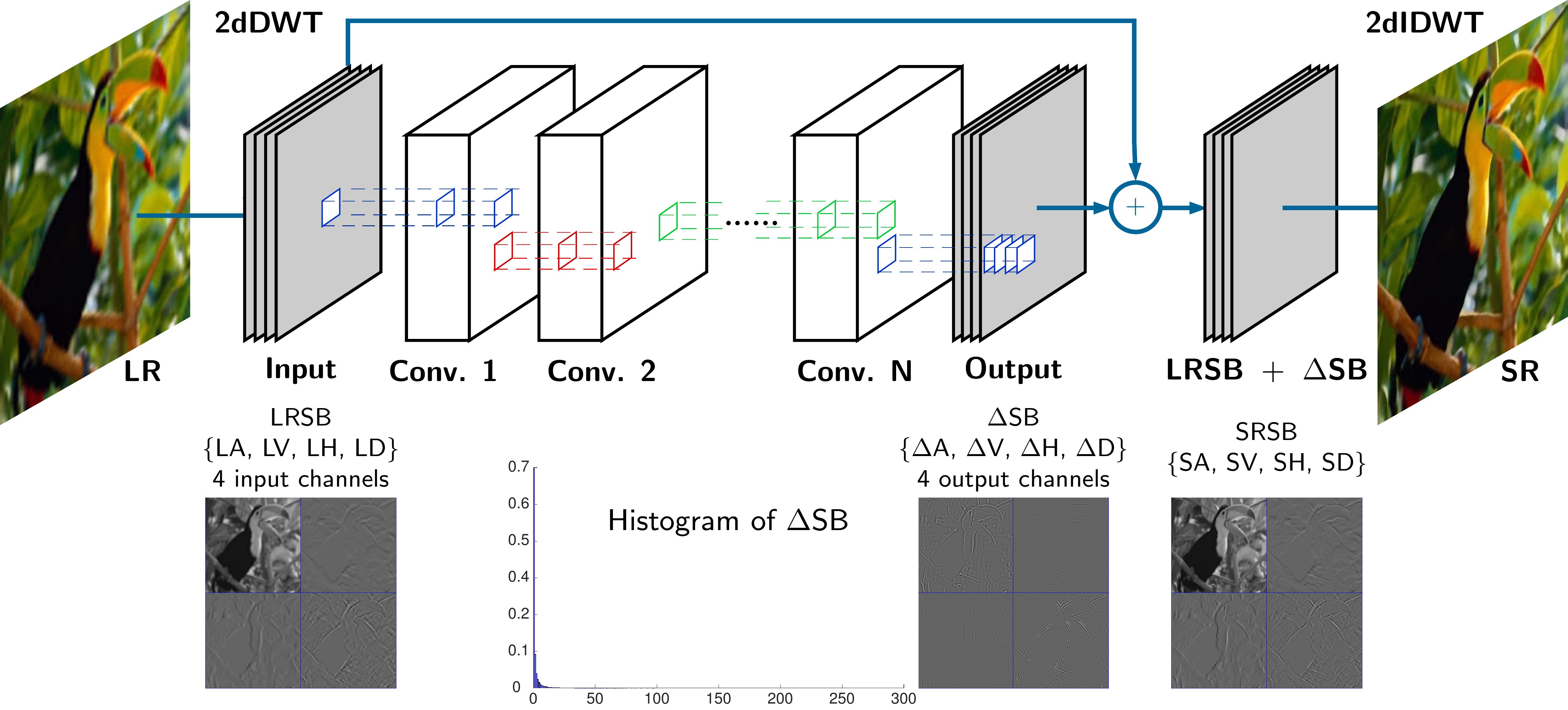}
      \caption{Wavelet prediction for SR network structure: there are input layers which takes four channels and output layers produce four channels. The network body has repeated $N$ same-sized layers with ReLU activation functions (Here $N=10$). One example of the input LRSB and network output $\Delta$SB are plotted. The histogram of all coefficients in $\Delta$SB is drawn to illustrate the sparsity of the outputs.}
      \label{fig:str}
      \end{figure*}

Therefore, with the help of wavelet transformation, the SR problem becomes a wavelet coefficients prediction problem. In this section, we propose a new deep learning based method to predict details of wavelet sub-bands from the input LR image. To the best of our knowledge, DWSR is the first deep learning based wavelet SR method \cite{guo2017deep}.

\subsection{Network Structure}\label{netstr}
The structure of the proposed network is illustrated in Figure \ref{fig:str}. The proposed network has a  deep structure similar to the residual network \cite{he2016deep} with two input and output layers with 4 channels.
While most of the deep learning based SR methods have only one channel for input and output, our network takes four input channels into consideration and produces four corresponding channels at the output. Inspired by recent advance in deep residual learning for super-resolution, we design our layers to have 64 filters of size $4\times3\times3$ in the first layer and 4 filters of size $64\times3\times3$ in the last layer.
In the middle part of the network, the network has $N$ same-sized hidden layers with $64\times3\times3\times64$ filters each. The output of each layer, except the output layer, is fed into ReLU activation function to generate a nonlinear activation map.

Usually, the CNN based SR methods only take valid regions into consideration while feeding forward the inputs. For example, in SRCNN \cite{dong2014learning}, the network has three layers with the filter size of $9\times9$, $1\times1$ then $5\times5$, from which we can compute the cropped out information width, which is $(9+1+5-3)=12$ pixels. During the training process, SRCNN takes in sub-images of size $33\times33$, but only produce outputs of size $21\times21$. This procedure is unfavorable in our deep model since the final output could be too small to contain any useful information.

To solve this problem, we use zero padding at each layer to keep the outputs having the same sizes as the inputs. In this manner, we can produce the same size final outputs as the inputs. Later, the experiments show that with the special wavelet sparsity, the padding will not affect the quality of the SR results.

\subsection{Training Procedure}
To train the network, the low-resolution training images are enlarged by bicubic interpolation with the original downscale factor. Then the enlarged LR images are passed through the 2dDWT with Haar wavelet to produce four LR wavelet Sub-Bands (LRSB) which is denoted as:
\begin{equation}
\text{LRSB} = \{\text{LA, LV, LH, LD}\} := \text{2dDWT\{LR\}}
\label{LRSB}\end{equation}
where the LA, LV, LH, and LD are sub-bands containing wavelet coefficients for average, vertical, horizontal and diagonal details of the LR image, respectively. 2dDWT\{LR\} denotes the 2dDWT of the LR image.

The transformation is also applied on the corresponding HR training images to produce four HR wavelet Sub-Bands (HRSB):
\begin{equation}
\text{HRSB}=\text{\{HA,HV,HH,HD\}}:=\text{2dDWT\{HR\}}
\label{HRSB}\end{equation}
where the HA, HV, HH, and HD denote the sub-bands containing wavelet coefficients for average, vertical, horizontal and diagonal details of the HR image, respectively.

Then the difference $\Delta$SB (residual) between corresponding LRSB and HRSB is computed as:
\begin{equation}\label{dSB}
\begin{split}
\Delta \text{SB} &= \text{HRSB} - \text{LRSB} \\
&= \{\text{HA}-\text{LA}, \text{HV}-\text{LV},\text{HH}-\text{LH},\text{HD}-\text{LD}\}\\&=\{\Delta\text{A}, \Delta\text{V}, \Delta\text{H}, \Delta\text{D}\}\end{split}
\end{equation}
$\Delta$SB is the target that we desire the network to produce with input LRSB. The feeding forward procedure is denoted as $f(\text{LRSB})$.

The cost of the network outputs is defined as:
\begin{equation}
\text{cost}=\frac{1}{2}\|\Delta \text{SB}-f(\text{LRSB})\|^2_2\label{cost}
\end{equation}

The weights and biases can be denoted as $(\Theta,\bm{b})$. Then the optimization problem is defined as:
\begin{equation}(\Theta, \bm{b}) = \arg\min_{\Theta,\bm{b}}\frac{1}{2}\|\Delta \text{SB}-f(\text{LRSB})\|^2_2 + \lambda\|\Theta\|^2_2\end{equation}where the $\|\Theta\|^2_2$ is the standard weight decay regularization with parameter $\lambda$.

Essentially, we want our network to learn the differences between wavelet sub-bands of LR and HR images. By adding these differences (residual) to the input wavelet sub-bands, we will get the final super-resolution wavelet sub-bands.
\subsection{Generating SR Results}

To produce SR results, the bicubic enlarged LR input images are transformed by 2dDWT to produce LRSB as Equation (\ref{LRSB}). Then LRSB is fed forward through the trained network to produce $\Delta$SB. Adding LRSB and  $\Delta$SB together generates four SR wavelet Sub-Bands (SRSB) denoted as:
\begin{equation}\begin{split}\text{SRSB} &=\{\text{SA}, \text{SV}, \text{SH}, \text{SD}\}\\&= \text{LRSB} + \Delta \text{SB} \\&= \{\text{LA} + \Delta \text{A}, \text{LV} + \Delta \text{V}, \text{LH} + \Delta \text{H}, \text{LD} + \Delta \text{D}\}
\end{split}\end{equation}

Finally, 2dIDWT generates the SR image results:
\begin{equation}
\text{SR} = \text{2dIDWT}\{\text{SRSB}\}
\end{equation}

\subsection{Understanding Wavelet Prediction}

Training in wavelet domain can boost up the training and testing procedure. Using wavelet coefficients encourages activation sparsity in hidden layers as well as the output layer. Moreover, by using residuals, wavelet coefficients themselves become sparser and therefore easier for the network to learn sparse maps rather than dense ones. The histogram in Figure \ref{fig:str} illustrates the sparse distribution of all the $\Delta$SB coefficients. This high level of sparsity further reduces the training time required for the network， resulting in more accurate super-resolution results.


In addition, training a deep network is actually to minimize a cost function which is usually defined by $l2$ norm. This particular norm is used because it homogeneously describes the quality of the output image compared to the ground truth. The image quality is then quantified by the assessment metric PSNR. However, SSIM \cite{wang2004image} has been proven to be a conceptually better way to describe the quality of an image (comparing to the target) which unfortunately cannot be easily optimized. Nearly all the SR methods use SSIM as final testing metric but it is not emphasized in the training procedure.


However, DWSR encourages the network to produce more structural details. As shown in Figure \ref{fig:str}, the SRSB has more defined structural details than LRSB after adding the predicted $\Delta$SB. With Haar wavelet, every fine detail has different intensity of coefficients spreading in all four sub-bands. Overlaying four sub-bands together can enhance the structural details the network taking in by providing additional relationships between structural details. At a given spatial location, the first sub-band gives the general information of the image, following three detailed sub-bands provide horizontal/vertical/diagonal structural information to the network at this location. The structural correlation information between the sub-bands helps the network weights forming in a way to emphases the fine details.

By taking more structural similarity into account while training, the proposed network increases both the PSNR and SSIM assessments to deliver a visually improved SR result (see Section \ref{SRArti} for quantitative comparisons). Moreover, benefiting from wavelet domain information, DWSR produces SR results with less artifacts while other methods suffer from misleading artificial blocks introduced by bicubic (see Section \ref{SRArti}).

\section{Experimental Evaluation}
\label{Sec:DWSR_Experiment}

\subsection{Data Preparation}

During the training phase, the NTIRE \cite{Timofte_2017_CVPR_Workshops} 800 training images are used without augmentation. The NTIRE HR images $\{Y_i\}_{i=1}^{800}$ are down-sampled by the factor of $c$. Then the down-sampled images are enlarged busing bicubic interpolation by the same factor $c$ to form the LR training images $\{X_i\}_{i=1}^{800}$. Note that the image $Y_i$ is cropped so that its width and height be multiple of $c$. Therefore $X_i$ and $Y_i$ have the same size where $Y_i$ represents the HR training image, $X_i$ represents the corresponding LR training image. Patches of size $41\times41$ pixels sub-images with $10$ pixels overlapping are then extracted from $X_i$ and $Y_i$ for training.

For each sub-image from $X_i$, the LRSB is computed as Equation (\ref{LRSB}). For each corresponding sub-image from $Y_i$, the HRSB is computed as Equation (\ref{HRSB}). Then the residual $\Delta$SB is computed as Equation (\ref{dSB}).

During the testing phase, several standard testing data sets are used. Specifically, Set5 \cite{bevilacqua2012low}, Set14 \cite{zeyde2010single}, BSD100 \cite{MartinFTM01}, Urban100 \cite{huang2015single} are used to evaluate our proposed method DWSR.

Both training and testing phases of DWSR only utilize the luminance channel information. For color images, Cr and Cb channels are directly enlarged by bicubic interpolation from LR images. These enlarged chrominance channels are combined with SR luminance channel to produce color SR results.

\subsection{Training Settings}

During the training process, several training techniques are used. The gradients are clipped to 0.01 by norm clipping option in the training package. We use Adam optimizer as described in  \cite{kingma2014adam} to updates $\Theta$ and $\bm{b}$. The initial learning rate is $0.01$ and decreases by $25\%$ every $20$ epochs. The weight regulator is set to $1\times10^{-3}$ to prevent over-fitting. 
Other than input and output layers, the DWSR has $N=10$ same-sized convolutional hidden layers with the filter size of $64\times3\times3\times64$. This configuration results in a network with only half of parameters in VDSR \cite{Kim_2016_VDSR}.

The training scheme is implemented with TensorFlow \cite{tensorflow2015-whitepaper} package with Python 2.7 interaction interface. We use one GTX TITAN X GPU 12 GB for both the training and testing.
\begin{figure*}[t]
\centering
\includegraphics[width=\linewidth]{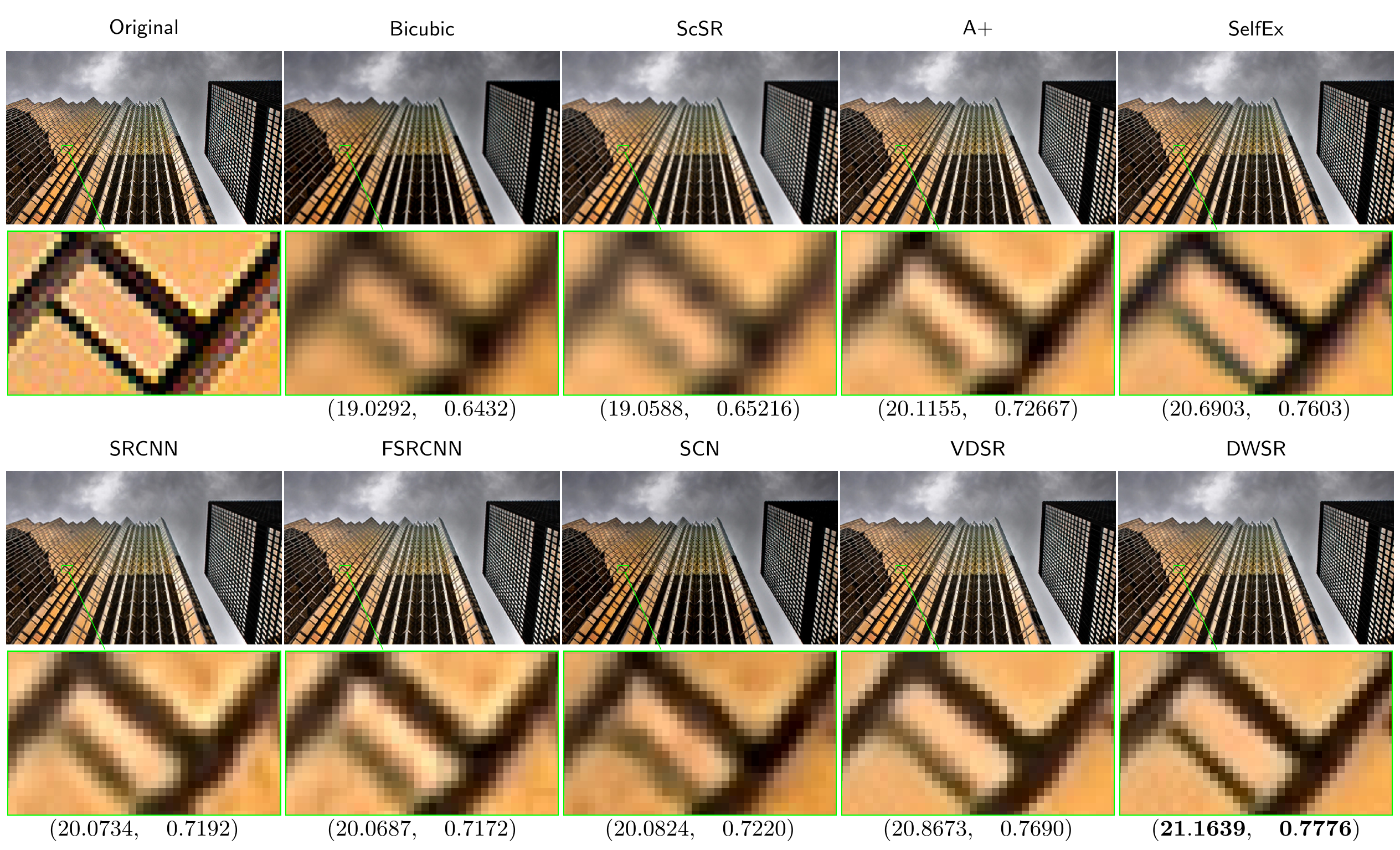}
\caption{Test image No.$19$ in Urban100 data set. From top left to bottom right are results of: ground truth, bicubic,  ScSR, A$+$, SelfEx, SRCNN, FSRCNN, SCN, VDSR, DWSR. The numeral assessments are labeled as $\left(\text{PSNR},\ \text{SSIM}\right)$. DWSR (bottom right) produces more defined structures with better SSIM and PSNR than state-of-the-art methods.}\label{19}
\end{figure*}

\subsection{Convergence Speed}

Since the gradients are clipped to a numerical large norm, with the high initial learning rate, DWSR reaches convergence with a really fast speed and produces practical results (see following reported evaluations). Figure \ref{fig:epoch} shows the convergence process during the training by plotting the evaluation of cost over training epochs. After 100 epochs, the network is fully converged and ($\Theta$,$\bm{b}$) is used for testing. The training procedure for 100 epochs takes about 4 hours to finish with one GPU.
\begin{figure}[h]
      \centering
      \includegraphics[width=\linewidth,trim = {25, 0, 25, 15}, clip]{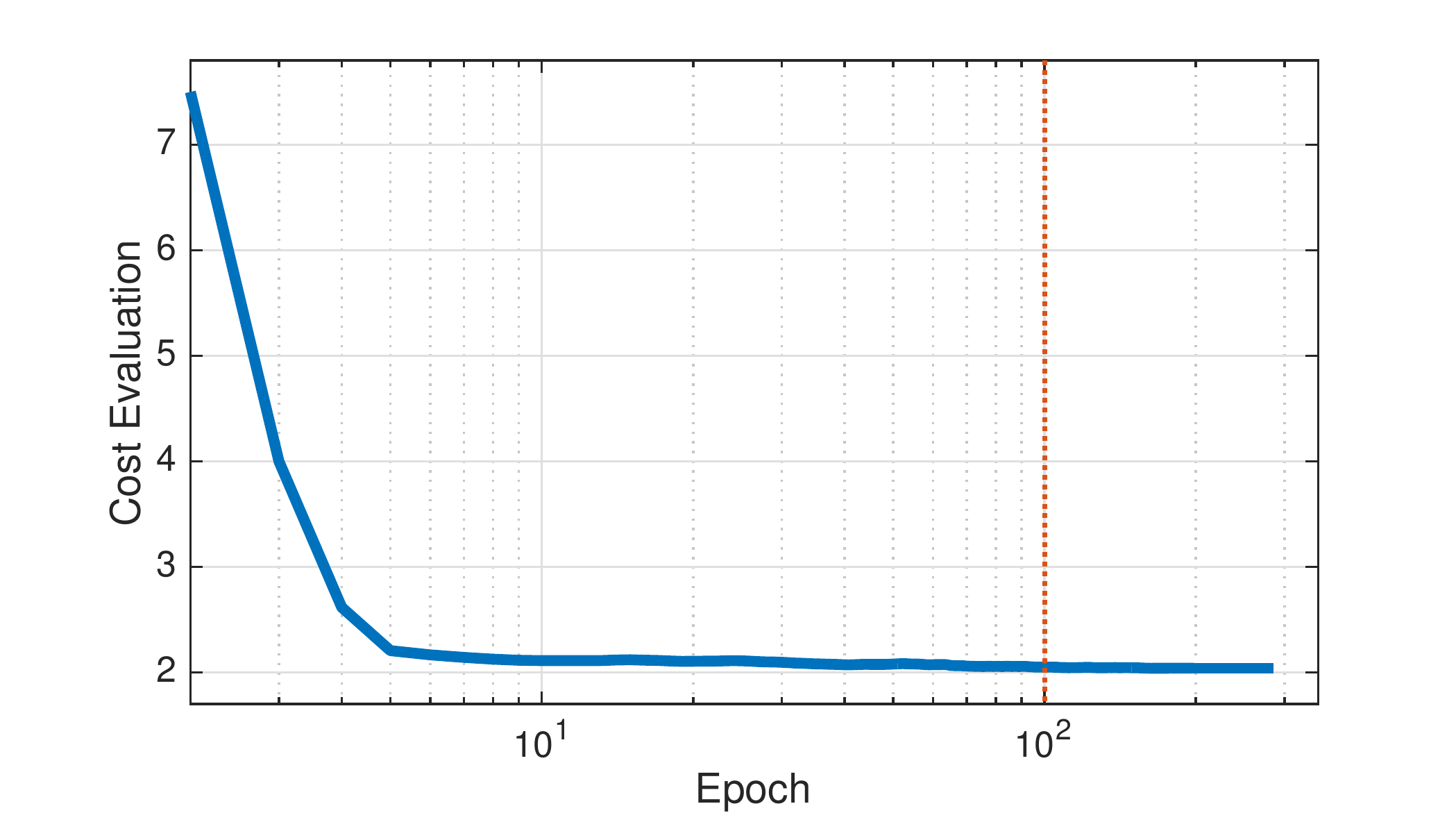}
      \caption{The evaluations of cost function (\ref{cost}) over training epochs for training scale factor 4. At 100 epoch, the network training convergences.} \label{fig:epoch}
\end{figure}

\begin{figure*}[t]
\centering
\includegraphics[width=\linewidth]{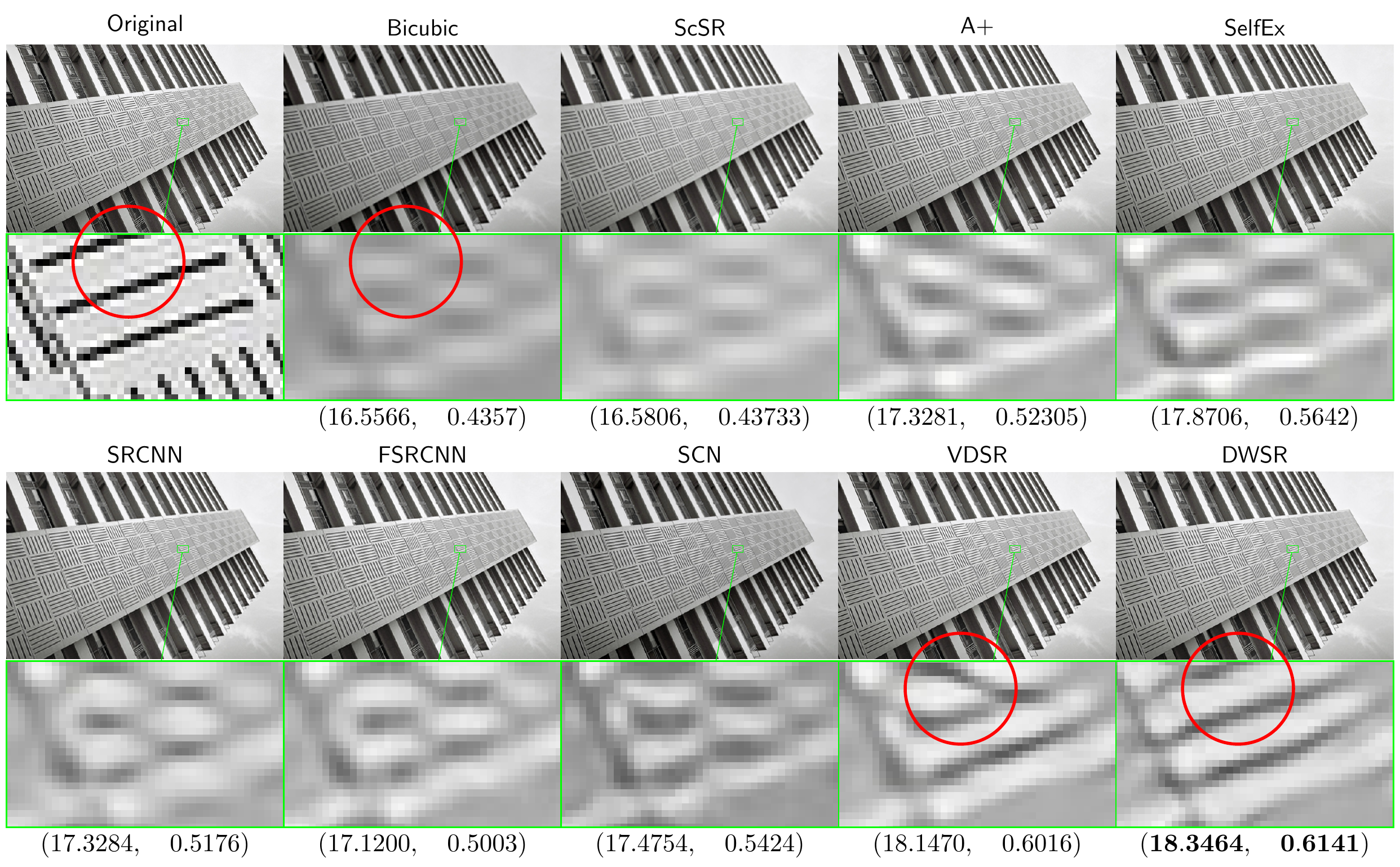}
\caption{Test image No.$92$ in Urban100 data set. From top left to bottom right are results of: ground truth, bicubic, ScSR, A$+$, SelfEx, SRCNN, FSRCNN, SCN, VDSR, DWSR. The numeral assessments are labeled as $\left(\text{PSNR},\ \text{SSIM}\right)$. DWSR (bottom right) produces more fine structures with better SSIM and PSNR than state-of-the-art methods. Also note DWSR does not produce artifacts diagonal edges in the \tcb{red circled region}.}\label{92}
\end{figure*}

\subsection{Comparison with State-of-the-Art} \label{testSet}

We compare DWSR with several state-of-the-art methods and use Bicubic as the baseline reference\footnote{Please refer to \url{http://signal.ee.psu.edu/DWSR.html} for high-quality color images and to download our code.}.

ScSR \cite{yang2010image} and A$+$ \cite{timofte2014a+} are selected to represent the sparse coding based and dictionary learning based methods. For deep learning based methods, DWSR is compared with SCN \cite{wang2015deep}, SelfEx \cite{huang2015single}, FSRCNN \cite{dong2016accelerating}, SRCNN \cite{dong2014learning} and VDSR \cite{Kim_2016_VDSR}. We use publicly published testing codes from different authors, the tests are carried on GPU as mentioned above for deep learning based methods. For FSRCNN, SRCNN and sparse based methods we use their public CPU testing codes.

Table \ref{PSNR} shows the summarized results of  PSNR and SSIM evaluations. The best results are shown in \tcr{\bf red} and second best are shown in \tcb{blue}. DWSR has a clear advantage on the large scaling factors owing to its reliance on incorporating the structural information and correlation from the wavelet transform sub-bands. For large scale factors, DWSR delivers better results than the best known method (VDSR) with {\bf only half} parameters benefiting from training in wavelet feature domain.

Table \ref{TIME} shows the execution time of different methods. Since DWSR only has half of the parameters than the most parameterized method (VDSR) and benefiting from really sparse network activations, DWSR takes much less time to apply super-resolution. For 2K images in NTIRE testing set, DWSR takes less than 0.1s to produce the outputs of the network including the loading time from GPU. 

Figure \ref{fig:PSNRvsTime} compares state-of-the-art methods for image super-resolution in terms of PSNR and running time. DWSR shows the best trade-off between running time and PSNR value among other methods.
\begin{figure}
      \centering
      \includegraphics[trim ={12, 0, 12, 8},clip, width=\linewidth]{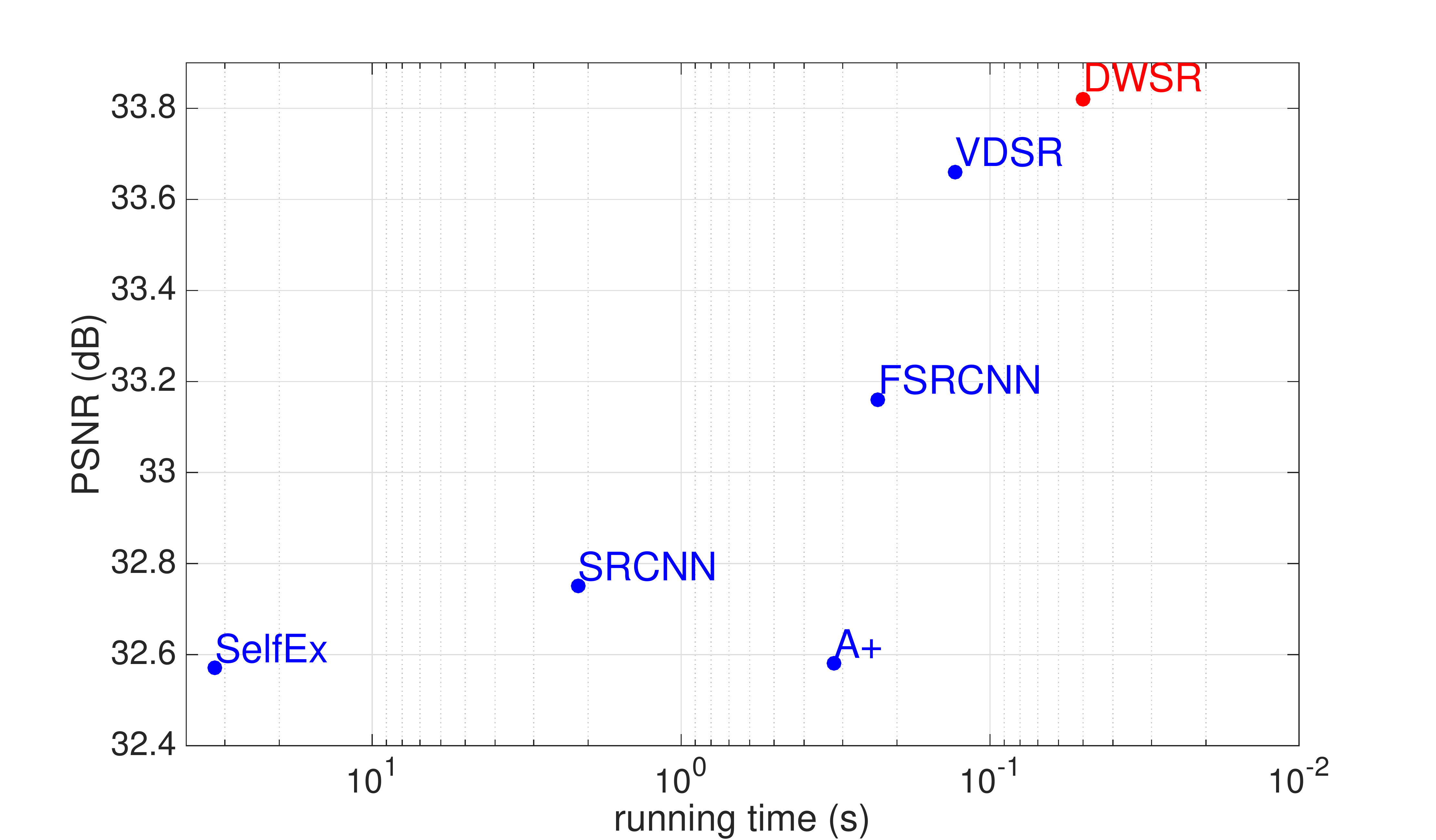}
      \caption{DWSR and other state-of-the-art methods reported PSNR with scale factor of 3 on Set5. For experimental setup see Section \ref{testSet}.}
      \label{fig:PSNRvsTime}
\end{figure}

Figure \ref{19} shows SR results of a testing image from Urban100 dataset with scale factor 4. Overall, deep learning based methods produce better results than sparse coding based and dictionary learning based methods. Compared to SRCNN, DWSR produces more defined structures benefiting from training in wavelet domain. Compared to VDSR, DWSR results give higher PSNR and SSIM values using less than half parameters of VDSR with a faster speed. Visually, the edges are more enhanced in DWSR than other state-of-the-art methods and are clearly illustrated in the enlarged areas. The image generated by DWSR has less artifacts that are caused by initial bicubic interpolation of LR image and results in sharper edges which are consistent with the ground truth image. Also quite clearly, DWSR has an advantage on reconstructing edges especially diagonal ones due to the fact that this structural information is prominently emphasized with sub-bands in Haar wavelets coefficients.
\begin{table*}[t]
\centering
\caption{PSNR and SSIM result comparisons with other approaches for 4 different datasets.}
\label{PSNR}\resizebox{\textwidth}{!}{
\begin{tabular}{|r|l||c||c||c||c||c||c||c||c|}
\hline
\begin{tabular}[c]{@{}c|c@{}}PSNR&SSIM\end{tabular}&
                               & \begin{tabular}[c]{@{}c@{}}Bicubic\\{[}Baseline{]}\end{tabular}
                               & \begin{tabular}[c]{@{}c@{}}ScSR\\{[}TIP 10{]}\end{tabular}
                               & \begin{tabular}[c]{@{}c@{}}A+\\{[}ACCV 14{]}\end{tabular}
                               & \begin{tabular}[c]{@{}c@{}}SelfEx\\{[}CVPR 15{]}\end{tabular}
                               & \begin{tabular}[c]{@{}c@{}}FSRCNN\\{[}ECCV 16{]}\end{tabular}

                               & \begin{tabular}[c]{@{}c@{}}SRCNN\\{[}PAMI 16{]}\end{tabular}
                               & \begin{tabular}[c]{@{}c@{}}VDSR\\{[}CVPR 16{]}\end{tabular}
                               & \begin{tabular}[c]{@{}c@{}}DWSR\\{[}ours{]}\end{tabular}                                   \\ \hline
Set5
& \begin{tabular}[c]{@{}l@{}}x2\\ x3\\ x4\end{tabular}
& \begin{tabular}[c]{@{}c|c@{}}33.64&0.9292\\ 30.39&0.8678\\ 28.42&0.8101\end{tabular}
& \begin{tabular}[c]{@{}c|c@{}}35.78&0.9485\\ 31.34&0.8869\\ 29.07&0.8263\end{tabular}
& \begin{tabular}[c]{@{}c|c@{}}36.55&0.9544\\ 32.58&0.9088\\ 30.27&0.8605\end{tabular}
& \begin{tabular}[c]{@{}c|c@{}}36.47&0.9538\\ 32.57&0.9092\\ 30.32&0.8640\end{tabular}
& \begin{tabular}[c]{@{}c|c@{}}36.94&0.9558\\ 33.06&0.9140\\ 30.55&0.8657\end{tabular} 
& \begin{tabular}[c]{@{}c|c@{}}36.66&0.9542\\ 32.75&0.9090\\ 30.48&0.8628\end{tabular}
& \begin{tabular}[c]{@{}c|c@{}}\tcr{\bf 37.52}&\tcr{\bf 0.9586}\\ \tcb{33.66}&\tcb{0.9212}\\ \tcb{31.35}&\tcb{0.8820}\end{tabular}
& \begin{tabular}[c]{@{}c|c@{}}\tcb{37.43}&\tcb{0.9568}\\ \tcr{\bf 33.82}&\tcr{\bf 0.9215}\\ \tcr{\bf 31.39}&\tcr{\bf 0.8833}\end{tabular} \\ \hline
Set14
& \begin{tabular}[c]{@{}l@{}}x2\\ x3\\ x4\end{tabular}
& \begin{tabular}[c]{@{}c|c@{}}30.22&0.8683\\ 27.53&0.7737\\ 25.99&0.7023\end{tabular}
& \begin{tabular}[c]{@{}c|c@{}}31.64&0.8940\\ 28.19&0.7977\\ 26.40&0.7218\end{tabular}
& \begin{tabular}[c]{@{}c|c@{}}32.29&0.9055\\ 29.13&0.8188\\ 27.33&0.7489\end{tabular}
& \begin{tabular}[c]{@{}c|c@{}}32.24&0.9032\\ 29.16&0.8196\\ 27.40&0.7518\end{tabular}
& \begin{tabular}[c]{@{}c|c@{}}32.54&0.9088\\ 29.37&0.8242\\ 27.50&0.7535\end{tabular} 
& \begin{tabular}[c]{@{}c|c@{}}32.42&0.9063\\ 29.28&0.8209\\ 27.40&0.7503\end{tabular}
& \begin{tabular}[c]{@{}c|c@{}}\tcb{33.02}&\tcb{0.9102}\\ \tcb{29.77}&\tcb{0.8308}\\ \tcb{28.01}&\tcb{0.7664}\end{tabular}
& \begin{tabular}[c]{@{}c|c@{}}\tcr{\bf 33.07}&\tcr{\bf 0.9106}\\ \tcr{\bf 29.83}&\tcr{\bf 0.8308}\\ \tcr{\bf 28.04}&\tcr{\bf 0.7669}\end{tabular} \\ \hline
B100
& \begin{tabular}[c]{@{}l@{}}x2\\ x4\end{tabular}
& \begin{tabular}[c]{@{}c|c@{}}29.55&0.8425\\ 25.96&0.6672\end{tabular}
& \begin{tabular}[c]{@{}c|c@{}}30.77&0.8744\\ 26.61&0.6983\end{tabular}
& \begin{tabular}[c]{@{}c|c@{}}31.21&0.8864\\ 26.82&0.7087\end{tabular}
& \begin{tabular}[c]{@{}c|c@{}}31.18&0.8855\\ 26.84&0.7106\end{tabular}
& \begin{tabular}[c]{@{}c|c@{}}31.66&0.8920\\ 26.92&0.7201\end{tabular} 
& \begin{tabular}[c]{@{}c|c@{}}31.36&0.8879\\ 26.84&0.7101\end{tabular}
& \begin{tabular}[c]{@{}c|c@{}}\tcr{\bf 31.85}&\tcr{\bf 0.8960}\\ \tcb{27.23}&\tcb{0.7238}\end{tabular}
& \begin{tabular}[c]{@{}c|c@{}}\tcb{31.80}&\tcb{0.8940}\\ \tcr{\bf 27.25}&\tcr{\bf 0.7240}\end{tabular}  \\ \hline
Urban100
& \begin{tabular}[c]{@{}l@{}}x2\\ x4\end{tabular}
& \begin{tabular}[c]{@{}c|c@{}}26.66&0.8408\\ 23.14&0.6573\end{tabular}
& \begin{tabular}[c]{@{}c|c@{}}28.26&0.8828\\ 24.02&0.7024\end{tabular}
& \begin{tabular}[c]{@{}c|c@{}}29.20&0.8938\\ 24.32&0.7186\end{tabular}
& \begin{tabular}[c]{@{}c|c@{}}29.54&0.8967\\ 24.78&0.7374\end{tabular}
& \begin{tabular}[c]{@{}c|c@{}}29.87&0.9010\\24.61&0.7270\end{tabular} 
& \begin{tabular}[c]{@{}c|c@{}}29.50&0.8946\\ 24.52&0.7221\end{tabular}
& \begin{tabular}[c]{@{}c|c@{}}\tcr{\bf 30.76}&\tcb{0.9140}\\ \tcb{25.18}&\tcb{0.7524}\end{tabular}
& \begin{tabular}[c]{@{}c|c@{}}\tcb{30.46}&\tcr{\bf 0.9162}\\ \tcr{\bf 25.26}&\tcr{\bf 0.7548}\end{tabular} \\ \hline
\end{tabular}}
\end{table*}
\begin{table*}[t]
\centering
\caption{Results of the execution time comparison to other approaches}
\label{TIME}\resizebox{\textwidth}{!}{
\begin{tabular}{|r|l|c|c|c|c|c|c|c|}
\hline
                               &
                               & \begin{tabular}[c]{@{}c@{}}ScSR\\{[}TIP 10{]}\end{tabular}
                               & \begin{tabular}[c]{@{}c@{}}A+\\{[}ACCV 14{]}\end{tabular}
                               & \begin{tabular}[c]{@{}c@{}}SelfEx\\{[}CVPR 15{]}\end{tabular}
                               & \begin{tabular}[c]{@{}c@{}}FSRCNN\\{[}ECCV 16{]}\end{tabular}

                               & \begin{tabular}[c]{@{}c@{}}SRCNN\\{[}PAMI 16{]}\end{tabular}
                               & \begin{tabular}[c]{@{}c@{}}VDSR\\{[}CVPR 16{]}\end{tabular}
                               & \begin{tabular}[c]{@{}c@{}}DWSR\\{[}ours{]}\end{tabular}                                   \\ \hline
Set5
& \begin{tabular}[c]{@{}l@{}}x2\\ x3\\ x4\end{tabular}
& \begin{tabular}[c]{@{}l@{}}80.22\\ 82.67\\ 84.88\end{tabular}
& \begin{tabular}[c]{@{}c@{}}0.58\\ 0.32\\ 0.24\end{tabular}
& \begin{tabular}[c]{@{}c@{}}45.76\\32.28\\ 29.32\end{tabular}
& \begin{tabular}[c]{@{}l@{}}0.30\\ 0.23\\ 0.26\end{tabular} 
& \begin{tabular}[c]{@{}c@{}}2.56\\ 2.63\\2.16\end{tabular}
& \begin{tabular}[c]{@{}l@{}}0.13\\0.13\\0.12\end{tabular}
& \begin{tabular}[c]{@{}c@{}}\bf 0.06\\\bf 0.05\\\bf 0.06\end{tabular} \\ \hline
Set14
& \begin{tabular}[c]{@{}l@{}}x2\\ x3\\ x4\end{tabular}
& \begin{tabular}[c]{@{}l@{}}86.12\\ 91.52\\ 89.25\end{tabular}
& \begin{tabular}[c]{@{}c@{}}0.85\\0.59\\0.32\end{tabular}
& \begin{tabular}[c]{@{}c@{}}112.3\\76.02\\66.06\end{tabular}
& \begin{tabular}[c]{@{}l@{}}0.32\\ 0.42\\ 0.39\end{tabular} 
& \begin{tabular}[c]{@{}c@{}}4.52\\4.25\\4.68\end{tabular}
& \begin{tabular}[c]{@{}l@{}}0.25\\0.26\\0.25\end{tabular}
& \begin{tabular}[c]{@{}c@{}}\bf 0.07\\\bf 0.08\\\bf 0.07\end{tabular} \\ \hline
B100
& \begin{tabular}[c]{@{}c@{}}x2\\ x4\end{tabular}
& \begin{tabular}[c]{@{}c@{}}98.03\\ 100.43\end{tabular}
& \begin{tabular}[c]{@{}c@{}}0.60\\0.26\end{tabular}
& \begin{tabular}[c]{@{}c@{}}62.02\\36.67\end{tabular}
& \begin{tabular}[c]{@{}c@{}}0.32\\ 0.39\end{tabular}
& \begin{tabular}[c]{@{}c@{}}2.65\\2.98\end{tabular}
& \begin{tabular}[c]{@{}l@{}}0.16\\0.26\end{tabular}
& \begin{tabular}[c]{@{}c@{}}\bf 0.09\\\bf 0.12\end{tabular} \\ \hline
Urban100
& \begin{tabular}[c]{@{}c@{}}x2\\ x4\end{tabular}
& \begin{tabular}[c]{@{}c@{}}1021.06\\ 1282.33\end{tabular}
& \begin{tabular}[c]{@{}c@{}}2.96\\1.21\end{tabular}
& \begin{tabular}[c]{@{}c@{}}663.66\\662.68\end{tabular}
& \begin{tabular}[c]{@{}c@{}}2.23\\2.35 \end{tabular}
& \begin{tabular}[c]{@{}c@{}}23.2\\25.6\end{tabular}
& \begin{tabular}[c]{@{}c@{}}0.98\\1.07\end{tabular}
& \begin{tabular}[c]{@{}c@{}}\bf 0.33\\\bf 0.38\end{tabular} \\ \hline
\end{tabular}}
\end{table*}
\subsection{Large Scaling Factor SR Artifacts} \label{SRArti}
Figure \ref{92} illustrates SR results from different methods with scale factor 4. DWSR produces more enhanced details than state-of-the-art methods. Moreover, since the scale factor is large for bicubic interpolations to keep the structural information, some artificial blocks are introduced during the bicubic  enlargement. Meanwhile nearly all the deep learning based methods are utilizing the bicubic interpolations as the starting point, these artificial blocks get more pronounced during the SR enhancements. Eventually, the enhancements on the artificial blocks produce artificial edges in the SR results. For instance, in Figure \ref{92}, these blocks and artificial edges are labeled within red circles for bicubic and VDSR. The diagonal edges are introduced by SR enhancement on the artificial blocks from bicubic enlargement, which are not present in the ground truth image.

However, DWSR utilizes wavelet coefficients to take in more structural correlation information into account which does not enhance the artificial blocks and produces edges more similar to the ground truth.

Our work so far presents a deep wavelet super-resolution (DWSR) technique that recovers the ``missing details'' by using (low-resolution) wavelet sub-bands as inputs. DWSR is significantly economical in the number of parameters compared to most state-of-the-art methods and yet achieves competitive or better results. We contend that this is because wavelets provide an image representation that naturally simplifies the mapping to be learned. While we used the Haar wavelet, effects of different wavelet basis can be examined in future work. Of particular interest could be to learn the ``optimal" wavelet basis for the SR task.

In the remainder of this chapter, we investigate the performance of deep learning methods for super-resolution in low training regime and propose to exploit image priors to alleviate the resulting performance degradation.

\section{Image Priors for Super Resolution}
\label{sec:img_priors}

The statistical information embedded in natural images have recently received much attention
from different communities for both understanding the human visual system and also designing effective image processing algorithms \cite{field1987relations}.
Natural images are different from images generated by random
noise in that they exhibit meaningful structures. Examples of such structures
are local regularities, such as edges and self-similarities, etc. Consequently, the natural images are only a tiny
fraction of the space of all the images that can be generated by
all the possibilities of pixel values.

Natural images have many unique statistical properties \cite{zontak2011internal, weiss2007makes}.
One of the most well known such properties is that they exhibit heavy-tailed
distribution when applying derivative filters onto them.
Intuitively natural images are locally smooth; therefore, local differences
will be small and  the distribution will decrease
faster than the Gaussian.
On the other hand,
natural images have many structural details such as edges, where
the derivative response can be large and it contributes to the
heavier tails than the Gaussian distribution.
This prior knowledge has been successfully applied in a wide range of applications, including image denoising \cite{weiss2007makes}, deblurring \cite{field1987relations} and super-resolution \cite{kim2010single}.
Apart from this heavy-tailed statistics characteristic of natural images, they also have many other statistical properties, such as scale invariance and similar joint statistics. The former states  that
the natural images have similar heavy-tailed distributions at
different scales \cite{zontak2011internal} and the latter means that the neighboring pixels in the natural images exhibit high statistical dependency \cite{roth2009fields}.

Image priors and statistics are a very active research topic and researchers are still investigating it. The most classical image prior, Gaussian model applied to the derivatives of images, is widely used due to its simplicity:
\bea
    P(\vect y) \propto \exp \Big( -\frac{\|\nabla \vect y \|_2^2}{\eta^2} \Big)
\eea
where $\nabla \vect y$ represents the gradient of the image $\vect y$. Gaussian prior has the advantage of having a closed form solution for many optimization problems; however, it usually fails to produce satisfying solutions as it also smoothens the image. To overcome this problem and preserve the edge structure, Laplacian prior is used which has been proved to preserve image discontinuities better:
\bea
    P(\vect y) \propto \exp \Big( -\frac{\|\nabla \vect y \|_1^2}{\eta^2} \Big). \label{Eq:NIP_l1}
\eea
Laplacian priors are related to $\ell_1$-norm regularization which promotes the sparsity in the solution. Such priors can preserve edges in the image; however, they do not capture natural images' statistics very well and the resulting images look piecewise linear. This is due to the fact that natural images follow a distribution with heavier tails than Laplacian or Gaussian. The solution presented for this problem is to use the hyper-Laplacian distribution for the edges in the image \cite{zhang2012generative,krishnan2009fast}:
\bea
    P(\vect y) \propto \exp \Big( -\frac{\|\nabla \vect y \|_\alpha }{\eta^2} \Big).
\eea
where the norm $\|\vect y\|_\alpha$ is defines as $\|\vect y\|_\alpha = \sum_i |y_i|^\alpha$.
As suggested by the literature \cite{krishnan2009fast,kim2010single}, the parameter $\alpha$ which controls the sparseness of the desired gradient in the natural images is usually picked between 0.5 and 0.8. More complicated priors that can take into account the long-range inter-relations of pixels can also be used to further capture statistical data in natural images. These priors can be set manually, or more interestingly they can be learned from training data and applied on unseen images for super-resolution. In this chapter, we take the image priors as suggested by Kim \emph{et al.} \cite{kim2010single,tappen2003exploiting} and improve upon them.
\small{\bea\resizebox{\textwidth}{!}{$
        P(\vect {y_h} | \vect {y_l}) = \frac{1}{C} \underbrace{ \prod_{\substack{\{i,j\} \\ \{s,t\}\in \mathcal{N}(i,j)} }  \exp \left[ - \left( \frac{|y_h(i,j) - y_h(s,t)|}{\sigma_N} \right)^\alpha \right] }_{\text{prefer strong edges (edge based prior)}}
                                    \underbrace{ \prod_{ \{i,j\}  }  \exp \left[ - \left( \frac{|\mathbb{T} \big( y_h(i,j) \big) - y_l(i,j)|}{\sigma_R} \right)^2 \right] }_{\text{Reconstruction is faithful to LR image}}
                                    $} \nonumber\\\label{Eq:NIP_orig}
\eea}
The above prior tries to capture natural image priors (NIP) and reconstruction model in one framework. $\vect{y_h}$ represents the estimated high resolution image and $\vect{y_l}$ denotes the corresponding low resolution image and $\mathcal{N}(i,j)$ represents a neighborhood of pixels at location $(i,j)$. For a given image, the second product
term ensures that  when the same downsampling kernel ($\mathbb{T}$) is applied on the super resolution result ($\vect{y_h}$), final super resolution result is prevented from flowing far away from the input low resolution image $\vect{y_l}$. In this form of NIP framework, the second term is the reconstruction constraint which measures the distance
between the input low-resolution image and an image reconstructed from the high-resolution configuration according to the down-sampling model (blurring and sub-sampling), while the first product term (NIP term) tends to penalize pixel value differences in the neighborhood of each pixel $(i,j)$. Subsequently, this distribution prefers a strong edge rather than
a set of small edges (such as ringing artifacts) and can be used to resolve the problem
of smooth edges.

To adapt the NIP prior and the reconstruction constraint to super-resolution problem in a learning-based method, we are revising the prior distribution in \eqref{Eq:NIP_orig} so that the reconstruction constraint is penalizing the difference between the estimated high-resolution image and the ground truth high-resolution image. This is a better fit for learning-based methods in super-resolution where the cost function is the difference between the inferred image and the ground truth. We then rewrite the NIP as follows:
\small{\bea \resizebox{\textwidth}{!}{$
        P(\vect {y_h} | \vect {y_g}) = \frac{1}{C} \prod_{\substack{\{i,j\} \\ \{s,t\}\in \mathcal{N}(i,j)} }  \exp \left[ - \left( \frac{|y_h(i,j) - y_h(s,t)|}{\sigma_N} \right)^\alpha \right]
                                    \underbrace{ \prod_{ \{i,j\}  }  \exp \left[ - \left( \frac{| y_h(i,j)  - y_g(i,j)|}{\sigma_R} \right)^2 \right] }_{\text{Revised to compare output with ground truth HR image}}
                                    $}
                                    \nonumber\\\label{Eq:NIP_prior}
\eea}

Also note that in the revised NIP prior no downsampling/blurring kernel ($\mathbb{T}$) is used. In the new NIP prior, which is specific to super resolution, we want the inferred super resolution result to be statistically close to the ground truth image.  The above  formulation is very similar to \eqref{Eq:NIP_orig} where the low resolution image $\vect{y_l}$ is now replaced with the ground truth high resolution image ($\vect{y_g}$). It essentially encourages the inferred high resolution image to be close to the ground truth using the NIP probabilistic prior. The reconstruction constraint here corresponds to a generative model and with the NIP provides a MAP framework where we can take the negative log-likelihood of the posterior and find the minimum of that. Essentially, maximizing the posterior using NIP priors leads to the following minimization problem:
\bea
\vect{y_h} =     &\arg\min\limits_{\vect{y_h}}   & \sum\limits_{\substack{\{i,j\} \\ \{s,t\}\in \mathcal{N}(i,j)} }  \left( \frac{|y_h(i,j) - y_h(s,t)|}{\sigma_N} \right) ^\alpha +
                                    \sum\limits_{ \{i,j\}  }  \left( \frac{| y_h(i,j)  - y_g(i,j)|}{\sigma_R} \right)^2 \label{Eq:NIP_cost1}\\
            = &   \arg\min\limits_{\vect{y_h}}  &  \frac{\sigma_R^2}{\sigma_N^\alpha}  \sum\limits_{\substack{\{i,j\} \\ \{s,t\}\in \mathcal{N}(i,j)} }   |y_h(i,j) - y_h(s,t)|^\alpha +
                                   \sum\limits_{ \{i,j\}  }  | y_h(i,j)  - y_g(i,j)|^2 \label{Eq:NIP_cost2}
\eea
Rewriting the MAP estimation in the form above helps us interpret the cost function often used for image super-resolution and also implement the new NIP cost function in an efficient manner using convolutions. The second sum in \eqref{Eq:NIP_cost2} is essentially summing up pixel level square differences between the estimated high-resolution image and the high-resolution ground truth image. This can be easily captured by $\| \mat{y_h}  - \mat{y_g} \|_F^2$ which is the error norm of a high resolution image and the super-resolution result. It is noteworthy to mention that this is the most commonly used cost function for image super-resolution in the deep learning frameworks which essentially back-propagate the gradient of error terms to the weights of the deep network. On the other hand, the first term in \eqref{Eq:NIP_cost2} is a \emph{local} error constraint on pixel values and summed for all the pixels in the images. If we assume a simple neighborhood $\mathcal{N}(i,j)$ to be the eight-neighborhood vicinity around any pixel, the NIP prior as defined above can be written as summation over $8$ filtered images that are also passed through a special non-linear activation function. Since the aforementioned filters are simple difference filters and are linear, they can be implemented with eight convolution filters (shown in Fig. \ref{Fig:Filters}) and followed by non-linear activation function i.e. $|\cdot|^\alpha$ with a learnable parameter $\alpha$. Despite the regular NIP assumption that $\alpha$ is fixed, it is assumed here to be learnable so we can find the best $\alpha$ that fits the training images we have. The overall cost function in its new form is written as follows:
\bea
    &\arg\min\limits_{\vect{y_h}}   &  \frac{\sigma_R^2}{\sigma_N^\alpha}  \left(\sum\limits_{k=1}^{8}   \|\mat{y_h} \ast \mat{F_k}  \|_\alpha \right) +
                                    \| \mat{y_h}  - \mat{y_g} \|_F^2 \label{Eq:NIP_cost3}
\eea
It is noteworthy to mention  that this new cost function can be implemented using convolutions followed by a simple non-linearity layer. This makes it efficient for implementation purposes in the deep leaning structures using convolutional neural networks.
\begin{figure}
  \centering
  \includegraphics[width=\textwidth]{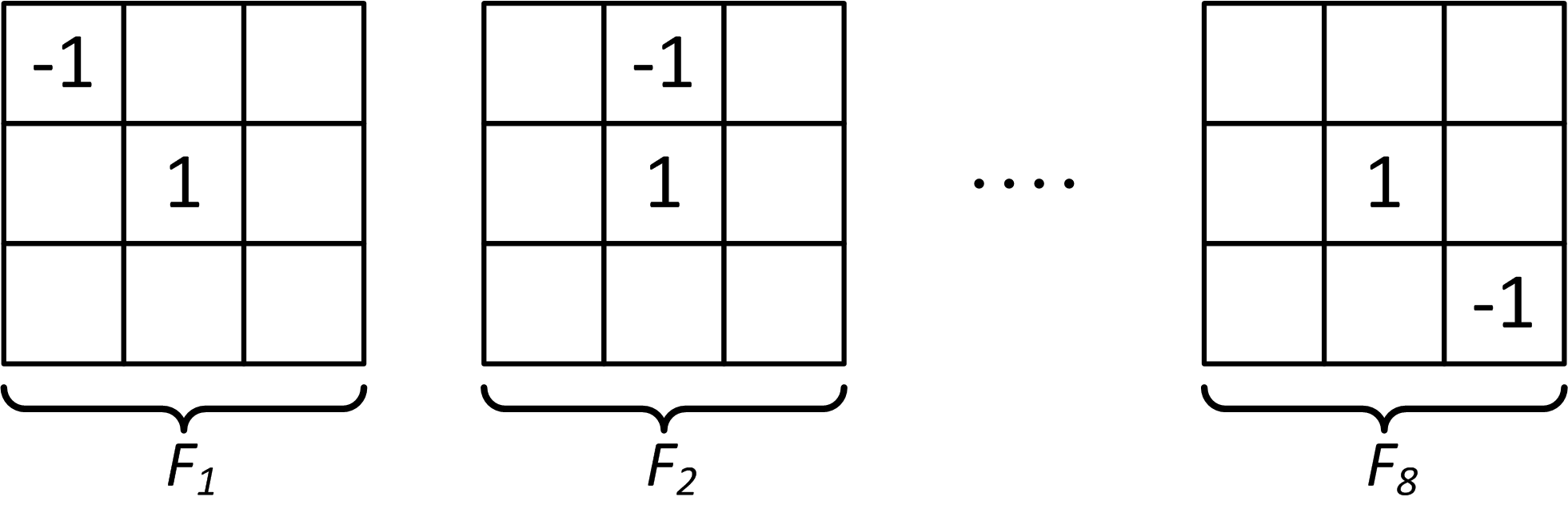}
  \caption{ Eight convolution filters that can be used to implement the NIP prior loss}
  \label{Fig:Filters}
\end{figure}

For optimizing a deep network using this cost function we need to make sure the cost function is differentiable so the error can propagate back through the network using a back-propagation approach. However, the cost function in \eqref{Eq:NIP_cost3} is not differentiable at zero since it has an infinite slope. Therefore, in the optimization procedure, it produces infinitely large gradients which makes the network unstable (see figure \ref{Fig:NIP} on the right). One way to alleviate this problem is to fix the parameter $\alpha$ to be exactly equal to one which is exactly a relaxation equivalent to \eqref{Eq:NIP_l1} (see figure \ref{Fig:NIP} on the left for $\alpha=1$). A better approach to tackle this problem is to approximate the norm function with something having a large but finite derivative at zero. For example we can approximate $\|x\|_\alpha$ for $\alpha=0.1$ with $10 \log\Big((e^{10}-1)|x|+1\Big)$ (See figure \ref{Fig:NIP} in the middle).
\begin{figure}
  \centering
  \includegraphics[width=0.99\textwidth]{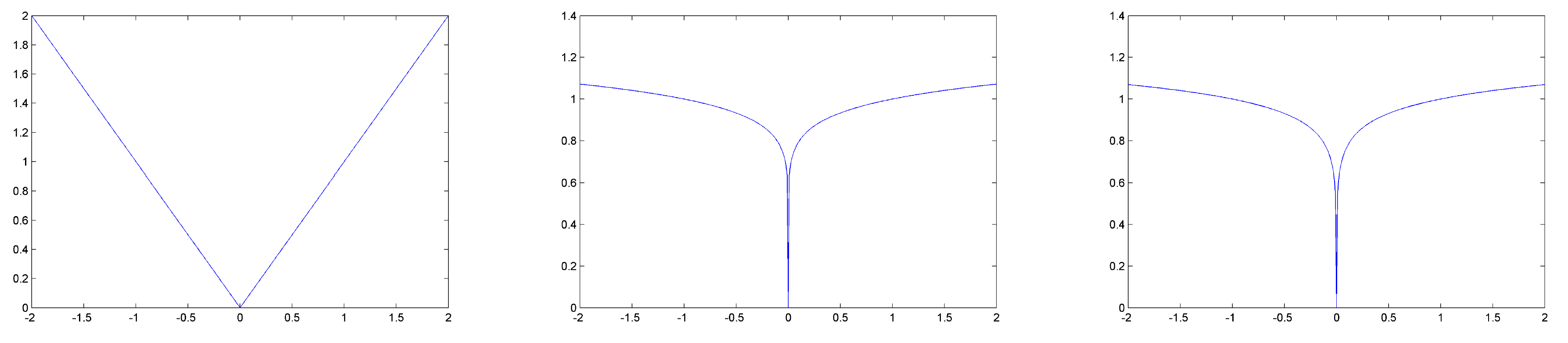}
  \caption{ Image priors introduced in this section. Illustrations are for one dimension only. From left to right: $\|x\|_1$,  $10 \log\Big((e^{10}-1)|x|+1\Big)$, $\|x\|_{0.1}$}
  \label{Fig:NIP}
\end{figure}

\subsection{Deep Super Resolution in Low Training Scenario}
\label{Sec:SR_lowTr}
The abundance of training data in deep learning provides very compelling results in areas such as object recognition and natural language processing as well as low-level vision tasks such as super-resolution and denoising. However, the performance of such networks degrades when the presence of abundant training is not an option and usually they perform very bad in these low training scenarios. One way of compensating for lack of enough training data is to use image statistics and priors as suggested in the previous section. Natural Image Priors (NIP) are among the most suitable ones for SR task. In this section, we aim to apply NIP on deep networks specialized for the super-resolution task. The proposed network structure is shown in Fig. \ref{fig:NIP_network}, which consists of an SR network for generating the super-resolution result and also a few additional convolutional layers to impose the NIP prior.
\begin{figure}
  \centering
  \includegraphics[width=0.99\textwidth]{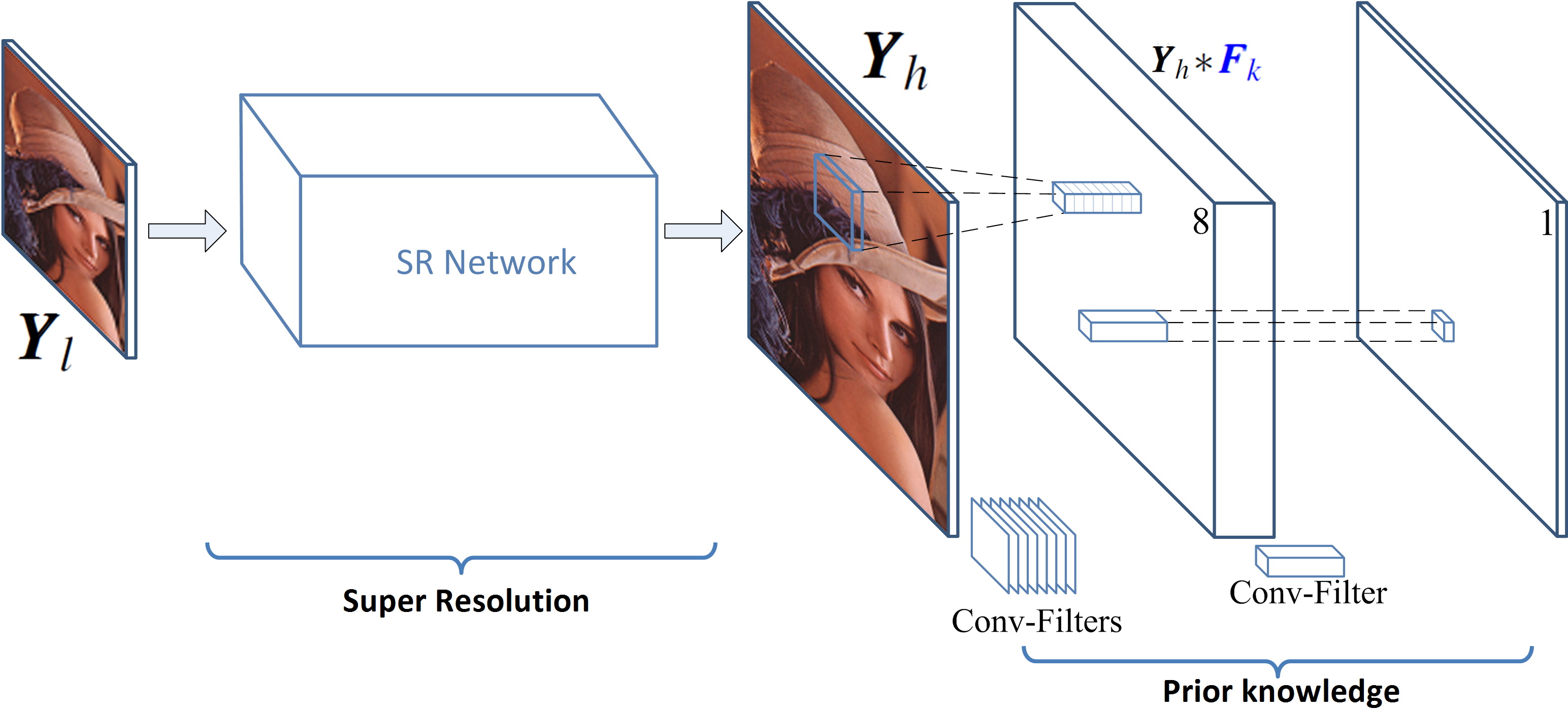}
  \caption{The network structure for imposing NIP priors}
  \label{fig:NIP_network}
\end{figure}

The ``SR Network" in Fig. \ref{fig:NIP_network} can be chosen to be any network specific for super-resolution task and here we pick the Very Deep Super Resolution (VDSR) \cite{Kim_2016_VDSR} network as one of the state-of-the-art methods for validating our idea. However, this idea can be applied to any other SR network such as SRCNN, etc. VDSR (shown in Fig. \ref{fig:VDSR}) is a residual network with 20 convolutional layers that takes the input low-resolution image as input and generates the output residuals that needed to be added to the input image in order to generate a high-resolution output image.
\begin{figure}
  \centering
    \includegraphics[width=0.9\textwidth]{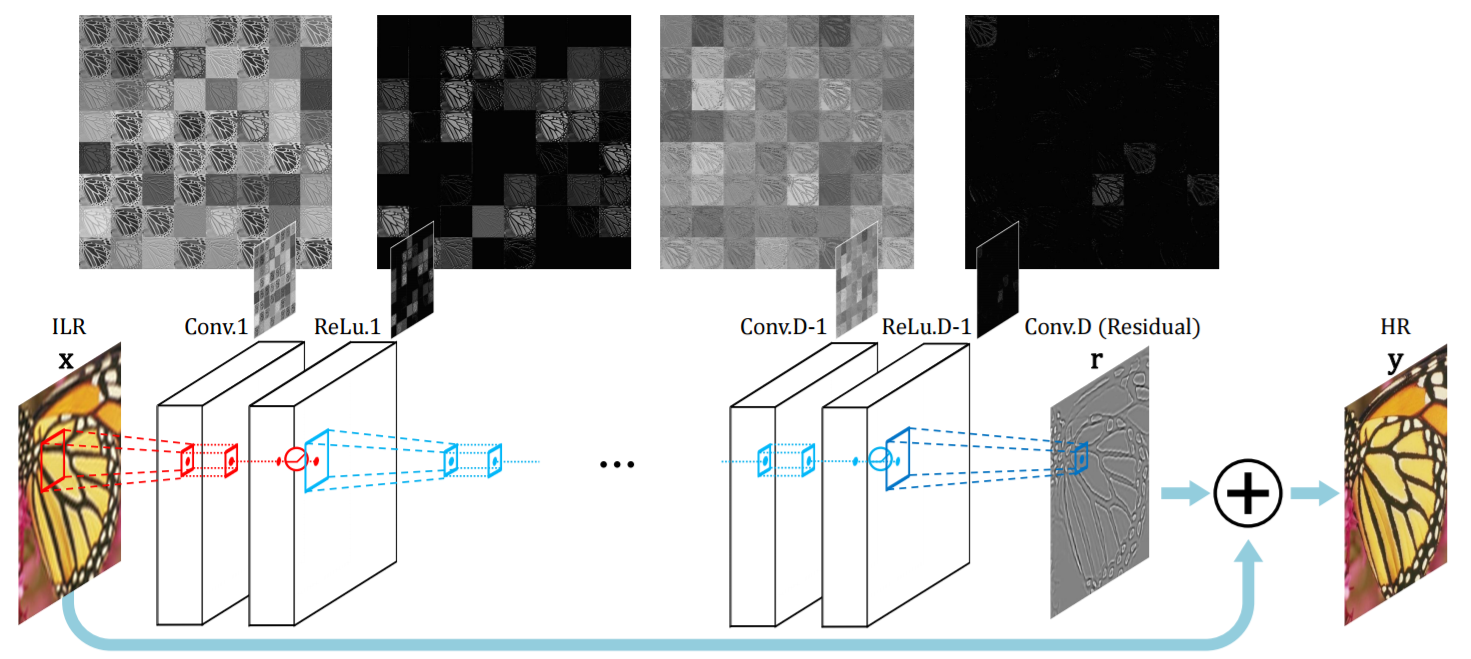}
    \caption{VDSR network for super resolution}
    \label{fig:VDSR}
\end{figure}

The output ``SR Network" goes into another layer of convolution with 8 non-learnable filters that are illustrated in Fig \ref{Fig:Filters} and passed through the nonlinear equivalent of $\alpha$-norm to form a data cube with 8 channels. These channels are summed across channels and then across spatial dimensions to provide the NIP part of the loss function in \eqref{Eq:NIP_cost3}. Note that the parameter $\alpha$ here can be learnable. Although the filters in the last layer are not learnable and are fixed, the error that is caused by NIP layers propagates back to the main SR network and causes the weights to adjust for NIP. The cost function to be minimized here is as proposed in \eqref{Eq:NIP_cost3} which is approximated using the log-function so that it is differentiable. The back-propagation rules needed to optimize the network stay the same for the reconstruction error term (second term in \eqref{Eq:NIP_cost3}) but for the last layer which corresponds only to computing the loss for NIP priors can be summarized as follows where we need $\frac{\partial E}{\partial W_{ab}^{\ell_p}}$ for every layer $\ell_p=2,..,L_p-1$. Using chain rule we can derive:
\bea
    \frac{\partial E }{\partial W_{a,b,c}^{\ell_p}} = \sum_{k=1}^{3}  \sum_{i=0}^{n-1}  \sum_{j=0}^{n-1} \delta_{i,j,k}^{\ell_p+1} ~.~ A_{i+a,j+b,k+c}^{\ell_p}
    \label{Eq:dEdW}\\
    \delta_{i,j,k}^{\ell_p} = \sum_{s=1}^{3} \; \sum_{u=0}^{n-1} \; \sum_{v=0}^{n-1} \delta_{i-u,j-v,k-s}^{\ell_p+1} ~.~ W_{u,v,s}^{\ell_p}
    \label{Eq:modified_delta}
\eea
And for the last layer $\ell_p=L_p$:
\bea
    \delta_{i,j}^{L_p} = \frac{\partial E }{\partial Z_{i,j}^{L_p}} \stackrel{\mat Z^{L_p}=\mat A^{L_p}}{=\joinrel=\joinrel=\joinrel=} \frac{\partial E }{\partial A_{i,j,k}^{L_p}}
        = 1
\eea
where $\mat W^{\ell_p}$ are the weights associated with layer $\ell_p$ of convolutions, $\mat A^{\ell_p}$ are the activation outputs of layer $\ell_p$, $\mat Z^{\ell_p}$ are the inputs to the layer $\ell_p$ and $\delta_{\ell_p+1}$ is the error propagated back from layer $\ell_{p+1}$ to layer $\ell_p$. For the last layer, Activation maps $\mat A^{L_p}$ and inputs $\mat Z^{L_p}$ are the same, i.e. $\mat Z^{L_p}=\mat A^{L_p}$.

This network is designed so that it can capture image statistics from the training data and generate output images with respect to natural image prior (NIP), especially in scenarios where training data is limited and generic deep SR networks fail to provide satisfying results for super-resolution. In the following sections, we provide different types of prior that can come in handy in low training scenarios and then discuss the effects of such priors in practice.
%
%
%
%
%

\section{Experimental Results}
In this section, we provide the experimental results and procedures corresponding to our method. We first describe the datasets used for training and testing, then explain the training procedure used. Finally, we compare our method with state of the art methods for super-resolution in high training and low training scenarios to show the benefits of regularizing deep networks with image priors.

\subsection{Dataset Preparation and Training Procedure }
For training dataset we use the 291 images from \cite{schulter2015fast} which contains natural images. Data Augmentation, including flipping, rotation, and scaling, was performed for training with high amount of data. For test scenario, we use the `set 14' \cite{zeyde2010single} dataset. The training procedure is very similar to what was proposed in section \ref{Sec:DWSR_Experiment}; however, with some small modifications for purpose of stability and faster convergence. As mentioned before, the SR network is chosen to be similar to VDSR \cite{Kim_2016_VDSR} with 20 convolutional layers. An additional convolutional layer with non-learnable (fixed) weights is also added to compute the loss function corresponding to natural image priors. The training uses batches of size 64 and momentum and weight decay parameters are set to $0.9$ and $0.0001$. Also, gradient clipping is used as proposed by \cite{Kim_2016_VDSR} to prevent gradients from exploding.

We train all experiments over 300 epochs over all training data (no matter how much training data is used). The learning rate was initially set to $0.1$ and then decreased by a factor of 10 at epochs 60 and 140. Similar to other recent SR methods, our framework applies bicubic interpolation to color components of images and only the luminance channel is fed to the deep network.

\subsection{NIP with Abundant Training}
In this section, we provide the experimental results under an abundant amount of training data. Essentially, we train our NIP network with the regularized cost function which also takes natural image priors into account. We denote $\sigma$ as $\frac{\sigma_R^2}{\sigma_N^\alpha}$ and train the network with more than $140,000$ training data pairs from the database. The core super-resolution network in our NIP network is the state-of-the-art VDSR network and we are showing comparison against it. Note that we train our network and also VDSR from scratch using the same initialization of weights and the same order of batches for fairness of comparison and to rule out any performance boost or degradation due to local minima. The first set of results for a scaling factor of $3$ are shown in Fig \ref{Fig:Woman_HighTr} and \ref{Fig:Lenna_HighTr}. We show the ground truth image as well as the bicubic interpolation, VDSR and also the results of our NIP network. Note that the value of $\sigma$ is picked according to a cross-validation procedure on a different set of images and the best value of $\sigma = e^{-5}$ is picked. However, we are also showing the results for $\sigma = e^{-4}$ which reveals that assigning higher regularizer parameters to NIP prior can cause the images to become smoother and piecewise linear as described and expected before. Comparing the visual results of VDSR and our NIP network with the right value of $\sigma$ shows that there is no visible difference in the output images of VDSR and our NIP networks if unlimited training is available. Also the average performance of each method over \emph{Set 14} is provided in Table \ref{Table:HighTrSet14}. It is quite clear that when training data is readily available, the VDSR network can capture structural information from the training data and incorporate it in the weights of the network, and again validates that adding prior information and regularizing the network does not change the performance of network. Fig. \ref{Fig:diagrams_HighTr} also shows the evolution of the cost function value, the NIP term in the cost function and PSNR on the training batches as the optimization progresses.
\begin{figure}
  \centering
  \includegraphics[width=\textwidth]{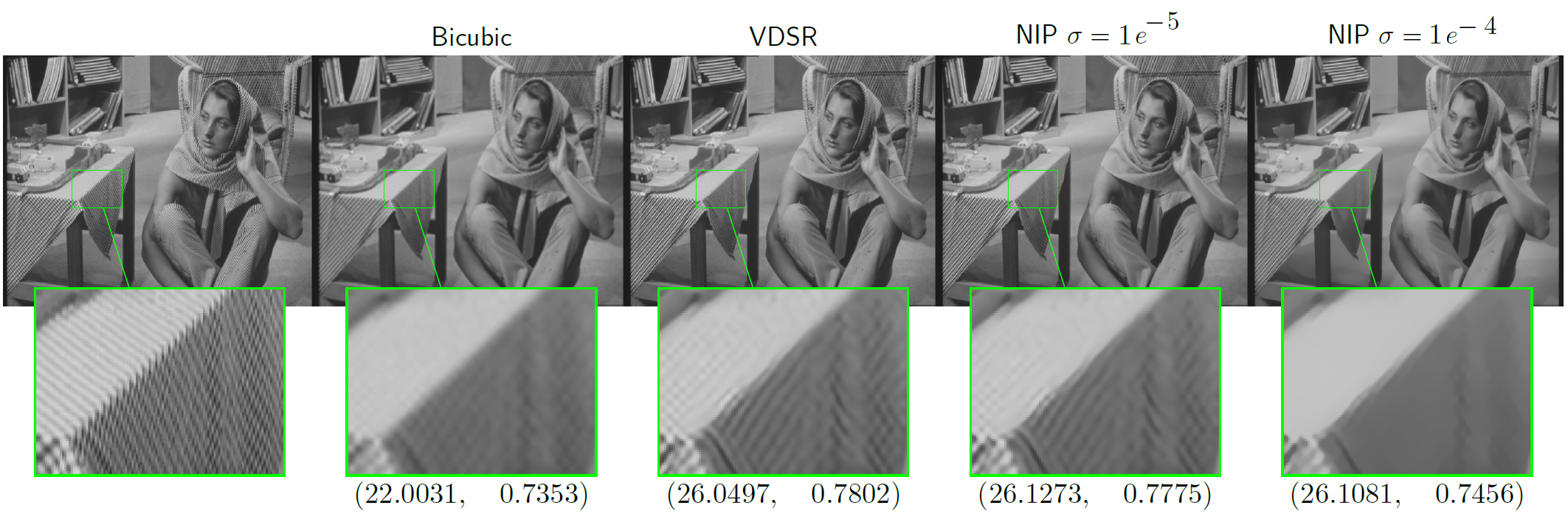}
  \caption{The image ``Woman" from Set 14. Numbers in parenthesis denote the PSNR and SSIM values respectively.}
  \label{Fig:Woman_HighTr}
\end{figure}
\begin{figure}
  \centering
  \includegraphics[width=\textwidth]{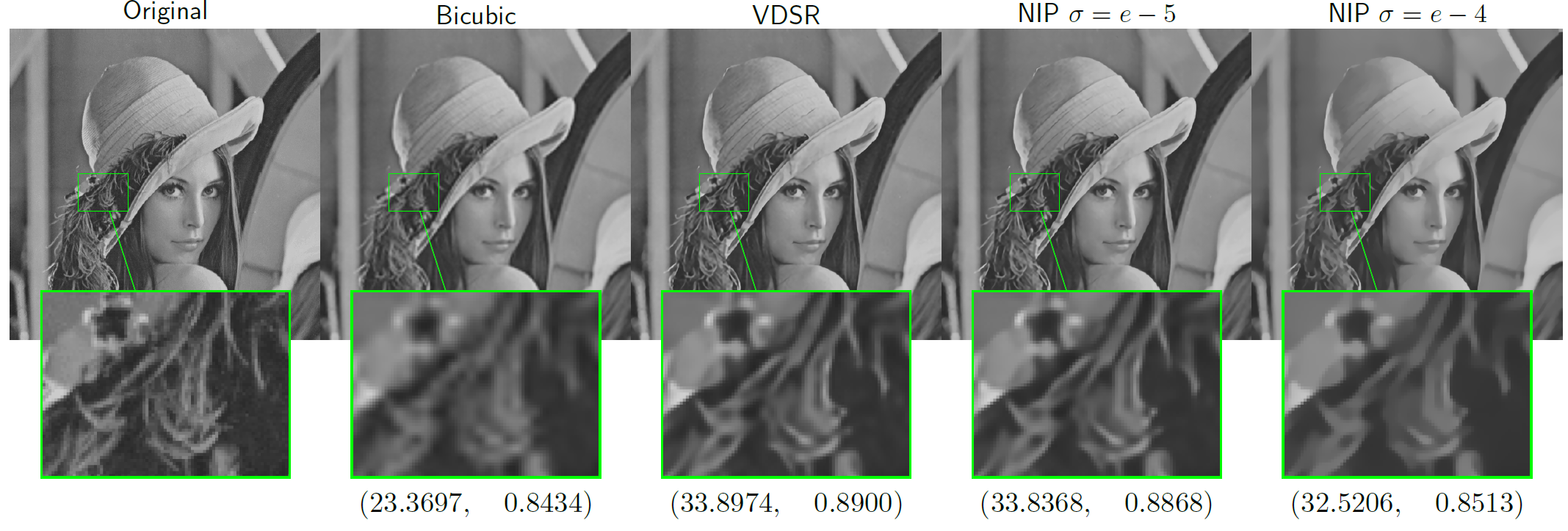}
  \caption{The image ``Lenna" from Set 14. Numbers in parenthesis denote the PSNR and SSIM values respectively.}
  \label{Fig:Lenna_HighTr}
\end{figure}
\begin{table}
\caption{ {Quantitative Results average over Set 14} } 
\centering
\begin{tabular}{c|c c  } 
\hline\hline 
Method      & SSIM              & PSNR       \\  
\hline 
VDSR                 &    \textbf{0.8301}  &     \textbf{29.7396}  \\
Bicubic                  &    0.7427  &     22.4457  \\
NIP $\sigma=e^{-4}$ &    0.7776  &    28.8047      \\
NIP $\sigma=e^{-5}$ &    0.8264  &    29.7242      \\
\hline 
\end{tabular}
\label{Table:HighTrSet14} 
\end{table}
\begin{figure}
      \centering
      \includegraphics[width=0.85\textwidth]{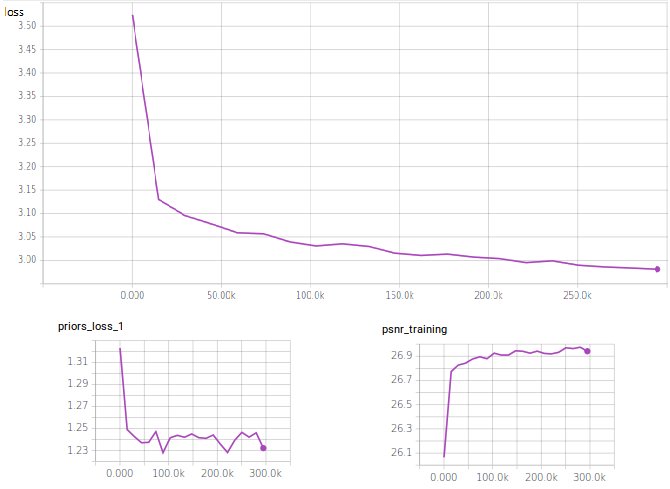}
      \caption{Evolution of cost function value, the NIP term in the cost function and PSNR on the training batches as the optimization progresses }
      \label{Fig:diagrams_HighTr}
\end{figure}

\subsection{NIP with Limited Training}
In this section, we investigate the performance of well-known super-resolution methods such as VDSR in low training scenario and provide evidence that how incorporating priors in the learning stage of neural networks can help alleviate the problem of lack of training data. We partially use the database and exploit only 5000 sample training patches. We first train the VDSR network with this limited training data. Unsurprisingly, the performance of VDSR degrades both visually and quantitatively based on PSNR and SSIM values and is illustrated in Fig \ref{Fig:baboon_HighLowExtLow_Tr}.

\begin{figure}
  \centering
  \includegraphics[width=\textwidth]{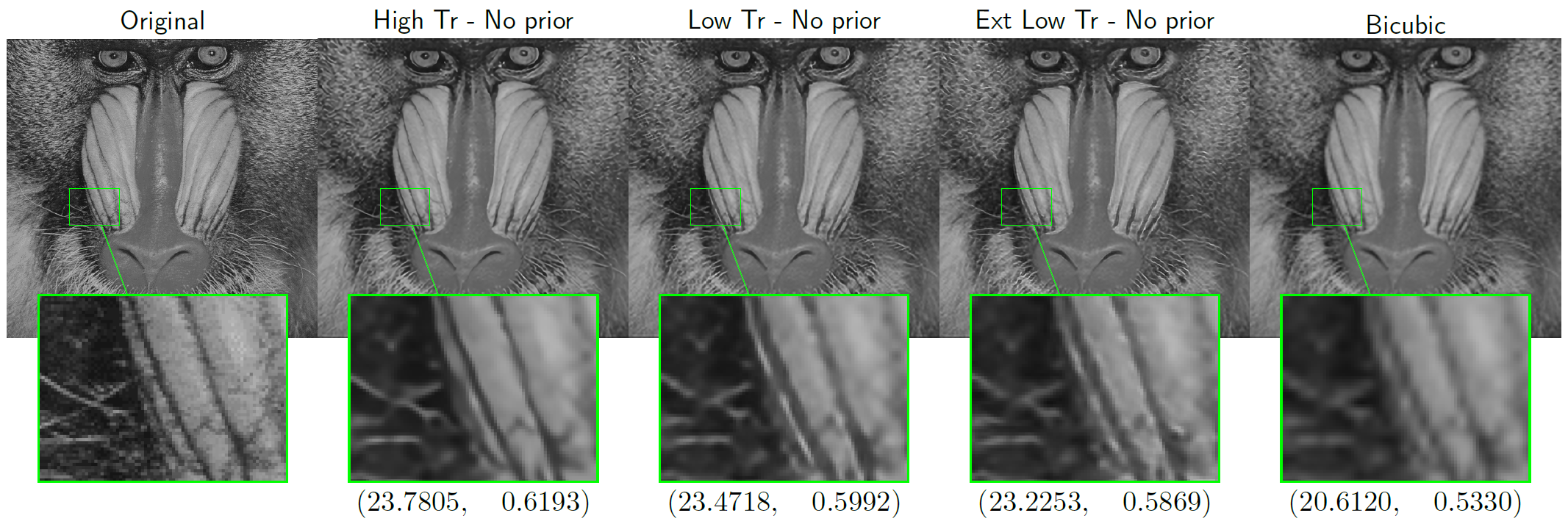}
  \caption{Performance of VDSR network under different training scenarios. Scenarios from left to right: Original ground truth image - VDSR trained with abundant training - VDSR trained low training data (5000 pairs) - VDSR trained with extremely low training data (1000 pairs) - Bicubic results. Numbers in parenthesis denote the PSNR and SSIM values respectively.}
  \label{Fig:baboon_HighLowExtLow_Tr}
\end{figure}

Next, we train our NIP network with low training data and compare it with the previously learned VDSR network without any prior knowledge and again in the low training scenario. We use $\sigma = e^{-7}$ which is obtained from a cross-validation procedure on a separate set of images. Experimental results on three different images are shown in Fig \ref{Fig:monarch_LowTr_VDSRvsNIP} to \ref{Fig:ppt3_LowTr_VDSRvsNIP}. In addition to quality measure improvements on these images based on PSNR and SSIM values, it is visually apparent that introducing priors in the low training scenario reduces artifacts around edges and provides visually more pleasant images. Table \ref{Table:LowTr} provides the average performance of NIP with priors compared to VDSR method with low training. It is now clear that in absence of abundant training data, presence of priors helps to improve the SR results. Essentially, priors compensate for the lack of enough training data for learning fine structures in the network. Fig. \ref{Fig:diagrams_LowTr} also shows the evolution of cost function value, the NIP term in the cost function and PSNR on the training batches as the optimization progresses in the low training scenario.

\begin{figure}
  \centering
  \includegraphics[width=\textwidth]{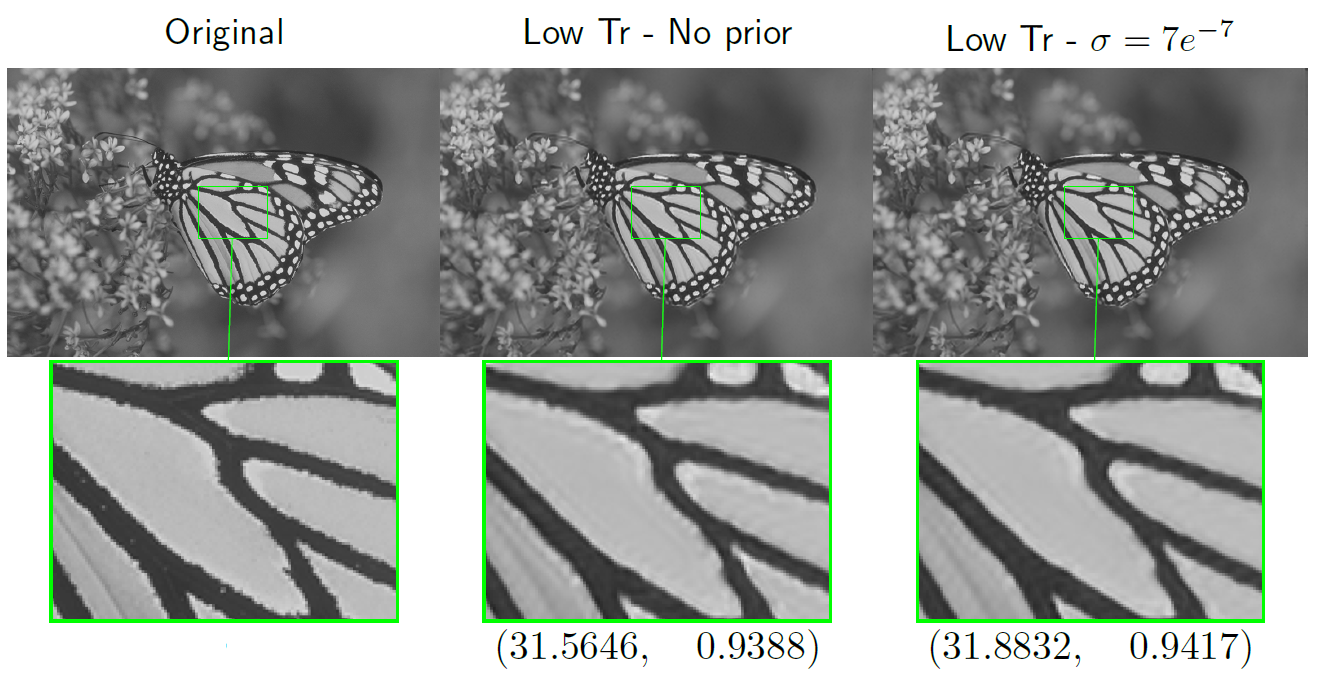}
  \caption{``monarch" - Performance of NIP network with and without prior knowledge. From left to right: Original ground truth image - VDSR trained with low training and no priors - NIP network trained with low training exploiting prior knowledge. Numbers in parenthesis denote the PSNR and SSIM values respectively.}
  \label{Fig:monarch_LowTr_VDSRvsNIP}
\end{figure}
\begin{figure}
  \centering
  \includegraphics[width=\textwidth]{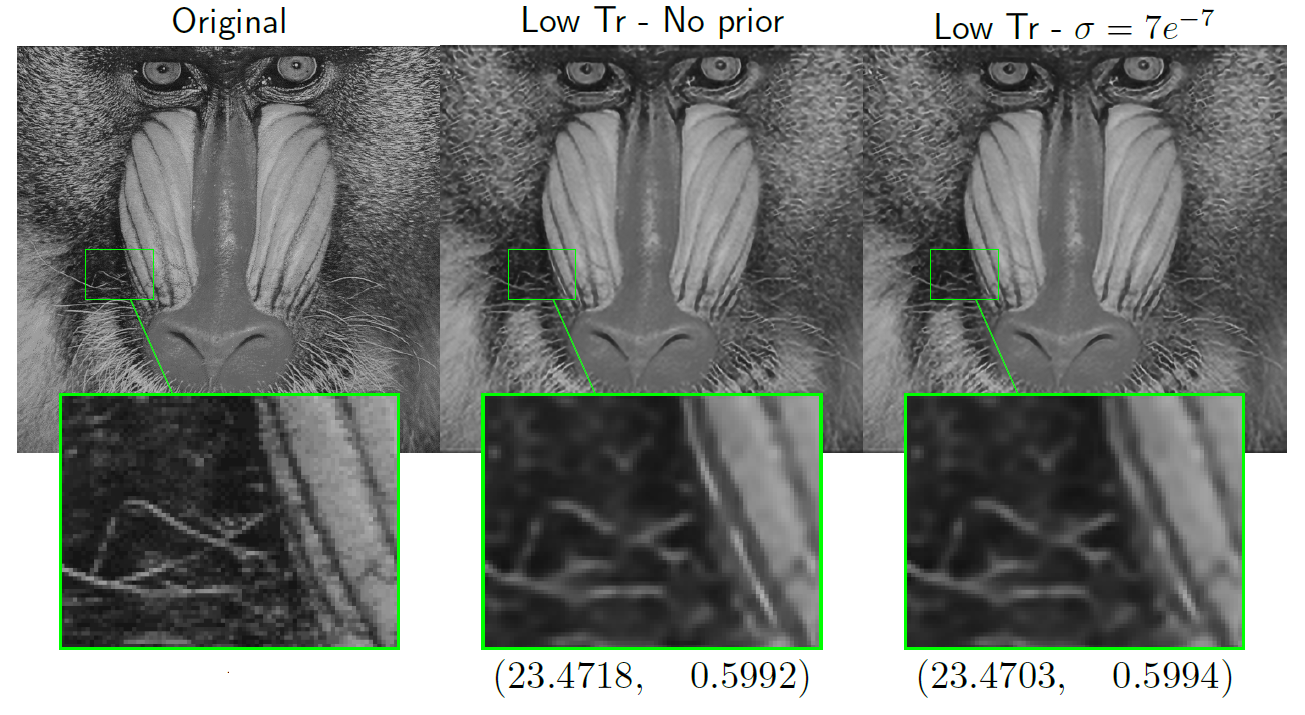}
  \caption{``baboon" - Performance of NIP network with and without prior knowledge. From left to right: Original ground truth image - VDSR trained with low training and no priors - NIP network trained with low training exploiting prior knowledge. Numbers in parenthesis denote the PSNR and SSIM values respectively.}
  \label{Fig:baboon_LowTr_VDSRvsNIP}
\end{figure}
\begin{figure}
  \centering
  \includegraphics[width=\textwidth]{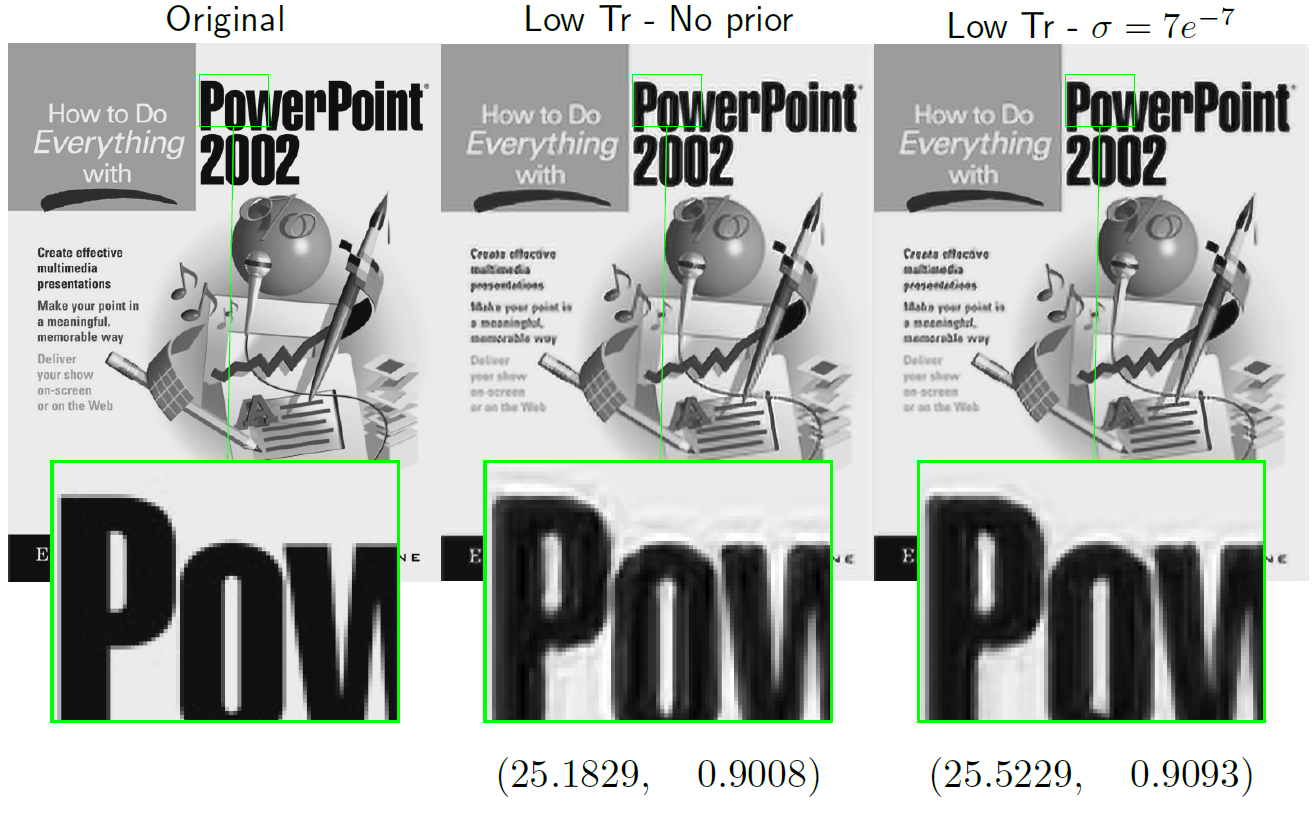}
  \caption{``ppt3" - Performance of NIP network with and without prior knowledge. From left to right: Original ground truth image - VDSR trained with low training and no priors - NIP network trained with low training exploiting prior knowledge. Numbers in parenthesis denote the PSNR and SSIM values respectively.}
  \label{Fig:ppt3_LowTr_VDSRvsNIP}
\end{figure}
\begin{table}
\caption{ {Quantitative results averaged over Set 14 in low training scenario} } 
\centering
\begin{tabular}{c c c  } 
\hline\hline 
Method      & SSIM              & PSNR       \\  
\hline 
Bicubic                               &    0.7427  &     22.4457  \\
VDSR - no prior                  &    0.8039  &     28.4035  \\
NIP                              &    \textbf{0.8051}  &    \textbf{28.4660}     \\
\hline 
\end{tabular}
\label{Table:LowTr} 
\end{table}
\begin{figure}
      \centering
      \includegraphics[width=\textwidth]{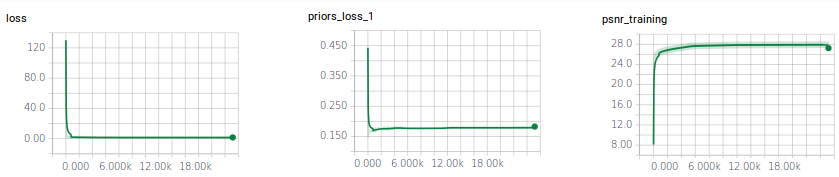}
      \includegraphics[width=\textwidth]{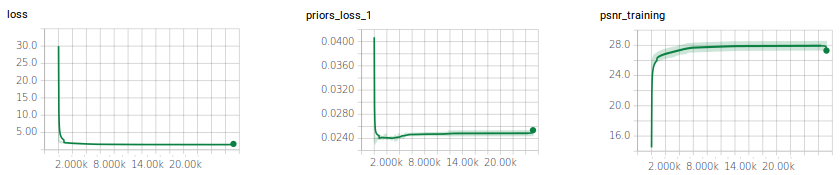}
        \caption{Evolution of cost function value, the NIP term in the cost function and PSNR on the training batches as the optimization progresses in low training scenario. Top row corresponds to VDSR with low training and bottom row corresponds to NIP network which uses prior knowledge.  }
        \label{Fig:diagrams_LowTr}
\end{figure}

Finally, to show the importance of image priors and how they can help the super-resolution task, we perform another experiment in which the amount of available training data is further limited. We reduced the amount of training sample pairs from 5000 to 1000 in this new experiment (namely extremely low training scenario) and repeated the procedure explained above. Using a cross-validation technique, $\sigma$ is picked to be $7e^{-6}$. Fig. \ref{Fig:baboon_LowTr_VDSRvsNIP} shows the ``monarch" image and compares the ground truth image with the VDSR method and NIP under extremely low training scenario. Prevalence of NIP with priors is also visually clear especially in the zoomed regions and also in terms of SSIM and PSNR. To further support that the notion of priors in lack of sufficient training is very beneficial, Table \ref{Table:ExtLowTr} provides the average performance of NIP with priors compared to the VDSR method in this extremely low training scenario.

\begin{figure}
  \centering
  \includegraphics[width=\textwidth]{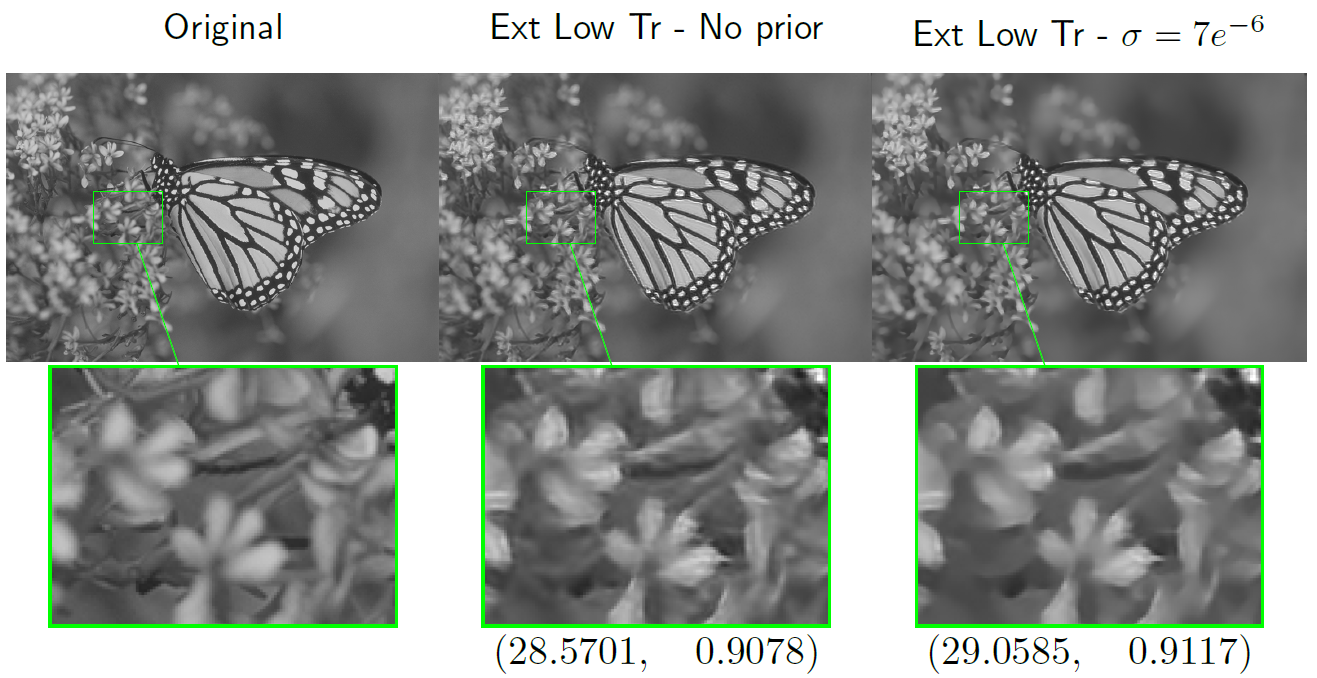}
   \caption{Performance of NIP network with and without prior knowledge. From left to right: Original ground truth image - VDSR trained with extremely low training and no priors - NIP network trained with extremely low training exploiting prior knowledge. Numbers in parenthesis denote the PSNR and SSIM values respectively.}
  \label{Fig:baboon_LowTr_VDSRvsNIP}
\end{figure}
\begin{table}
\caption{ {Quantitative Results average over Set 14 in extremely  low training scenario} } 
\centering
\begin{tabular}{c c c } 
\hline\hline 
Method      & SSIM              & PSNR     \\  
\hline 
Bicubic                               &    0.7427  &     22.4457 \\
VDSR -  no prior                 &    \textbf{0.7663}  &     26.8243 \\
NIP                              &    0.7550  &    \textbf{26.9855}     \\
\hline 
\end{tabular}
\label{Table:ExtLowTr} 
\end{table}

\section{Conclusions and Future work}
In this section, we first provided a novel deep wavelet super-resolution (DWSR) technique that recovers the ``missing details'' by using (low-resolution) wavelet sub-bands as inputs. DWSR is significantly economical in the number of parameters compared to most state-of-the-art methods and yet achieves competitive or better results. We contend that this is because wavelets provide an image representation that naturally sparsifies the mappings to be learned. While we used the Haar wavelet, effects of different wavelet bases can be examined in future work. Of particular interest could be to learn the ``optimal" wavelet basis for the SR task. Further, we investigated the effect of training data on performance of super-resolution networks in deep learning and proposed the use of natural image priors to encourage some notion of sparsity in the edge statistics of the image. This physically meaningful prior information on natural images shows promising performance improvement particularly in absence of abundant training data and shows the benefits of proposed NIP priors.
As viable future research direction for this line of work, we propose to identify other meaningful physical prior information for use in SR tasks and to demonstrate its benefits, especially in low training data scenarios.

\chapter{Conclusions and Future Work}
\label{chapter:conclusion}

\section{Summary of Main Contributions}
The overarching theme in this research is the design of \emph{signal recovery}  algorithms by exploiting  \emph{physically meaningful prior information}. On the theoretical front, This dissertation solves challenging problems in sparse signal and image processing. Using sparsity as a prior is tremendously interesting in a wide variety of applications; however, existing solutions to address this issue were sub-optimal and often fail to capture the intrinsic sparse structure of physical phenomenon. We address a very fundamental question in this area of how to efficiently and effectively capture sparsity in natural signals. More specifically, capturing general sparse structure in signals and images is a very challenging task and is considered an NP-hard mixed integer problem. However, in this work, we tried to break the trade-off between computational burden and performance where we proposed a novel method called Iterative Convex Refinement (ICR) for sparse recovery. In this work, a sequence of tractable convex optimization problems are solved in order to solve the original hard non-convex mixed integer programming problem for sparse recovery. ICR resulted in solutions that are far less computationally expensive and provides significant performance improvement over existing state-of-the-art solutions. Many signal processing problems in computer vision and recognition world can benefit from this result in sparse signal processing. 

On the other hand, one of the most significant challenges in signal recovery and image processing is the enhancement of image quality. We again addressed this challenge by using prior knowledge from physically meaningful assumptions on the sparse signals. In particular, we pose probabilistic priors to promote sparsity on design parameters of the problem and showed performance improvements in many applications including image super-resolution.

Throughout this dissertation, we proposed extensions of the super-resolution task where we addressed this problem from two different perspectives with the common theme of exploiting prior information: (1) extension of sparsity-based super-resolution problems
to color channels by taking edge similarities amongst RGB color bands into account as cross-channel correlation constraints. (2) view the super-resolution problem from a deep learning standpoint and provide deep network architectures designed for using structural knowledge of images for super-resolution. Furthermore, we investigated exploiting sparsifying priors into deep networks and analyzed their influence on the performance of super-resolution especially in the absence of abundant amount of training data.

In order to verify that the use of prior information  is indeed beneficial in a variety of scenarios including low training scenarios, we considered  important real-world applications including: (1) Signal Recovery with sparsifying priors
(2) Image Recovery and
(3) Image Super-Resolution with cross-channel constraints and natural image priors.
In each problem, we observed that our prior model exhibited performance boosts and robustness
to low training scenarios.

\section{Potential Future Research Directions}
The contributions in the previous chapters naturally point towards various directions for
future research. We mention some of the possible extensions in this section.
\subsection{Signal and Image Recovery}
Future research may investigate further analysis of our Iterative Convex Refinement (ICR) properties and provide more analytical evidence about the convergence of ICR and its rate of convergence. Accelerating the ICR algorithm and parameter learning for obtaining more accurate recovery results can be further investigated. Also, extensions to multi-task scenarios where measurements are available from multiple sources is a viable research direction.
In this line of research exploiting different sparsifying priors and extension to collaborative signal recovery is the most reasonable and promising future research direction.

For image restoration, we used a color prior for super-resolution task. While our multi-channel super-resolution work incorporated signal priors to capture cross-channel color constraints, chrominance geometry can be captured via many different ways as is suggested in \cite{Srinivas:ColorSR_CIC2011, Farsiu:ColorDemosaicSR_TIP2006, Keren:ColorSR_1999MachineVision, Dai:SoftCutColorSR_2009TIP}. Incorporating these as constraints or regularizers in a sparsity-based color SR framework forms a viable direction for future work. For example, prior knowledge of gradient maps of images or edge structures can be used for boosting the performance of super-resolution. Especially when the training data is not readily available.

\subsection{Deep Learning for Inverse Problems in Computational Imaging}

Clearly, the most recent trend in many of computational imaging problems is deep learning, which has arisen as a promising framework providing
state-of-the-art performance for many other applications including but not limited to image classification, segmentation, etc.  Moreover, regression-type neural networks demonstrated significant improvements in results on inverse problems such as denoising \cite{burger2012image, xie2012image}, deconvolution \cite{xu2014deep} and super-resolution \cite{dong2014learning}.

More recently, researchers are investigating the link between conventional approaches in signal processing and deep learning
networks \cite{gregor2010learning, chen2015learning}. For instance, LeCun et al. \cite{gregor2010learning} explored the similarity between the ISTA method \cite{daubechies2004iterative} and a specific neural network and demonstrated that layer-wise neural networks act as an approximated sparse coder \cite{jin2017deep}.

Despite these, many practical and theoretical questions remain regarding how to benefit from conventional approaches in deep learning frameworks. For example, in scenarios where training data is not sufficiently available, how we can benefit from the established methods in conventional signal processing to help deep frameworks. The main question here is can the same gain as in the conventional established methods be realized in deep learning frameworks?

Along the same direction, we explored the potentials of regularized deep neural networks in this dissertation. Motivated by the fact that prior information can be beneficial in low training scenarios in conventional learning-based methods, we investigate the use of prior knowledge for image super-resolution from a deep learning standpoint.  However, regularized networks with different kinds of priors can be applied on a various range of applications in computational imaging. Among these, super-resolution, denoising, inpainting, etc. are the most common ones.
\begin{itemize}

\item Super-resolution: Specifically, deep super-resolution where we tackled in this dissertation by applying prior knowledge has many potentials for future work. Specifically, From a deep learning standpoint, there are many potential future directions. For example, deep learning community always builds their models with the assumption of having abundant training data. However, as stated before this is not always the case and there are many applications especially in the medical side that the training data is often limited, e.g. MRI image reconstruction, CT super resolution, etc. To the best of our knowledge, we were among the firsts to address such issue in the deep neural networks and we believe there is significant and potential room for improvement in this direction.
As viable future research direction for this line of work, we propose to identify other meaningful physical prior information for use in deep networks for the SR tasks and to demonstrate its benefits, especially in low training data scenarios.

  \item Denoising: Observed image signals are often corrupted by acquisition channel. The goal of
image restoration techniques is to restore the original image from a noisy observation of it. Image
denoising  is a very common image recovery problem. Deep learning community has successfully addressed this problem \cite{xie2012image} but the assumption here is the availability of generous amount of training data which is not always readily available. For example in many medical imaging applications, or target detection in radar applications, training data is very limited and deep frameworks may fail in this situation. However, as suggested by conventional methods, regularizing the learning procedure with prior knowledge or physically meaningful priors may help in these scenarios.

\item Inpainting is another example of an inverse problem in computational imaging. This problem has also a long history of literature from conventional methods to recent deep frameworks \cite{xie2012image, jin2017deep}. This problem has also the same limitation when it comes to performance in the low training scenarios. Using prior knowledge or regularizers for the deep structure can be further investigated for this inverse problem as well. Such priors can be used in different forms such as the knowledge of gradient maps of the images, edge information \cite{kim2010single} or color constraints \cite{Srinivas:ColorSR_CIC2011}.


\end{itemize}



%


} 

   \begin{singlespace}
   \bibliographystyle{IEEEtran}
   \addcontentsline{toc}{chapter}{Bibliography}
   \bibliography{Biblio-Database,IEEEabrv}
   \end{singlespace}

\newpage
\thispagestyle{empty}

\small
\textbf{Education:}\\
Pennsylvania State University, Ph.D. Electrical Engineering, Dec. 2017.\\
Pennsylvania State University, M.Sc. Electrical Engineering, Aug. 2016.\\
Sharif University of Technology, B.Sc. Electrical Engineering, Sep. 2011.\\
\textbf{Experience:}\\
Research Assistant, Pennsylvania State University, Aug. 2012 - Jul. 2017.\\
Research Intern, Samsung Research America, May 2016 - Aug. 2016.\\
Teaching Assistant, Pennsylvania State University, Jan. 2014 - May 2014.\\
\textbf{Personal Website:}\\
\url{https://www.linkedin.com/in/hojjatmousavi} \\ 
\textbf{Selected Publications related to Ph.D. research}
\begin{enumerate}
\vspace{-0.1in}\item Sparsity-based Color Image Super Resolution via Exploiting Cross Channel Constraints
        \textbf{H.S. Mousavi}, V. Monga.
        IEEE Transactions on Image Processing, vol.26, no.11, pp. 5094 - 5106, Nov. 2017.
\vspace{-0.1in}
\item Histopathological Image Classification Using Discriminative Feature-oriented Dictionary Learning
        T. H. Vu, \textbf{H.S. Mousavi}, V. Monga, UK A. Rao, G Rao
        IEEE Transactions on Medical Imaging, vol.35, no.3, pp.738 - 751, Mar. 2016

\vspace{-0.1in}\item  ICR: Iterative Convex Refinement for Sparse Signal Recovery Using Spike and Slab Priors
        \textbf{H.S. Mousavi}, V. Monga, T. D. Tran
        IEEE Signal Processing Letters, vol.22, no.11, pp. 1903 - 1907, Nov. 2015

\vspace{-0.1in}\item Simultaneous sparsity model for histopathological image representation and classification
        U. Srinivas, H. S. Mousavi, V. Monga, A. Hattel, and B. Jayarao
        IEEE Transactions on Medical Imaging, Volume 33, Issue 5, pp. 1163 - 1179, May 2014.

\vspace{-0.1in}\item Automated Discrimination of Lower and Higher Grade Gliomas Based on Histopathological Image Analysis
      \textbf{H.S. Mousavi}, V. Monga, A. U. Rao, G. Rao
      Journal of Pathology Informatics, vol.6, Mar 2015.

\vspace{-0.1in}\item Deep Wavelet Prediction for Image Super-resolution
        T. Guo, \textbf{H.S. Mousavi}, T. H. Vu, V. Monga.
        Proc. IEEE International Conference on Computer Vision and Pattern Recognition Workshops (CVPRW), July, 2017.
        New Trends in Image Restoration and Enhancement workshop and challenge on image super-resolution

\vspace{-0.1in}\item Sparsity Based Super Resolution Using Color Channel Constraints
        \textbf{H .S. Mousavi}, V. Monga.
        Proc. IEEE International Conference on Image Processing (ICIP), Sep, 2016.

\vspace{-0.1in}\item Deep Learning Based Super-Resolution With Coupled Auto-Encoder Networks
        T. Guo, \textbf{H.S. Mousavi}, V. Monga.
        Proc. IEEE Global Conference on Signal and Information Processing (Global-SIP), Dec 2016.

\vspace{-0.1in}\item Collaborative Hierarchical Image Classification via Spike and Slab Priors
     \textbf{H.S. Mousavi}, U. Srinivas, Y. Suo, M. Dao, V Monga, T. D. Tran
     Proc. IEEE International Conference on Image Processing (ICIP),  pp. 4236-4240, Oct. 27-30, 2014

\end{enumerate}

\backmatter

\end{document}